\documentclass[final,3p,times]{elsarticle}
\usepackage{amssymb}   % 数学符号表
\usepackage{caption}   % 图编号
\usepackage{subfigure} % 子图
\usepackage{amsmath}   % 公式
\usepackage{setspace}  % 公式间距
\usepackage{algorithm} % 伪代码
\usepackage{algpseudocode} % 伪代码
\usepackage{makecell}  % 表格
\usepackage{multirow}  % 表格
\usepackage[justification=centering]{caption} % 图标标题居中
\usepackage{multicol}
\usepackage{stfloats}  % h-放在此处;t-放在顶端;b-放在底端;p-在本页
\usepackage{color}     % 对正文的修改部分做高亮处理
\usepackage{amsthm}    % 证明结尾的白色方框
\usepackage{threeparttable}
\newtheorem{theorem}{Theorem}
\newdefinition{definition}{Definition}
\newtheorem{example}{Example}

\makeatletter
\newenvironment{breakablealgorithm}{
	\begin{center}
		\refstepcounter{algorithm}
		\hrule height.8pt depth0pt \kern2pt
		\renewcommand{\caption}[2][\relax]{
			{\raggedright\textbf{\ALG@name~\thealgorithm} ##2\par}
			\ifx\relax##1\relax 
			\addcontentsline{loa}{algorithm}{\protect\numberline{\thealgorithm}##2}
			\else 
			\addcontentsline{loa}{algorithm}{\protect\numberline{\thealgorithm}##1}
			\fi
			\kern2pt\hrule\kern2pt
		}}{\kern2pt\hrule\relax
	\end{center}}
\makeatother

\journal{Information Sciences}

\begin{document}

\begin{frontmatter}
\title{Sequential three-way decisions with a single hidden layer feedforward neural network}
  
\author[1,2]{Youxi Wu}	

\author[1]{Shuhui Cheng}	
		
	\ead{shuhui\_cheng@163.com}
  
	\author[1]{Yan Li}
		
	\author[1]{Rongjie Lv}
		
	\author[3]{Fan Min}

		
		\address[1]{School of Economics and Management, Hebei University of Technology, Tianjin 300401, China}
		
		\address[2]{School of Artificial Intelligence, Hebei University of Technology, Tianjin 300401, China}

		\address[3]{School of Computer Science, Southwest Petroleum University, Chengdu 610500, China}
		
\begin{abstract}
The three-way decisions strategy has been employed to construct network topology in a single hidden layer feedforward neural network (SFNN). However, this model has a general performance, and does not consider the process costs, since it has fixed threshold parameters. Inspired by the sequential three-way decisions (STWD), this paper proposes STWD with an SFNN (STWD-SFNN) to enhance the performance of networks on structured datasets. STWD-SFNN adopts multi-granularity levels to dynamically learn the number of hidden layer nodes from coarse to fine, and set the  sequential threshold parameters. Specifically, at the coarse granular level, STWD-SFNN handles easy-to-classify instances by applying strict threshold conditions, and with the increasing number of hidden layer nodes at the fine granular level, STWD-SFNN focuses more on disposing of the difficult-to-classify instances by applying loose threshold conditions, thereby realizing the classification of instances. Moreover, STWD-SFNN considers and reports the process cost produced from each granular level. The experimental results verify that STWD-SFNN has a more compact network on structured datasets than other SFNN models, and has better generalization performance than the competitive models.	All models and datasets can be downloaded from https://github.com/wuc567/Machine-learning/tree/main/STWD-SFNN.	

\end{abstract}
		
%\newpage

			\begin{keyword}
{network topology \sep hidden layer node \sep sequential three-way decisions \sep granular level \sep sequential thresholds}
   
%				sequential three-way decisions  \sep  neural network  \sep  network topology  \sep  granular level \sep hidden layer node
			\end{keyword}
		\end{frontmatter}
		
	\section{Introduction}
		Neural networks \cite{Cao2016INS,Li2022ESWA,Liu2022KBS} are the mathematical models that simulate the connection structure of brain neurons. They have been widely implemented in applications, including video frame inpainting \cite{Szeto2020TPAMI} and automatic driving \cite{Shekar2021IJCV}. The performance of neural networks is mainly affected by hyperparameter selection and network topology. Hyperparameter selection \cite{Arqub2014INs,Arqub2021MMAS} is a classical topic in machine learning, which can be realized by grid search \cite{Pontes2016Neurocomputing,Shi2019SJ} and particle swarm optimization \cite{Aziz2021ESWA,Nistor2022ESWA}. In addition, network topology \cite{Akyol2020ESWA,Rusek2020JSAC,Ye2018ESWA} is the key of neural network design, which can be realized through three-way decisions \cite{Cheng2021INS} and an incremental learning mechanism \cite{Escovedo2020APIN,Kim2020TFS,Yang2019AIR}.		
		
	To achieve an effective network structure, three-way decisions with a single hidden layer feedforward neural network (TWD-SFNN) \cite{Cheng2021INS} adopts a novel model to guide the number of hidden layer nodes. In addition, as a shallow neural network model, TWD-SFNN provides a new perspective for the topology design of multilayer neural networks, hence laying the theoretical foundation for the framework of deep learning. However, for practical applications, TWD-SFNN has two drawbacks: (i) in terms of the performance of TWD-SFNN, the generalization ability of TWD-SFNN needs to be further improved; and (ii) to analyze the relationship between the costs and number of hidden layer nodes more thoroughly,  the process costs of TWD-SFNN need to be considered.
		
	To improve the generalization ability of neural networks on structured datasets, and further enrich the theoretical framework of deep learning, we employ sequential three-way decisions to guide the growth of the network topology. The sequential three-way decisions (STWD) \cite{Qian2022IJAR,Yang2019IJAR,Zhang-stwd-2021KBS} is a kind of progressive granular computing \cite{Ciucci2020INS,Niu2022INS,Zhang2022ESWA}, which realizes the multi-stage learning from coarse to fine by constructing granular levels. Specifically, STWD makes deterministic decisions for only a few instances at the coarse-grained granular level. With the addition of information, the decision of the remaining instances is gradually realized at the finer level. Recently, STWD has been applied to medical treatment \cite{Chu2020CIE}, face recognition \cite{Li2016KBS}, and attribute reduction \cite{Qian2017IJAR}.
		
	To obtain a better performance, we propose a model called STWD with a single hidden layer feedforward neural network (STWD-SFNN). STWD-SFNN enhances the learning ability of the SFNN and ensures the compactness of the network structure. The differences between TWD-SFNN and STWD-SFNN are as follows. Although TWD-SFNN has the multi-granularity property, the process costs produced from granular levels are not investigated in the development discipline of costs. On the other hand, although TWD-SFNN has the sequential property, it partitions the instances by adopting the constant threshold parameters when a new hidden layer node is added, making it difficult to focus on more difficult-to-classify instances. The contributions of this paper are as follows. 

 %Hence, compared with TWD-SFNN, STWD-SFNN can improve learning ability and consider process costs, including test costs and delay costs. 
		
\begin{enumerate}[(1)] 
	\item        	 To improve the generalization ability of TWD-SFNN, we employ the STWD scheme to optimize the structure of SFNN, and propose a new approach called STWD-SFNN.
			
	\item        	  STWD-SFNN handles the easy-to-classify instances by adopting strict conditions. As the number of hidden layer nodes increases, STWD-SFNN addresses the difficult-to-classify instances by adopting loose conditions until all instances can be classified.
			
	\item        	 The experimental results show that the network structure of STWD-SFNN is significantly better than other SFNN models, and STWD-SFNN has better accuracy and weighted-f1 performance than the competitive models. 
\end {enumerate}
			
	The remainder of the paper is organized as follows. Section 2 introduces the preliminaries. Section 3 describes the STWD-SFNN model. Section 4 analyzes the performance of STWD-SFNN, and Section 5 presents the conclusions.
			
\section{Preliminaries}
In this section, we briefly introduce the training of SFNN with focal loss and Adam optimizer, the TWD-SFNN model, and the STWD model.
			
\subsection{Training SFNN  with focal loss and Adam optimizer}
\label{focal_adam}
    The parameters connected to the hidden layer are one of the factors affecting the performance of SFNN \cite{Belciug2021ESWA}. Therefore, the focal loss and Adam optimizer for network parameters are introduced as follows.
		
	Since the labels in the classification dataset may have category imbalance, to better deal with the practical problems, the loss function of SFNN adopts the focal loss \cite{Zhou2020TIP}, which also makes the model focus more on the difficult-to-classify instances in training, and is calculated according to Equation (\ref{equ1_FL}). %Its calculation Equation is as follows.
		
		\begin{spacing}{0.7} 
			%\vspace{0em}        
			\begin{normalsize} 
				\begin{equation}
				L_{\rm FL} = \left\{
				\begin{array}{lr} 
				-\delta(1-\hat{y})^{\theta}log\hat{y}     &  y=+1 \\
				-(1-\delta)\hat{y}^{\theta}log(1-\hat{y}) &  y=-1 \\  
				\end{array} \right.  \label{equ1_FL}	
				\end{equation}
			%\vspace{0em}   
			\end{normalsize} 
		\end{spacing} 
	
	\noindent  where $ \hat{y} $ is the predicted value of the neural network; $ \delta $  is a factor that balances the proportion of positive and negative labels; and $ \theta$  $(\theta \ge 0) $ is a parameter that adjusts the weight reduction rate of easy-to-classify instances, and the greater the $ \theta$, the greater the weight reduction rate. In this paper, $ \delta $ is the ratio of the number of each label to the total number of labels, and $ \theta$ is set to two which is the same as \cite{Cheng2021INS}. In addition, to prevent the overfitting of SFNN, the cost function is calculated according to Equation (\ref {equ2_L_NN}).
	
		\begin{spacing}{0.5}   
			%\vspace{0em}         
			\begin{normalsize} 
				\begin{equation}
				L_{\rm NN} = L_{\rm FL} + \frac{\lambda_{\rm SFNN}}{2} \times (\Vert \mathbf W_{1} \Vert^{2} + \Vert \mathbf b_{1} \Vert^{2} + \Vert \mathbf W_{2} \Vert^{2} + \Vert \mathbf b_{2} \Vert^{2})  \label{equ2_L_NN}	
				\end{equation}
			\end{normalsize} 
		    %\vspace{0em} 		    
		\end{spacing} 
	
	\noindent where $ \lambda_{\rm SFNN} $  is a regularization factor of SFNN;  $ \mathbf W_{1} $ and  $ \mathbf b_{1} $ are the weight and bias connecting the input layer to the hidden layer of SFNN, respectively; $ \mathbf W_{2} $  and $ \mathbf b_{2} $  are the weight and bias from the hidden layer to the output layer of SFNN, respectively.
	
	Since the gradient of network parameters may have oscillation in the optimization, to better optimize $ \mathbf W_{1} $, $ \mathbf b_{1} $,$  \mathbf W_{2} $, and $ \mathbf b_{2} $, SFNN selects the Adam optimizer \cite{Khan2020JAS}. Without loss of generality, we take $ \mathbf W_{1} $ as an example which is calculated according to Equation (\ref{equ3_Adam}).		
	
		\begin{spacing}{0.5}  
				%\vspace{0em}
				\begin{normalsize} 
					\begin{equation}
					\begin{split}
					\mathbf V_{\rm d \mathbf W_{1}}^{(h)} &= \rho_{1} \times \mathbf V_{\rm d \mathbf W_{1}}^{(h-1)} + (1-\rho_{1}) \times {\rm d} \mathbf W_{1}^{(h)}  \\
					\mathbf S_{\rm d \mathbf W_{1}}^{(h)} &= \rho_{2} \times \mathbf S_{\rm d \mathbf W_{1}}^{(h-1)} + (1-\rho_{2}) \times ({\rm d} \mathbf W_{1}^{(h)})^{2} \\
					\hat{\mathbf V}_{\rm d \mathbf W_{1}}^{(h)} &= \frac{\mathbf V_{\rm d \mathbf W_{1}}^{(h)}}{1-\rho_{1}^{(h)}} \\
					\hat{\mathbf S}_{\rm d \mathbf W_{1}}^{(h)} &= \frac{\mathbf S_{\rm d \mathbf W_{1}}^{(h)}}{1-\rho_{2}^{(h)}} \\
					\mathbf W_{1}^{(h)} &=   \mathbf W_{1}^{(h-1)} - \mu \times \frac{\hat{\mathbf V}_{\rm d \mathbf W_{1}}^{(h)}}{ \sqrt{\hat{\mathbf S_{\rm d \mathbf W_{1}}^{(h)}}} + \tau}  \label{equ3_Adam} 	
				    \end{split}
				    \end{equation}				
		 	    \end{normalsize}
		    	%\vspace{0em}	 	    	
		\end{spacing}  

	\noindent where $ \mathbf V_{\rm d \mathbf W_{1}}^{(h)} $ and $ \mathbf S_{\rm d \mathbf W_{1}}^{(h)} $ are updated first-order momentum estimates and second-order momentum estimates of $ \mathbf W_{1} $, respectively; $ \hat{\mathbf V}_{\rm d\mathbf W_{1}}^{(h)} $ and $ \hat{\mathbf S}_{\rm d\mathbf W_{1}}^{(h)} $  are deviation corrections of $ \mathbf V_{\rm d \mathbf W_{1}}^{(h)} $ and $ \mathbf S_{\rm d \mathbf W_{1}}^{(h)} $, respectively; $ \mathbf W_{1}^{(h)} $ is the updated value of $ \mathbf W_{1} $ at the $ h $-th moment; $ \rho_{1} $ and $ \rho_{2} $ denote the exponential decay rate of the first-order and second-order momentum estimates, respectively; $ \mu $ is the learning rate; $ \tau $ is used to avoid that the denominator of $ \mathbf W_{1}^{(h)} $ is zero, and $ \tau $ is set to $ 10^{-8} $. For convenience, we mark $ \mathbf W_{1}^{(h)} $ returned by Adam as $ \overline{\mathbf W}_{1} $. 
  To update the parameters of the rest network , we replace $ \mathbf W_{1} $ in Equation \eqref{equ3_Adam} with $ \mathbf b_{1} $, $ \mathbf W_{2} $, and $ \mathbf b_{2} $.
  
  %  It is worth noting that when we update the values for $ \mathbf b_{1} $, $ \mathbf W_{2} $, and $ \mathbf b_{2} $, we replace $ \mathbf W_{1} $ in Equation \eqref{equ3_Adam} with $ \mathbf b_{1} $, $ \mathbf W_{2} $, and $ \mathbf b_{2} $.
  
  %It is worth noting that when we need to calculate the updated value for the rest network parameters, including $ \mathbf b_{1} $, $ \mathbf W_{2} $, and $ \mathbf b_{2} $, we can replace only $ \mathbf W_{1} $ in Equation \eqref{equ3_Adam} with $ \mathbf b_{1} $, $ \mathbf W_{2} $, and $ \mathbf b_{2} $, respectively.

\subsection{TWD-SFNN} 	   

TWD-SFNN \cite{Cheng2021INS} adopts the trisecting-and-acting model \cite{Wang2022INS,Yi2022ESWA,Zhan2022TFS} to find the number of hidden layer nodes. In essence, TWD-SFNN is a model for dealing with the binary classification problem. Suppose we have a binary classification dataset $ \{ \mathbf x_{i}, \mathbf y_{i} \}_{i = 1}^{d}$, where $ \mathbf x_{i} = (\mathbf x_{i1}, \mathbf x_{i2}, \cdots, \mathbf x_{im}) ^{ \mathbf T } \in \mathbf{R}^{m} $, $ \mathbf y_{i} \in \{ +1, -1 \}$. We also suppose that TWD-SFNN has total $d$  instances and $ u$ $(1\le u\le d) $ unclassified instances, where the unclassified instances are the instances that are partitioned into the region of boundary (BND), and $u$ is the total number of instances minus the total number of instances classified into the regions of positive(POS) and negative(NEG). TWD-SFNN considers the total expected risk of classifying $u$ instances.  
		
		\begin{spacing}{0.5}
		%\vspace{0em}
		\begin{normalsize} 
		    \begin{equation}
		    \begin{split}
		    %\begin{gather*}
			Risk_{{\rm{TWD-SFNN}}} =  \min_{(\alpha,\beta), {\rm or} \gamma}\Bigg (\sum_{o=1 \atop p_{o} \ge \alpha , {\rm or}  p_{o} \ge \gamma}^{u_{p}}   \lambda_{PP} \times p_{o} + \lambda_{PN} \times (1-p_{o})   \\ +  \varepsilon \times \sum_{k=1 \atop \beta < p_{k} < \alpha }^{u_{b}}  \lambda_{BP} \times p_{k} +  \lambda_{BN} \times (1-p_{k})  \\  +  \sum_{j=1 \atop p_{j} \le \beta , {\rm or} p_{j} < \gamma}^{u_{n}}  \lambda_{NP} \times p_{j} +\lambda_{NN} \times (1- p_{j}) \Bigg )  \\
			{\rm s.t.}  \quad    0 < \beta < \gamma < \alpha < 1, \quad   \varepsilon \ge 1, \quad u=u_{p}+u_{b}+u_{n}
			%\end{gather*}  
			\end{split}
			\end{equation}
		\end{normalsize}
		%\vspace{0em}
	    \end{spacing}
    
		\noindent where $ \varepsilon $ is a penalty factor and $ \varepsilon \ge 1 $; $ p_{e} $ is the conditional probability that the $ e $-th $ (1 \le e \le u) $ instance belongs to the positive labels set $ Y $, i.e., $ p_{e} = \frac{\left| Y \cap [x_{e}] \right| } {\left| [x_{e}] \right| } $,  $ [x_e] $ is the equivalent class of the $ e $-th   instance, and the equivalence class is a set of instances with the same features; $ u_{p} $,  $ u_{b} $, and  $ u_{n} $ are the numbers of instances partitioned into POS, BND, and NEG, respectively; $ \lambda_{*P} $ and $ \lambda_{*N} $ are the costs of decision-making when instances belong to the POS and NEG, respectively; $ (\alpha, \beta, \gamma) $ are the thresholds, and can be calculated as follows.

		\begin{spacing}{0.5}
			%\vspace{0em}
			\begin{normalsize}   
			    \begin{gather}
				%\begin{gather*}
				\alpha = \frac{\lambda_{PN}-\lambda_{BN}}{(\lambda_{PN}-\lambda_{BN})+(\lambda_{BP}-\lambda_{PP})} \\
				\beta  = \frac{\lambda_{BN}-\lambda_{NN}}{(\lambda_{BN}-\lambda_{NN})+(\lambda_{NP}-\lambda_{BP})} \\
				\gamma = \frac{\lambda_{PN}-\lambda_{NN}}{(\lambda_{PN}-\lambda_{NN})+(\lambda_{NP}-\lambda_{PP})}
				%\end{gather*}
				\end{gather}
			\end{normalsize}
		    %\vspace{0em}
		\end{spacing}

		To meet the requirements of the three-way decisions on data types, TWD-SFNN adopts $ k $-means++ to discretize the numerical data into categorical data. When the number of misclassification instances is more than that of clusters specified in $k$-means++, $ \alpha $ and $ \beta $ are employed and the classification criteria are as follows.

		(P) if $ p_{e} \ge \alpha $, then $ [x_e] \in $ POS($ Y $); 

		(B) if $ \beta < p_{e} < \alpha $, then $ [x_e] \in $ BND($ Y $); 

		(N) if $ p_{e} \le \beta  $, then $ [x_e] \in $ NEG($ Y $). 

		Whereas when the number of instances in the misclassification set is no more than these clusters specified in $ k $-means++, $ \gamma $ is adopted in TWD-SFNN and the classification criteria are as follows.
    
		(P) if $ p_{e} \ge \gamma $, then $ [x_e] \in $ POS($ Y $);  
    
		(N) if $ p_{e} < \gamma  $, then $ [x_e] \in $ NEG($ Y $). 
    
	TWD-SFNN settles the penalty factor $  \varepsilon(\varepsilon \ge 1) $ in the decision-making risk function, which means that TWD-SFNN can avoid too many training instances being classified into BND, so as to avoid the overfitting of the model. In addition, if BND is not empty, it is necessary to increase the number of hidden layer nodes of TWD-SFNN until BND is empty.

\subsection{ STWD } \label{STWD}
{ The sequential strategy is an essential application of multi-granularity \cite{Ma2022TKDD,Ma2022Neurocomputing} to three-way decisions. As information continues to be added and updated, a comprehensive decision-making scheme of STWD is gradually developed. STWD \cite{Wu2022TKDD,Xu2022INS,Zhang2021TSMCS} relies on the idea of multi-granularity  hierarchical processing. STWD has a decision table \cite{Qian2017IJAR} which is defined as $ S = \left(U, A = C \cup D, V = \{V_{a} \mid a \in A\}, I = \{I_{a} \mid a \in A\} \right) $, where $ U $ is a universe; $ A $ is an attribute set; $ C $ is a conditional attribute set, and $ D $ is a decision attribute set. For any attribute $ a $, $ V_{a} $ is the conditional attribute value, and $ I_{a} $ is the corresponding decision attribute value. }
	
{ Suppose $ U $ has $ t $ granular levels, denoted by $ \{ 1,2,\cdots,t\} $, which are added from the first level to the $ t $-th level step by step. Given a state set $ {\{X,\neg X}\} $, an action set $ \{{a_P},{a_B},{a_N}\}$, where $ X $ and $ \neg X $ represent different states that the instance $ x $ belongs to $ X $ and $ \neg X $, respectively; $ {a_P}$, $ {a_B}$, and $ {a_N} $ represent the accepted decision-making, the delayed decision-making, and the rejected decision-making, respectively. We suppose that when $ x \in X $ at the $ i $-th $ \left(1\le i\le t\right) $ level, $ \lambda_{PP}^{(i)} $, $ \lambda_{BP}^{(i)} $, and $ \lambda_{NP}^{(i)} $ represent the losses caused by taking $ {a_P} $, $ {a_B}$, and $ {a_N} $, respectively, while $ \lambda_{PN}^{(i)}$, $ \lambda_{BN}^{(i)} $, and $ \lambda_{NN}^{(i)} $ represent the corresponding losses when $ x \in \neg X $ at the $ i $-th level. Thus, under the action of $ {a_P} $, $ {a_B} $, and $ {a_N} $, the expected losses at the $ i $-th level of STWD are shown as follows.}
	
            \begin{spacing}{0.45} % 行间距   
            	\begin{normalsize}  
            	    \begin{gather}  % 居中对齐
            		%\begin{gather*}  % 居中对齐
            		{ R^{(i)}({a_P} \mid [x]_{R}) = {\lambda_{PP}^{(i)}} \times p^{(i)}(X \mid [x]_{R}) + {\lambda_{PN}^{(i)}} \times p^{(i)}(\neg X \mid [x]_{R}) } \\
            		{ R^{(i)}({a_B} \mid [x]_{R}) = {\lambda_{BP}^{(i)}} \times p^{(i)}(X \mid [x]_{R}) + {\lambda_{BN}^{(i)}} \times p^{(i)}(\neg X \mid [x]_{R}) }\\
            		{R^{(i)}({a_N} \mid [x]_{R}) = {\lambda_{NP}^{(i)}} \times p^{(i)}(X \mid [x]_{R}) + {\lambda_{NN}^{(i)}} \times p^{(i)}(\neg X \mid [x]_{R}) }
            		%\end{gather*}
            		\end{gather}
            	\end{normalsize}
            \end{spacing}
	
	{ \noindent where $ p^{(i)}(X \mid [x]_{R}) $ and $ p^{(i)}(\neg X \mid [x]_{R}) $ represent the conditional probability that the equivalent class $ [x]_{R} $ at the $ i $-th $ \left(1\le i\le t\right) $ level belongs to $ X $ and $ \neg X $, respectively. For convenience, we mark $ p^{(i)}(X \mid [x]_{R}) $ as $ p^{(i)} $. According to Bayes decision criterion, the following optimal classification decision rules can be obtained.}
	
	{ (P) If $ R^{(i)}({a_P} \mid [x]_{R}) \le R^{(i)}({a_B} \mid [x]_{R}) $ and $ R^{(i)}({a_P} \mid [x]_{R}) \le R^{(i)}({a_N} \mid [x]_{R}) $, then $ x \in {\rm POS}^{(i)}(X) $ ; }
	
	{ (B) If $ R^{(i)}({a_B} \mid [x]_{R}) \le R^{(i)}({a_P} \mid [x]_{R}) $ and $ R^{(i)}({a_B} \mid [x]_{R}) \le R^{(i)}({a_N} \mid [x]_{R}) $, then $ x \in {\rm BND}^{(i)}(X) $ ; }
	
	{ (N) If $ R^{(i)}({a_N} \mid [x]_{R}) \le R^{(i)}({a_P} \mid [x]_{R}) $ and $ R^{(i)}({a_N} \mid [x]_{R}) \le R^{(i)}({a_B} \mid [x]_{R}) $, then $ x \in {\rm NEG}^{(i)}(X) $ . }
	
	{ Furthermore, STWD has a total order relationship over granular levels \cite{Zhan2022TFS}, i.e., $ [x]_{R_{1}} \preceq [x]_{R_{2}} \preceq \cdots \preceq [x]_{R_{t}} $, where $ [x]_{R_{1}} $ and $ [x]_{R_{t}} $ are the equivalent class of instance $ x $ calculated at the coarse-grained level and the fine-grained level, respectively. Moreover, STWD has a conditional probability $ p^{(i)} $ and thresholds $ (\alpha_{i}, \beta_{i})$ $( 1 \le i \le t-1 ) $ for the previous $ t $-1 levels and also has a conditional probability $ p^{(t)} $ and a threshold $ \gamma $ at the $t$-th level. These parameters satisfy the following conditions. }
		
		\begin{spacing}{0.5}
			%\vspace{0em}
			\begin{normalsize}       	       	
				\begin{gather*}
{ 0 < \beta_{1} \le \beta_{2} \le \cdots  \le \beta_{t-1} < \gamma < \alpha_{t-1} \le \cdots \le \alpha_{2} \le \alpha_{1} < 1  }
				\end{gather*}  
			\end{normalsize}  
		    %\vspace{0em}
		\end{spacing}     
			
{ Compared with $ (\alpha_{i}, \beta_{i}) $, STWD at least increases the range of POS and NEG regions when adopting $ (\alpha_{i+1}, \beta_{i+1}) $. Accordingly, the BND of STWD will be smaller and smaller at the ($i$+1)-th granular level. It indicates that the sequential thresholds of STWD gradually realize the classification of the delay decision region, thereby completing the partition of whole instances.}

%For example, the region in the $ i $-th level $ L_{i}$  $(1 \le i \le t-1) $  is partitioned into three disjoint parts.

{Since STWD learns three-way decision-making from the first level to the ($t-1$)-th level, given the $i$-th  ($1\le i \le t-1$) granular level $L_i$, $L_i$ can be classified into three disjoint regions: }
		
		\begin{spacing}{0.5}
			%\vspace{0em}
			\begin{normalsize}  
			    \begin{gather}
				%\begin{gather*}
				{ {\rm POS}_{(\alpha_{i}, \beta_{i})}(Y) = \{ x \in U_{i} \mid p^{(i)} \succeq_{i} \alpha_{i}  \} }  \\   	
				{ {\rm BND}_{(\alpha_{i}, \beta_{i})}(Y) = \{ x \in U_{i} \mid \beta_{i} \prec_{i} p^{(i)} \prec_{i} \alpha_{i}  \} }	\\
				{ {\rm NEG}_{(\alpha_{i}, \beta_{i})}(Y) = \{ x \in U_{i} \mid p^{(i)} \preceq_{i} \beta_{i}  \} }
				%\end{gather*} 
				\end{gather} 
			\end{normalsize}   
		    %\vspace{0em}
		\end{spacing}    
	 
		{ \noindent where $ {\rm POS}_{(\alpha_{i}, \beta_{i})}(Y) $, $ {\rm BND}_{(\alpha_{i}, \beta_{i})}(Y) $, and ${\rm NEG}_{(\alpha_{i}, \beta_{i})}(Y) $ are the POS, BND, and NEG at the $ i $-th level, respectively. As the number of levels increases, $ {\rm BND}_{(\alpha_{i}, \beta_{i})}(Y) $ becomes smaller. However, if $ {\rm BND}_{(\alpha_{t-1}, \beta_{t-1})}(Y) \neq \varnothing $, STWD carries out two-way decisions at the $ t $-th level, which means that the region in the $ t $-th level $ L_{t}$ is partitioned into two disjoint regions. }
		
		\begin{spacing}{0.5}
			%\vspace{0em}
			\begin{normalsize}       	       	
				\begin{gather}
				%\begin{gather*}
				{ {\rm POS}_{(\gamma)}(Y) = \{ x \in U_{t} \mid p^{(t)} \succeq \gamma  \} }  \\     	
				{ {\rm NEG}_{(\gamma)}(Y) = \{ x \in U_{t} \mid p^{(t)} \prec \gamma  \} }
				%\end{gather*}  
				\end{gather}
			\end{normalsize}  
		    \vspace{-3em}
	    \end{spacing}

\section{ Proposed model }

%In this section, we will discuss STWD-SFNN.
		%To promote the research of network topology
{In this section, we propose STWD-SFNN which dynamically determines the number of hidden layer nodes by adopting granularity levels and sequential thresholds to improve the performance of network topology. To show the difference between SFNN and STWD-SFNN, Figs. \ref{fig_framsfnn} and  \ref{fig_framework} show the framework of SFNN and STWD-SFNN, respectively. The network topology of SFNN refers to the number of hidden nodes $z$, and we calculate hidden nodes $z$ according to the empirical formula methods. However, these methods for determining network topology lack a reliable theoretical basis. Therefore, we propose STWD-SFNN to dynamically determine the number of hidden nodes. As shown in Fig.  \ref{fig_framework}, STWD-SFNN adopts the sequence of STWD to guide the growth of network topology. Moreover, Fig. \ref{fig_thresholds} illustrates the granulation of STWD-SFNN from coarse to fine. }

%To improve the performance of network topology, we propose STWD-SFNN, which dynamically determines the number of hidden layer nodes by adopting granularity levels and sequential thresholds. Fig. \ref{fig_framework} shows the framework of STWD-SFNN, and Fig. \ref{fig_thresholds} illustrates the granulation of STWD-SFNN from coarse to fine.

      \begin{figure}[h]   % 放在当前页的最上方      	
        	\centering
        	\small
        	\captionsetup{font={small}}
        	\includegraphics[width=0.45\textwidth]{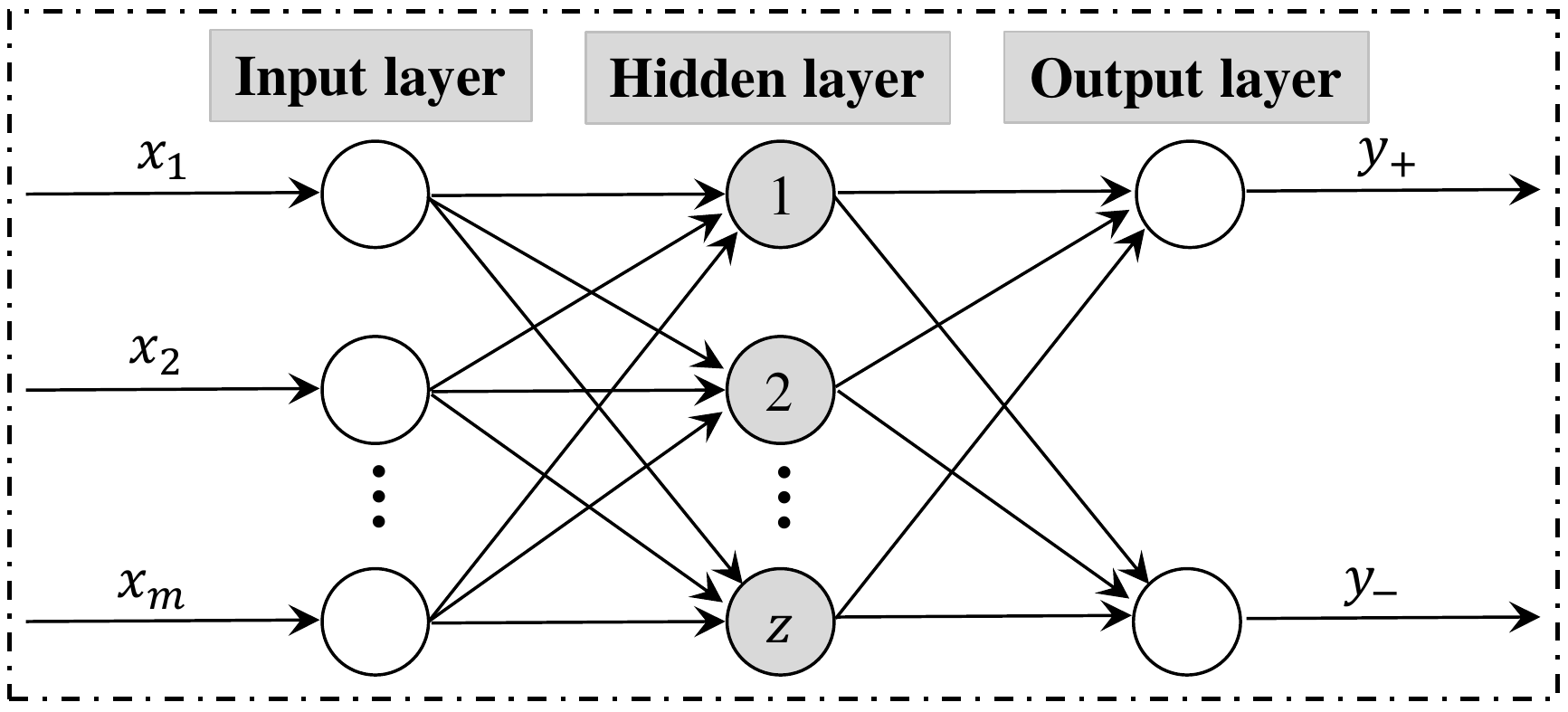}
    \caption{{ Framework of SFNN}}
        	\label{fig_framsfnn}
        \end{figure}

\begin{figure*}[h]   % 放在当前页的最下方
  		\centering
    		\small
    		\captionsetup{font={small}}
    		\includegraphics[width=1\textwidth]{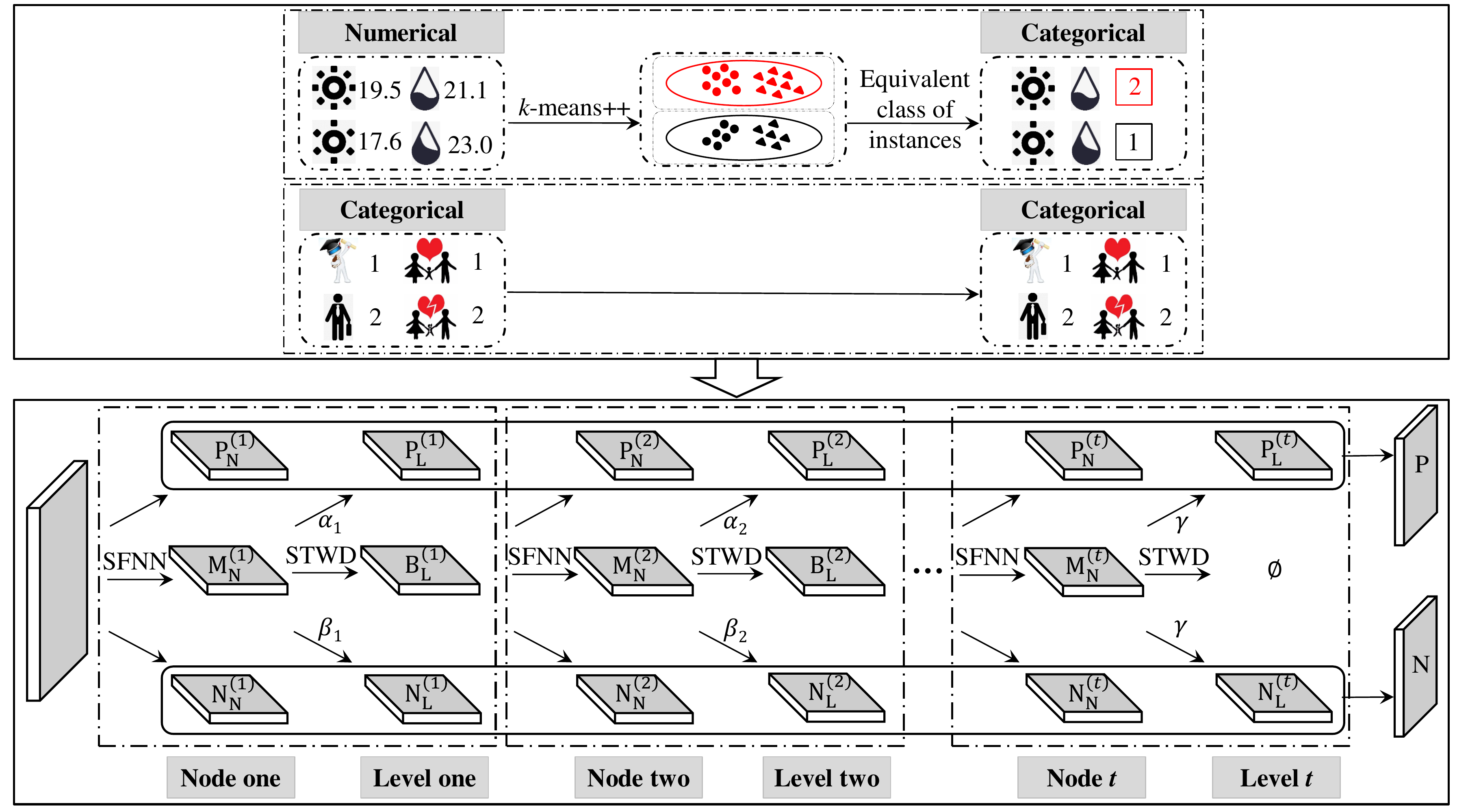}
\caption{{ Framework of STWD-SFNN. STWD-SFNN has two parts: discretion modular and training modular. If the instances contain numerical features, discretion modular adopts $k$-means++ to convert numerical data into discrete data to meet the requirements of STWD for data types. The training modular has $t$ levels. Each level has two parts: SFNN and STWD. SFNN is a network topology with one hidden layer node and gets correctly classified instances, i.e., $ {\rm P}^{(i)}_{\rm {N}} $ and $ {\rm N}^{(i)}_{\rm {N}} $, and misclassification instances $ {\rm M}^{(i)}_{\rm {N}} $.  The first $t$-1 part of STWD utilizes the strategy of three ways, and the last part applies the strategy of two ways.}}
    		\label{fig_framework}
    	\end{figure*}     
    
        \begin{figure*}[h]   % 放在当前页的最上方      	
        	\centering
        	\small
        	\captionsetup{font={small}}
        	\includegraphics[width=0.60\textwidth]{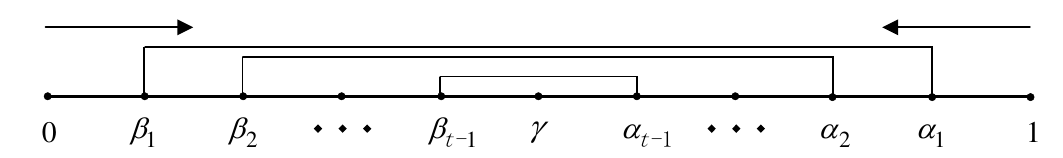}
    \caption{ Sequential thresholds of STWD-SFNN }
        	\label{fig_thresholds}
        \end{figure*}   
        
 Firstly, if the instances contain numerical features, such as light and humidity in Fig. \ref{fig_framework}, we adopt $k$-means++ \cite{Li2022KBS} to discretize the numerical data to meet the requirements of STWD for data types. Secondly, we initialize the network topology with one hidden layer node and get correctly classified instances, i.e., $ {\rm P}^{(i)}_{\rm {N}} $ and $ {\rm N}^{(i)}_{\rm {N}} $, and misclassification instances $ {\rm M}^{(i)}_{\rm {N}} $. Thirdly, we calculate the conditional probability of the instances belonging to the positive labels and the thresholds $ (\alpha_i, \beta_i)(1 \le i \le t-1)$ of the $ i $-th level. If the conditional probability is no less than $ \alpha_i $, the instances are classified into $ {\rm P}^{(i)}_{\rm {L}} $; if the conditional probability is not greater than $ \beta_i$, they are partitioned into $ {\rm N}^{(i)}_{\rm {L}} $; otherwise, they are classified into $ {\rm B}^{(i)}_{\rm {L}} $. It means that the information of the $i$-th level is not enough to support STWD-SFNN to make decisions, and a hidden layer node is added based on the current network topology, and the above process is repeated at the ($i$+1)-th level. To prevent the overfitting of STWD-SFNN, when $ {\rm B}^{(t-1)}_{\rm {L}} $ is not empty, we adopt the threshold $ \gamma $ at the $ t $-th level. If the conditional probability is no less than $\gamma$, the instances are classified into $ {\rm P}^{(t)}_{\rm {L}} $; otherwise, they are partitioned into $ {\rm N}^{(t)}_{\rm {L}} $. Finally, we have the regions of POS, BND, and NEG after $t$ turns of STWD-SFNN.        
 % by adopting Definition \ref{def5_regions}. 
    
        %Specifically, we initialize the network topology with one hidden layer node and return the misclassification instances predicted by SFNN. If the instances contain numerical features, such as light and humidity in Fig. \ref{fig_framework}, we adopt $ k $-means++ \cite{Li2022KBS} to discretize to meet the requirements of STWD for data types. In addition, we have the conditional probability of the instances belonging to the positive labels and the thresholds $ (\alpha_i, \beta_i)(1 \le i \le t-1)$ of the $ i $-th level. If the conditional probability is no less than $ \alpha_i $, the instances are classified into POS. If the conditional probability is not greater than $ \beta_i $, they are partitioned into NEG. Otherwise, it means that the information of the $ i $-th level is not enough to support STWD-SFNN to make decisions, and a hidden layer node is added on the basis of the current network topology, and the above process is repeated at the ($ i $+1)-th level. To prevent the overfitting of STWD-SFNN, if BND at the ($ t $-1)-th level is not empty, the threshold $ \gamma $ is used at the $ t $-th level. When the conditional probability is no less than $ \gamma $, the instances are classified into POS; otherwise, they are partitioned into NEG. Finally, POS and NEG of each granular level are merged as POS and NEG of STWD-SFNN, respectively.  

\subsection{ Related definitions }
			Suppose the granular levels of STWD-SFNN traverse from one to $ t $. Accordingly, the maximum number of hidden layer nodes of STWD-SFNN is $ t $.
			
			\begin{definition} \label{def1_DecicisonTable}
				Suppose there is a decision table with a number of hidden layer nodes  
				\vspace{-0.5em}	
				\begin{spacing}{0.5}  
					\begin{normalsize}  
						\begin{gather*}    		     
						T = \left(U, A = C \cup D, V = \{V_{a} \mid a \in A\}, I = \{I_{a} \mid a \in A\}, H \right)  
						\end{gather*}
					\end{normalsize}
					\vspace{-2em}	% -0.5
				\end{spacing} 
			\end{definition}
			
\noindent where $ U $, $ A $, $ C $, $ D $, $ V $, and $ I $ have the same meanings as in Section \ref{STWD}; $ H $ is a set composed of the hidden layer nodes, i.e., $ H = \{ h_{1}, h_{2}, \cdots, h_{t}\} $ and $ h_{i} $ is the $ i $-th $ ( 1 \le i \le t ) $ hidden layer node added.

	According to STWD, we have a action set $ \{ a_{P}, a_{B}, a_{N} \}$, which can classify the instances into POS, BND, and NEG, respectively. Hence, we have the following definitions.
			
			\begin{definition} \label{def2_result_cost}
				Suppose $ T $ has $ t $ granular levels, the result cost matrix for the $ i $-th $ (1 \le i \le t) $ granular level is  \\
				$$ 
				\mathbf \Lambda^{(i)}  =
				\left(
				\begin{matrix}
				\lambda_{PP}^{(i)} & \lambda_{BP}^{(i)} & \lambda_{NP}^{(i)} \\
				\lambda_{PN}^{(i)} & \lambda_{BN}^{(i)} & \lambda_{NN}^{(i)}
				\end{matrix}
				\right)
				$$	
			\end{definition} 
			
			\noindent where $ \lambda_{PP}^{(i)} $, $ \lambda_{BP}^{(i)} $, and $ \lambda_{NP}^{(i)} $ are the result costs of taking actions $ a_{P} $, $ a_{B} $, and $ a_{N} $ for the positive labels, respectively; while $ \lambda_{PN}^{(i)} $, $ \lambda_{BN}^{(i)} $, and $ \lambda_{NN}^{(i)} $ are the costs of taking actions $ a_{P} $, $ a_{B} $, and $ a_{N} $ for the negative labels, respectively. It should be noted that the various elements of $ \mathbf \Lambda^{(i)} $ at $ i $-th level need to satisfy $ 0 \le \lambda_{PP}^{(i)} < \lambda_{BP}^{(i)} < \lambda_{NP}^{(i)} < 1 $, $ 0 \le \lambda_{NN}^{(i)} < \lambda_{BN}^{(i)} < \lambda_{PN}^{(i)} < 1 $ , and $ (\lambda_{BN}^{(i)} - \lambda_{NN}^{(i)}) \times (\lambda_{BP}^{(i)} - \lambda_{PP}^{(i)}) < (\lambda_{PN}^{(i)} - \lambda_{BN}^{(i)}) \times (\lambda_{NP}^{(i)} - \lambda_{BP}^{(i)}) $.
			
			\begin{example} \label{exp1_result_cost}
				Suppose STWD-SFNN has three granular levels, i.e., $ t $ = 3. According to the conditions satisfied by the loss function, i.e., $ 0 \le \lambda_{PP}^{(i)} < \lambda_{BP}^{(i)} < \lambda_{NP}^{(i)} < 1 $, $ 0 \le \lambda_{NN}^{(i)} < \lambda_{BN}^{(i)} < \lambda_{PN}^{(i)} < 1 $, and $ (\lambda_{BN}^{(i)} - \lambda_{NN}^{(i)}) \times (\lambda_{BP}^{(i)} - \lambda_{PP}^{(i)}) < (\lambda_{PN}^{(i)} - \lambda_{BN}^{(i)}) \times (\lambda_{NP}^{(i)} - \lambda_{BP}^{(i)}) $, we can achieve the results by randomizing the elements of the result cost matrix in the range of [0,1)  at the first granular level, i.e., $ \mathbf \Lambda^{(1)} = \begin{bmatrix} 
				0      & 0.1506 & 0.9021 \\
				0.4592 & 0.1249 & 0 
				\end{bmatrix} $. Meanwhile, the result costs at the other two granular levels are $ \mathbf \Lambda^{(2)} = \begin{bmatrix} 
				0      & 0.4617 & 0.5962 \\
				0.6740 & 0.1344 & 0
				\end{bmatrix} $ and  $ \mathbf \Lambda^{(3)} = \begin{bmatrix} 
				0      & 0.3626 & 0.7064 \\
				0.7664 & 0.3727 & 0
				\end{bmatrix} $, respectively.
			\end{example}

			\begin{definition} \label{def3_alpha_beta}
				Suppose $T$ has $t$ granular levels, the thresholds $ (\alpha_{i}, \beta_{i}) $ of the $ i $-th $ (1 \le i \le t-1) $ level are  \\
				\begin{spacing}{0.5}
					\vspace{-2em}
					\begin{normalsize}      
						\begin{gather}
						%\begin{gather*}
						\alpha_{i} = \frac{\lambda_{PN}^{(i)}-\lambda_{BN}^{(i)}}{(\lambda_{PN}^{(i)}-\lambda_{BN}^{(i)})+(\lambda_{BP}^{(i)}-\lambda_{PP}^{(i)})} \\
						\beta_{i}  = \frac{\lambda_{BN}^{(i)}-\lambda_{NN}^{(i)}}{(\lambda_{BN}^{(i)}-\lambda_{NN}^{(i)})+(\lambda_{NP}^{(i)}-\lambda_{BP}^{(i)})}
						%\end{gather*}
						\end{gather}
					\end{normalsize}
				\end{spacing}
			\end{definition}  
			
			\begin{definition} \label{def4_gamma}
				Suppose $ T $ has $ t $ granular levels, the threshold $ \gamma $ of the $ t $-th granular level is \\
				\begin{spacing}{0.5}
					\vspace{-2em}	
					\begin{normalsize}      
						\begin{gather}
						%\begin{gather*}
						\gamma = \frac{\lambda_{PN}^{(t)}-\lambda_{NN}^{(t)}}{(\lambda_{PN}^{(t)}-\lambda_{NN}^{(t)})+(\lambda_{NP}^{(t)}-\lambda_{PP}^{(t)})}
						%\end{gather*}
						\end{gather}
					\end{normalsize}  
				\end{spacing} 
				\vspace{-0.5em} 
			\end{definition}  
			
\begin{example} \label{exp2_alpha_beta_gamma}
In this example, we adopt the result cost matrices in Example \ref{exp1_result_cost}. According to Definitions \ref{def3_alpha_beta} and \ref{def4_gamma}, the thresholds of STWD-SFNN at each granular level are  $ (\alpha_{1}, \beta_{1}) $ = (0.6894, 0.1425),  $ (\alpha_{2}, \beta_{2}) $ = (0.5389, 0.5016), and  $ \gamma $ = 0.5204, respectively. Since the thresholds of STWD-SFNN satisfy $ \beta_{1} < \beta_{2} < \gamma < \alpha_{2} < \alpha_{1} $, we retain  the result cost matrices of STWD-SFNN at each granular level. It should be noted that if the thresholds of STWD-SFNN do not meet the sequential property at a certain granular level, we need to reinitialize the result cost matrix at the corresponding level. 
\end{example}

			{
			\begin{definition} \label{def5_regions}
				Suppose $ T $ has $ t $ granular levels, the regions of POS, BND, and NEG after $ t $ turns are denoted as follows \\
				\begin{spacing}{0.5}
					\vspace{-2em}
					\begin{normalsize}
					    \begin{gather}
						%\begin{gather*}
						{\rm POS} = \bigcup_{i=1}^t [ {\rm P}^{(i)}_{\rm {N}} \cup {\rm P}^{(i)}_{\rm {L}} ] \\
						{\rm BND} = {\rm M}^{(t)}_{\rm {N}} \cup {\rm B}^{(t)}_{\rm {L}} \\
						{\rm NEG} = \bigcup_{i=1}^t  [ {\rm N}^{(i)}_{\rm {N}} \cup {\rm N}^{(i)}_{\rm {L}} ]               	%\end{gather*} 
						\end{gather}
					\end{normalsize} 
				\end{spacing}
				\vspace{-1em}
			\end{definition} }
			
			\noindent where $ {\rm P}^{(i)}_{\rm {N}} $ and $ {\rm P}^{(i)}_{\rm {L}} $ represent the set of instances classified into POS by SFNN and STWD at the $ i $-th turn, respectively; 
			$ {\rm M}^{(t)}_{\rm {N}} $ and $ {\rm B}^{(t)}_{\rm {L}} $ represent the set of instances misclassified by SFNN and the set of instances partitioned into BND by STWD at the $ t $-th turn, respectively; $ {\rm N}^{(i)}_{\rm {N}} $ and $ {\rm N}^{(i)}_{\rm {L}} $ represent the set of instances classified into NEG by SFNN and STWD at the $ i $-th turn, respectively. Since the threshold $ \gamma $ corresponding to the $ t $-th level can classify the instances of $ {\rm B}^{(t)}_{\rm {L}} $ into $ {\rm P}^{(t)}_{\rm {L}} $ and $ {\rm N}^{(t)}_{\rm {L}} $, we have $ {\rm BND} = \varnothing $ at the $ t $-th turn. 
			
	To describe the process costs of granularity \cite{Li2020ASC,Zhang-gs-2021KBS}, STWD-SFNN considers the test cost and delay cost \cite{Fang2020INS,Yang2022KBS,Yang2017KBS}.            
	{		
    \begin{definition}  
    \label{def6_process_costs}
	Suppose $ T $ has $ t $ granular levels, the process cost at the $ t $-th level is marked as $ Cost_{{\rm P}_t} $, and\\
				\begin{spacing}{0.5}
					\vspace{-1.5em}
					\begin{normalsize}      
						\begin{gather}
						%\begin{gather*}
						Cost_{{\rm P}_t} = (Cost_{{\rm PT}_t}, Cost_{{\rm PD}_t})  \\
						Cost_{{\rm PT}_t} = \sum_{i=1}^{t} m_{i} \times Cost_{{\rm PPT}_i}     \\
						Cost_{{\rm PD}_t} = \max_{i=1,2,\cdots,t} m_{i} \times Cost_{{\rm PPD}_i}        %\end{gather*}          
						\end{gather}
					\end{normalsize}
					\vspace{-2em} 
				\end{spacing}
			\end{definition} } 
			
	\noindent where $ Cost_{{\rm PT}_t} $ and $ Cost_{{\rm PD}_t} $ represent the test cost and delay cost at the $ t $-th granular level, respectively; $ m_{i} $ is the number of instances at the $ i $-th $ (1 \le i \le t) $ level and $ m_{i} > 0 $; $ Cost_{{\rm PPT}_i} $ and $ Cost_{{\rm PPD}_i} $ are the unit test cost and the unit delay cost at the $ i $-th $ (1 \le i \le t) $ level, respectively. It should be noted that when BND is not empty, the model continues to add a new turn of learning which produces the test and delay costs. Since the costs of test and delay are both positive real numbers and $ m_{i} > 0 $, the unit test cost and the unit delay cost need to satisfy $ Cost_{{\rm PPT}_i} > 0 $ and $ Cost_{{\rm PPD}_i} > 0 $, respectively. Meanwhile, since the dataset has no data on unit test cost and unit delay cost, and $ m_{i} $ gradually decreases as the model learns, to facilitate the observation of the more general correlation between the costs of test and delay and the granular layers, we can generate unit test costs and unit delay costs randomly and incrementally in the range of positive real numbers.
			
\begin{example} \label{exp3_process_costs}
	Suppose we have a two-class classification dataset with 10 instances shown in Table \ref{tab_example_data}, where the training data is $ \{ x_{1}, x_{2}, x_{3}, x_{4}, x_{5}, x_{6} \} $, the validation data is $ \{ x_{7}, x_{8} \} $, and the testing data is $ \{ x_{9}, x_{10} \} $. Suppose STWD-SFNN has at most three granular levels, and the unit test cost and unit delay cost at each level are $ [Cost_{{\rm PPT}_1}, Cost_{{\rm PPT}_2}, Cost_{{\rm PPT}_3}] = [1,2,3] $ and $ [Cost_{{\rm PPD}_1}, Cost_{{\rm PPD}_2}, Cost_{{\rm PPD}_3}] = [1,2,3] $, respectively. The calculation of test cost and delay cost is as follows. The instances classified into $ {\rm P}^{(1)}_{\rm {L}} $, $ {\rm B}^{(1)}_{\rm {L}} $, and $ {\rm N}^{(1)}_{\rm {L}} $ at the first granular level are $ \varnothing $, $ \{ x_{4}, x_{5} \} $, and $ \{ x_{1} \} $, respectively. Meanwhile, suppose the two instances $ \{ x_{4}, x_{5} \} $ are partitioned into $ {\rm N}^{(2)}_{\rm {L}} $ at the second level (Details will be shown in Section \ref{illstrativeexample}). Thus, STWD-SFNN has two granular levels, since  $ {\rm B}^{(2)}_{\rm {L}} $ is  empty. Hence, $ Cost_{{\rm PT}_1} $= 3 $ \times $ 1 = 3  and $ Cost_{{\rm PD}_1} $ = 3 $ \times $ 1 = 3 at the first level according to Definition \ref{def6_process_costs}. Similarly, $ Cost_{{\rm PT}_2} $ = 3 $ \times $ 1 + 2 $ \times $ 2 = 7 and  $ Cost_{{\rm PD}_2} = \max $ (3 $ \times $ 1, 2 $ \times $ 2) = 4  at the second level.    	 
\end{example}
		
    	 \begin{table}[ht]
    	 	\vspace{0em}
    	 	\setlength{\abovecaptionskip}{0cm}    % 段前
    	 	\setlength{\belowcaptionskip}{-0.2cm} % 段后
    	 	\small
    	 	\captionsetup{font={small}}
    	 	\newcommand{\tabincell}[2]{\begin{tabular}{@{}#1@{}}#2\end{tabular}}
    	 	\caption{Two-class dataset }
    	 	\centering{
    	 		\begin{tabular}{cccccc}\hline
			&$a_1$  &$a_2$  &$a_3$ &$a_4$ &D    \\ \hline
			$x_{1}$ &0.7415	&0.5407	&0.5795	&0.9009	&2 \\
			$x_{2}$ &0.6844	&0.3210	&0.0471	&0.3700	&1 \\  
			$x_{3}$ &0.7718	&0.0912	&0.4874	&0.5308	&1 \\ 
			$x_{4}$ &0.0818	&0.4263	&0.0354	&0.0621	&1 \\ 
			$x_{5}$ &0.5596	&0.4643	&0.3585	&0.3189	&2 \\
			$x_{6}$ &0.6397	&0.6535	&0.7739	&0.6809	&1 \\ 
			$x_{7}$ &0.7425 &0.0989	&0.7429	&0.4131	&1 \\	
			$x_{8}$ &0.9419	&0.5958	&0.4474	&0.7536	&2 \\ 		
			$x_{9}$ &0.4992 &0.2212	&0.9525	&0.4176	&1 \\
			$x_{10}$&0.2990	&0.4796	&0.1559	&0.7456	&2 \\  \hline					
    	 		\end{tabular}}
     		\vspace{-1em}
     		\label{tab_example_data}
     	\end{table}

\subsection{ STWD-SFNN model }
    	      
		%\subsubsection{ STWD-SFNN framework }
 % \subsubsection{Framework }

    	To prevent the overfitting of STWD-SFNN, we consider the $ L_{2} $ regularization factor of the loss function in SFNN and the penalty factor $ \varepsilon $ of the risk function in STWD. Hence, STWD-SFNN has three parts: preprocessing of data, SFNN, and STWD.    
    	 	
        %A. Preprocessing of data
\subsubsection{ Preprocessing of data }
    	
The instances processed by STWD \cite{Yao2019KBS} are usually categorical data, such as teachers and doctors. However, numerical data, such as the temperature and humidity of the weather, need to be processed to meet the calculation requirements of STWD. Therefore, $ k $-means++ \cite{Li2022KBS} is adopted to discretize the numerical data into the categorical data. Firstly, we randomly select one instance as the initial clustering center. Secondly, we calculate the distance between the cluster center and the remaining instances, and calculate the probability that each instance is selected as the next cluster center. Then, the next clustering center is selected according to the roulette method, and the discretization process is stopped until $ k $ clustering centers are found. Finally, the instances in each cluster are discretized into the same category. 

        %B. SFNN
    \subsubsection{ SFNN }
    	
Suppose STWD-SFNN contains $ u$ $(1 \le u \le d) $ unclassified instances. Firstly, we randomly generate the network parameters with $ i $-th hidden layer node, including  $ \mathbf w_{1}^{(i)} $,  $ \mathbf b_{1}^{(i)} $, $ \mathbf w_{2}^{(i)} $, and  $ \mathbf b_{2}^{(i)} $. To guarantee the repeatability of the experimental results, we set the seed of the random number generator to rng(0). Secondly, we adopt the focal loss function to guide the forward propagation learning, and calculate the learning results with the current network parameters. Thirdly, we apply the Adam optimizer to guide the error backpropagation, adaptively find the optimal network parameters, including $ \overline{\mathbf w}_{1}^{(i)} $, $ \overline{\mathbf b}_{1}^{(i)} $, $ \overline{\mathbf w}_{2}^{(i)} $, and  $ \overline{\mathbf b}_{2}^{(i)} $. Finally, STWD-SFNN with topology of $ i $-th hidden layer node yields the positive instances $ {\rm P}^{(i)}_{\rm {N}} $, the misclassification instances $ {\rm M}^{(i)}_{\rm {N}} $, and the negative instances $ {\rm N}^{(i)}_{\rm {N}} $, where the misclassification instances are also called the difficult-to-classify instances.

    	%Suppose the current STWD-SFNN model contains $ u (1 \le u \le d) $ unclassified instances. Firstly, we establish the structure of $ i $ hidden layer nodes. Secondly, we adopt the focal loss function to guide the forward propagation learning, and calculate the learning results of the model with the current network parameters. Thirdly, we apply the Adam optimizer to guide the error backpropagation, adaptively find the current optimal network parameters. Finally, the current STWD-SFNN model yields the misclassification instances based on a topology of $ i $ hidden layer nodes, and these instances comprise the misclassification dataset.
    	
    	%C. STWD
 \subsubsection{ STWD }
    	
  For the difficult-to-classify instances in the network, we first count the label equivalence class $ Y $ that the instances belong to the positive labels and the instance equivalence class $ [x]_{e}^{(i)} $. Next, we calculate the conditional probability that the instances belong to the positive labels according to Equation (\ref{equ4_P}).
   
    	\begin{spacing}{0.5} 
    		\begin{normalsize}
    			%\vspace{-2em}
    			\begin{equation}
    			p_{e}^{(i)} = \frac{\left| Y \cap [x]_{e}^{(i)} \right|}{\left| [x]_{e}^{(i)} \right|} \label{equ4_P}  
    			\end{equation}
    			%\vspace{-1.5em}
    		\end{normalsize}	
        \end{spacing}

    \noindent  In addition, according to Definitions \ref{def2_result_cost} and \ref{def3_alpha_beta}, we can index the result cost matrix $ \mathbf \Lambda^{(i)}$  and the thresholds $  (\alpha_{i}, \beta_{i}) $ at the $ i $-th $(1 \le i \le t-1)$ level. Since the risk varies in different regions, STWD-SFNN considers the total decision risk associated with classifying $ u $ instances. The risk is calculated according to 
    Equation (\ref {equ5_Risk_alpha_beta}).
 
    	\begin{spacing}{0.5} 
    		\begin{normalsize}       	       	
    			\begin{equation}
    			\begin{split} 
    			Risk^{(i)} = \min_{(\alpha_{i}, \beta_{i})} \Bigg (\sum_{o=1 \atop p_{o}^{(i)}\ge \alpha_{i} }^{u_{p}}  \lambda_{PP}^{(i)} \times p_{o}^{(i)} + \lambda_{PN}^{(i)} \times (1-p_{o}^{(i)})   \\
    			+  \varepsilon \times \sum_{k=1 \atop \beta_{i} < p_{k}^{(i)} < \alpha_{i} }^{u_{b}} \lambda_{BP}^{(i)} \times p_{k}^{(i)} + \lambda_{BN}^{(i)} \times (1-p_{k}^{(i)})   \\ 
    			+  \sum_{j=1 \atop p_{j}^{(i)} \le \beta_{i} }^{u_{n}} \lambda_{NP}^{(i)} \times p_{j}^{(i)} + \lambda_{NN}^{(i)} \times (1-p_{j}^{(i)}) \Bigg )  \\
    			{\rm s.t.}  \quad    0 < \beta_{i} <  \alpha_{i} < 1, 1 \le i \le t-1, \varepsilon \ge 1, u=u_{p}+u_{b}+u_{n} \label{equ5_Risk_alpha_beta}	
    			\end{split}
    			\end{equation}     			
    		\end{normalsize}
    	\end{spacing}

\noindent where $ u_{p} $, $ u_{b} $, and $ u_{n} $ are the numbers of instances in $ {\rm P}^{(i)}_{\rm {L}} $, $ {\rm B}^{(i)}_{\rm {L}} $, and $ {\rm N}^{(i)}_{\rm {L}} $, respectively; $ \varepsilon $ is a penalty factor to avoid too many instances being partitioned into $ {\rm B}^{(i)}_{\rm {L}} $, so as to prevent the overfitting of STWD-SFNN. Therefore, for the $ e $-th instance at the $ i $-th $(1 \le i \le t-1)$ granular level, the classification criteria are as follows.

    	(P1) if $ p_{e}^{(i)} \ge \alpha_{i} $, then $ [x]_{e}^{(i)} \in {\rm P}^{(i)}_{\rm {L}} $; 

    	(B1) if $ \beta_{i} < p_{e}^{(i)} < \alpha_{i} $, then $ [x]_{e}^{(i)} \in {\rm B}^{(i)}_{\rm {L}} $;
    	
    	(N1) if $ p_{e}^{(i)} \le \beta_{i}  $, then $ [x]_{e}^{(i)} \in {\rm N}^{(i)}_{\rm {L}} $.
    	
    	If BND of the ($ t $-1)-th level is not empty, we need to add another hidden layer node in STWD-SFNN. Based on Definition \ref{def4_gamma}, we index the result cost matrix $ \mathbf \Lambda^{(t)}$ and the threshold $ \gamma $ at the $ t $-th level. STWD-SFNN computes the total decision risk of classifying $ u $ instances.

    	\begin{spacing}{0.5}  
    		%\vspace{0em}
    		\begin{normalsize}       	       	
    			\begin{equation}
    			\begin{split} 
    			Risk^{(t)} = \min_{\gamma}\Bigg (\sum_{o=1 \atop p_{o}^{(t)}\ge \gamma }^{u_{p}}  \lambda_{PP}^{(t)} \times p_{o}^{(t)} + \lambda_{PN}^{(t)} \times (1-p_{o}^{(t)})   \\
    			+  \sum_{j=1 \atop p_{j}^{(t)} < \gamma }^{u_{n}}  \lambda_{NP}^{(t)} \times p_{j}^{(t)} + \lambda_{NN}^{(t)} \times (1- p_{j}^{(t)}) \Bigg )  \\
    			{\rm s.t.}  \quad    0 < \gamma < 1,u=u_{p}+u_{n}
    			\label{equ6_Risk_gamma}
    			\end{split} 
    			\end{equation}  
    		\end{normalsize}
    	    %\vspace{0em}
    	\end{spacing}
    
     	\noindent where $ u_{p} $ and $ u_{n} $ are the numbers of instances in $ {\rm P}^{(t)}_{\rm {L}} $ and $ {\rm N}^{(t)}_{\rm {L}} $, respectively. Hence, for the $ e $-th instance at the $ t $-th level, the classification criteria are as follows.

     	(P2) if $ p_{e}^{(t)} \ge \gamma $, then $ [x]_{e}^{(i)} \in {\rm P}^{(t)}_{\rm {L}} $; 

     	(N2) if $ p_{e}^{(t)} < \gamma  $, then $ [x]_{e}^{(i)} \in {\rm N}^{(t)}_{\rm {L}} $.
     	
\begin{example} \label{exp4_result_cost}
	We adopt the sequential thresholds of Example \ref{exp2_alpha_beta_gamma} and the dataset of Example \ref{exp3_process_costs}. Suppose the first granular level learns the three misclassified instances $ \{x_{1}, x_{4}, x_{5}\} $ passed from SFNN, which are partitioned into $ {\rm P}^{(1)}_{\rm {L}} = \varnothing $, $ {\rm B}^{(1)}_{\rm {L}} = \{ x_{4}, x_{5}\} $, and $ {\rm N}^{(1)}_{\rm {L}} = \{x_{1}\} $. Moreover, suppose the conditional probabilities at the first level are $ [p_{1}^{(1)}, p_{2}^{(1)}] = [0,0.5] $, according to Example \ref{exp2_alpha_beta_gamma} and Equation \eqref{def3_alpha_beta}, we have the result cost $ Risk^{(1)} $ = 0 + $ \lambda_{stwd} \times $ 2 $ \times $ (0.1506 $ \times $ 0.5 +  0.1249 $ \times $ 0.5) + 1 $ \times $ (0.9021 $ \times $ 0 + 0 $ \times $ 1) = 0.5510, where $ \lambda_{stwd} $ = 2. Similarly, suppose the second level classifies $ \{x_{4}, x_{5}\} $ into $ {\rm N}^{(2)}_{\rm {L}} $ with the conditional probability 0.5. Thus, $ Risk^{(2)} $ = 0 + 0 + 2 $ \times $ (0.5962 $ \times $ 0.5 + 0 $ \times $ 0.5) = 0.5962. 
\end{example}
     		
     	Algorithm 1 shows the pseudocode of STWD-SFNN.     	   
     	\begin{breakablealgorithm}
     		\small
     		\caption{ STWD-SFNN }
     		\hspace*{0.02in} \leftline{{\bf Input:}
     			$ T = \left(U, A = C \cup D, V = \{V_{a} \mid a \in A\}, I = \{I_{a} \mid a \in A\}, H \right) $ } \\
     		\hspace*{0.02in} \leftline{{\bf Output:}
     			hidden layer nodes $ i $, $ \mathbf W_{1} $ and $ \mathbf b_{1} $ connected the input } \\  \leftline{layer to the hidden layer, and $ \mathbf W_{2} $ and $ \mathbf b_{2} $ connected the hidden } \leftline{ layer to the output layer}    		
     		
     		\begin{algorithmic}[1]	 
     			\State  Initialization: rng(0), $ i  \leftarrow $ 1, $ t  \leftarrow $ 10, $ \theta \leftarrow $ 2,  $ \rho_{1} \leftarrow $ 0.9, $ \rho_{2} \leftarrow $ 0.999,  $ \tau \leftarrow  10^{-8} $, $ \mathbf w_{1}^{(i)} $, $ \mathbf b_{1}^{(i)} $, $ \mathbf w_{2}^{(i)} $, $ \mathbf b_{2}^{(i)} $, activation function (ReLU, leaky ReLU, SeLU, tanh, sigmoid, or swish), distribution of parameters (uniform or normal), POS $ \leftarrow \varnothing $, NEG $ \leftarrow \varnothing $, BND $ \leftarrow  T $,  $ Cost_{{\rm PPT}_i} $, $ Cost_{{\rm PPD}_i} $,  $ \mathbf \Lambda^{(i)}$, $ (\alpha_{i}, \beta_{i}) $, and $ \gamma $;
     			
     			\While { $ \mid {\rm BND} \mid \neq \varnothing $ }       			
     			\Statex \qquad  // Construct SFNN by adding one hidden layer node 
     			\State  Calculate the cost function using Equation \eqref{equ2_L_NN};
     			\State  Update $ \overline{\mathbf w}_{1}^{(i)} $,  $ \overline{\mathbf b}_{1}^{(i)} $, $ \overline{\mathbf w}_{2}^{(i)} $, and  $ \overline{\mathbf b}_{2}^{(i)} $ using Equation \eqref{equ3_Adam};   	
     			\State  Get $ {\rm P}^{(i)}_{\rm {N}} $, $ {\rm M}^{(i)}_{\rm {N}} $, and $ {\rm N}^{(i)}_{\rm {N}} $ of SFNN;		
     			%Get misclassified instances of SFNN;	
     			
     			\Statex \qquad  // Construct STWD on the misclassified instances
     			
     			\State  Use $ k $-means++ to discretize the numerical  data into categorical data;  // Preprocessing of data
     			
     			\State  Calculate  probability $ p_{e}^{(i)} $ using Equation \eqref{equ4_P}; 
     			
     			\If {$ i < t $ }
     			\State Achieve result cost matrix $ \mathbf \Lambda^{(i)} $, thresholds $ (\alpha_{i},\beta_{i}) $, and regions $ {\rm P}^{(i)}_{\rm {L}} $, $ {\rm B}^{(i)}_{\rm {L}} $, and $ {\rm N}^{(i)}_{\rm {L}} $ of STWD;
     			\State Compute $ Risk^{(i)} $ using Equation \eqref{equ5_Risk_alpha_beta};
     			\Else  	
     			
     			\State Achieve result cost matrix $ \mathbf \Lambda^{(t)} $, threshold $ \gamma $, and regions $ {\rm P}^{(t)}_{\rm {L}} $ and $ {\rm N}^{(t)}_{\rm {L}} $ of STWD;        	
     			\State Compute $ Risk^{(t)} $ using Equation \eqref{equ6_Risk_gamma};
     			\EndIf
     			\State  Compute $ Cost_{{\rm PT}_i} $ and $ Cost_{{\rm PD}_i} $ using Definition \ref{def6_process_costs};
     			
     			\State  Update  POS, BND, and NEG using Definition \ref{def5_regions}; 
     			\State  $ i \leftarrow i + 1 $;     
     			\EndWhile     
     			
     			\State $ {\mathbf W}_{1} = \begin{bmatrix}   \overline{\mathbf w}_{1}^{(1)}  \\ \cdots \\  \overline{\mathbf w}_{1}^{(i)}  \end{bmatrix} $, $ {\mathbf b}_{1} = \begin{bmatrix}  \overline{\mathbf b}_{1}^{(1)}  \\ \cdots \\  \overline{\mathbf b}_{1}^{(i)}  \end{bmatrix} $,  $ {\mathbf W}_{2} = \begin{bmatrix} \overline {\mathbf w}_{2}^{(1)} \cdots \overline {\mathbf w}_{2}^{(i)} \end{bmatrix} $, and $ {\mathbf b}_{2} = \overline {\mathbf b}_{2}^{(i)} $; 
     			
     			\hspace*{0.02in} \State \leftline{{\bf return } $ i $, $ \mathbf W_{1} $, $ \mathbf b_{1} $, $ \mathbf W_{2} $, and $ \mathbf b_{2} $;}
     		\end{algorithmic}	
     	\end{breakablealgorithm}
     	
       {According to Algorithm 1, there are seven parameters in STWD-SFNN. rng(0) is used to guarantee the repeatability of the experimental results. $ i $ is utilized to guarantee that the number of hidden nodes increases one by one. $ t $ is applied to guarantee the maximum number of hidden nodes. If the number of hidden nodes is small, it means that STWD-SFNN converges faster, and the training time of STWD-SFNN is shorter; otherwise, the training time of STWD-SFNN is longer. $ \theta $ is set to adjust the weight reduction rate of easy-to-classify instances. $ \rho_1 $ and $ \rho_2 $ are set to guarantee the exponential decay rate of the first-order and second-order momentum estimates of the Adam optimizer. $ \tau $ is settled to avoid the denominator of Adam being zero.}
       
\subsection{An illustrative example}\label{illstrativeexample}
A dataset in Table \ref{tab_example_data} is adopted to illustrate the principle of STWD-SFNN.  Without loss of generality, we select the result cost matrix and sequential thresholds of Examples \ref{exp1_result_cost} and \ref{exp2_alpha_beta_gamma}. The illustrative example is shown as follows.
    	
We initialize the number of hidden nodes to one, and randomize the parameters subjected to uniform distribution, including $ \mathbf w_{1}^{(1)} = (0.9168,-1.1743,0.4410,0.0746) $, $\mathbf b_{1}^{(1)} = (0.0047) $, $ \mathbf w_{2}^{(1)} = (0.1401,0.0087)^\top $, and $\mathbf b_{2}^{(1)} = (0.0018,0.0082)^\top $. Meanwhile, we adopt selu function and Adam optimizer, and then validate the classification performance of the model on $ \{ x_{7}, x_{8} \} $. Correspondingly, the optimized parameters are $ \overline{\mathbf w}_{1}^{(1)} = (0.8115,-1.0612,0.3465,0.1514) $, $ \overline{ \mathbf b}_{1}^{(1)} = (0.1139) $, $ \overline{\mathbf w}_{2}^{(1)} = (0.2019,0.0860)^\top $, and $ \overline{\mathbf b}_{2}^{(1)} = (0.1110,1177)^\top $.  According to $ \hat{\mathbf y} = \overline{\mathbf w}_{2}^{(1)} \times \rm selu(\overline{\mathbf w}_{1}^{(1)} \times \mathbf x^\top + \overline{ \mathbf b}_{1}^{(1)}) + \overline{\mathbf b}_{2}^{(1)} $, we have the prediction lables $[1, 1, 1, 2, 1, 1]$ of SFNN in training dataset. Compared with the truth labels $[2, 1, 1, 1, 2, 1]$, we have $ {\rm P}^{(1)}_{\rm {N}} = \{ x_{2}, x_{3}, x_{6}\} $, $ {\rm M}^{(1)}_{\rm {N}} = \{ x_{1}, x_{4}, x_{5}\} $, and $ {\rm N}^{(1)}_{\rm {N}} = \varnothing $. Since Table \ref{tab_example_data} is numerical data, we discretize $ \{ x_{1}, x_{4}, x_{5}\} $ into  $ \{ \{x_{1}\}, \{x_{4}, x_{5}\} \} $ by $ k $-means++ with two clusters. The conditional probabilities are 0 and 0.5, since the label equivalence class is $\{\{x_{4}\}, \{x_{1}, x_{5}\} \}$. Hence, STWD-SFNN has $ {\rm P}^{(1)}_{\rm {L}} = \varnothing$, $ {\rm B}^{(1)}_{\rm {L}} = \{ x_{4}, x_{5}\} $, and $ {\rm N}^{(1)}_{\rm {L}} = \{ x_{1}\} $ at the first level. According to Examples \ref{exp3_process_costs} and \ref{exp4_result_cost}, we have $ Risk^{(1)} = 0.5510 $, $ Cost_{ {\rm PT}_{1}} = 3 $, and $ Cost_{ {\rm PD}_{1}} = 3 $. Simultaneously, the decision of STWD-SFNN in the first turn are $ {\rm P}^{(1)}_{\rm {N}} \cup {\rm P}^{(1)}_{\rm {L}} = \{ x_{2}, x_{3}, x_{6}\} $, $ {\rm M}^{(1)}_{\rm {N}} \cap {\rm B}^{(1)}_{\rm {L}} = \{ x_{4}, x_{5}\} $, and $ {\rm N}^{(1)}_{\rm {N}} \cup {\rm N}^{(1)}_{\rm {L}} = \{ x_{1}\} $.
    	
    	%According to Equation \eqref{equ5_Risk_alpha_beta} and Example \ref{exp1_result_cost}, $ Risk^{(1)} = 0 + \lambda_{stwd} \times 2 \times (0.1506 \times 0.5 + 0.1249 \times 0.5) + 0 = 0.5510 $, where $ \lambda_{stwd} = 2 $. According to Definition \ref{def6_process_costs} and Example \ref{exp2_alpha_beta_gamma}, we have $ Cost_{PT_{1}} = 7 \times 1 = 7 $ and $ Cost_{PD_{1}} = 7 \times 1 = 7 $.
    	
Since $ {\rm M}^{(1)}_{\rm {N}} \cap {\rm B}^{(1)}_{\rm {L}} $ is not empty, we continue to add one hidden node in STWD-SFNN. The parameters to be optimized are $ \mathbf w_{1}^{(2)} = (-0.3114,-0.2530,1.0117,0.3236) $, $ \mathbf b_{1}^{(2)} = (0.0035) $, $ \mathbf w_{2}^{(2)} = (0.1034,-0.0356)^\top $, and $ \mathbf b_{2}^{(2)} = (0.0018,0.0082)^\top $. After optimization, we have $ \overline{\mathbf w}_{1}^{(2)} = (-0.2338,-0.1741,0.9333,0.2477) $, $ \overline{\mathbf b}_{1}^{(2)} = (0.0818) $, $ \overline{\mathbf w}_{2}^{(2)} = (0.1343,0.0133)^\top $, and $ \overline{\mathbf b}_{2}^{(2)} = (0.0768,0.0821)^\top $. STWD-SFNN has $ {\rm P}^{(2)}_{\rm {N}} = \varnothing $, $ {\rm M}^{(2)}_{\rm {N}} = \{ x_{4}, x_{5}\} $, and $ {\rm N}^{(2)}_{\rm {N}} = \varnothing $. Based on the instance equivalent class $ \{ x_{4}, x_{5} \} $, label equivalent class $ \{ \{ x_{4} \}, \{ x_{5} \} \} $, and the conditional probability 0.5, we have $ {\rm P}^{(2)}_{\rm {L}} = \varnothing $, $ {\rm B}^{(2)}_{\rm {L}} = \varnothing $, and $ {\rm N}^{(2)}_{\rm {L}} = \{ x_{4}, x_{5}\} $. Similarly, according to Examples \ref{exp3_process_costs} and \ref{exp4_result_cost}, we have $ Risk^{(2)} = 0.5962 $, $ Cost_{ {\rm PT}_{2}} = 7 $, and $ Cost_{ {\rm PD}_{2}} = 4 $. Meanwhile, the decisions for STWD-SFNN in the second turn are $ {\rm P}^{(2)}_{\rm {N}} \cup {\rm P}^{(2)}_{\rm {L}} = \varnothing $, $ {\rm M}^{(2)}_{\rm {N}} \cap {\rm B}^{(2)}_{\rm {L}} = \varnothing $, and $ {\rm N}^{(2)}_{\rm {N}} \cup {\rm N}^{(2)}_{\rm {L}} = \{ x_{4}, x_{5}\} $.
    	  	
The learning of STWD-SFNN is stopped, since $ {\rm M}^{(2)}_{\rm {N}} \cap {\rm B}^{(2)}_{\rm {L}} $ is empty. According to Definition \ref{def5_regions}, the decisions of STWD-SFNN with two hidden layer nodes are POS=$ \{ x_{2}, x_{3}, x_{6} \} $, BND=$ \varnothing $, and NEG=$ \{ x_{1}, x_{4}, x_{5} \} $. Meanwhile, the model is $ \hat{\mathbf y} = \mathbf W_{2} \times \rm selu(\mathbf W_{1} \times \mathbf x^\top + \mathbf b_{1}) + \mathbf b_{2} $, where $ \mathbf W_{1} =  \begin{bmatrix} \overline{\mathbf w}_{1}^{(1)} \\ \overline{\mathbf w}_{1}^{(2)} \end{bmatrix} = \begin{bmatrix} 0.8115,-1.0612,0.3465,0.1514 \\ -0.2338,-0.1741,0.9333,0.2477 \end{bmatrix} $, $ \mathbf b_{1} =  \begin{bmatrix}  \overline{\mathbf b}_{1}^{(1)} \\ \overline{\mathbf b}_{1}^{(2)} \end{bmatrix} = \begin{bmatrix} 0.1139 \\ 0.0818 \end{bmatrix} $, $ \mathbf W_{2} =  \begin{bmatrix} \overline{\mathbf w}_{2}^{(1)}, \overline{\mathbf w}_{2}^{(2)} \end{bmatrix} = \begin{bmatrix} 0.2019, 0.1343 \\ 0.0860, 0.0133 \end{bmatrix} $, and $ \mathbf b_{2} = \overline{\mathbf b}_{2}^{(2)} = \begin{bmatrix} 0.0768 \\ 0.0821 \end{bmatrix} $. Finally, we obtain the predicted value $ \hat{ \mathbf y} $ when applying STWD-SFNN to $ \{ x_{9}, x_{10} \} $.   
    	
    	% STWD-SFNN has $ Risk^{(2)} =  0 + 0 + 2 \times ( 0.5962 \times 0.5 + 0 \times 0.5) = 0.5962 $, $ Cost_{PT_{2}} = 7 \times 1 + 2 \times 2 = 11 $, and $ Cost_{PD_{2}} = \max (7 \times 1, 2 \times 2) = 7 $.   
    	
Since the thresholds of TWD-SFNN are relatively fixed, compared with STWD-SFNN, there are two exceptional cases, i.e., the coarsest thresholds and the finest thresholds. Case 1 (The coarsest): If the thresholds are $ (\alpha, \beta) = (\alpha_{1}, \beta_{1} ) = (0.6894, 0.1425) $, TWD-SFNN will continue to learn the network topology when the nodes of hidden layer are two. Case 2 (The finest): If the thresholds are $ (\alpha, \beta) = (\alpha_{2}, \beta_{2} ) = (0.5389,0.5016) $, TWD-SFNN will  stop the training of model early. However, STWD-SFNN adopts sequential thresholds to traverse from coarse-grained level to fine-grained level step by step, thus better handling the problem of network topology. Therefore, STWD-SFNN is a generalized model of TWD-SFNN.

		\subsection{Theoretical analysis of STWD-SFNN}
		\begin{theorem}
    		%\label{theo:1}
    		STWD-SFNN algorithm is convergent.           	
    	\end{theorem}

\begin{proof}
    (Proof by Contradiction) Suppose STWD-SFNN completes the learning of instances in infinite turns, it means that STWD-SFNN has $ u$ $(u \ge 1) $ instances which always are classified into BND. STWD-SFNN needs to compare probability $ p_{e}^{(t)} $ of the $ e $-th $ (1 \le e \le u) $ instance with thresholds $ (\alpha_{t}, \beta_{t}) $ at the $ t $-th($ t \ge 1 $) level. If $ p_{e}^{(t)} $ is larger than $ \alpha_{t} $, then the $ e $-th $ (1 \le e \le u) $ instance is classified into POS; if $ p_{e}^{(t)} $ is less than $ \beta_{t} $, then it is classified into NEG. If $ p_{e}^{(t)} $ is in $ [\beta_{t}, \alpha_{t}] $, then it is classified into BND. The probability of the $ e $-th instance falls into the range $ \alpha_{t} - \beta_{t}$, and hence the probability for $ u $ instances is $ (\alpha_{t} - \beta_{t})^{u}$. If $ u $ instances are always classified into BND, it is easy to know that $ {\rm E}(\alpha_{t} - \beta_{t})^{u} = 1 $. Thus, we have $ \alpha_{t} = 1 $ and $ \beta_{t} = 0 $ in infinite turns, which contradicts the assumption that $ 0 < \beta_{t} < \alpha_{t} < 1$, where $ t \ge 1$. In short, STWD-SFNN is convergent.
\end{proof}

		\section{ Experiments }
{To verify the classification performance of STWD-SFNN, we selected ten classification models: $m_{1}$-SFNN-R, $m_{2}$-SFNN-R, $m_{3}$-SFNN-R \cite{Belciug2021ESWA}, Grid search optimization of SFNN (GS-SFNN) \cite{Pontes2016Neurocomputing}, Particle swarm optimization of SFNN (PSO-SFNN) \cite{Nistor2022ESWA},  TWD-SFNN-R \cite{Cheng2021INS}, STWD-SFNN-NK, support vector classification (SVC) \cite{Liu2020TPAMI}, random forest (RF) \cite{Katuwal2020PR}, and k-nearest neighbors (KNN) \cite{Geler2020ESWA}. These models are used as the competitive models to evaluate the performance of our algorithm from the following six research questions (RQs).}
		
		RQ 1: How does STWD-SFNN perform compared with static topology discovery without STWD? 
		
		RQ 2: How does STWD-SFNN perform compared with choosing a dynamic tuning model without STWD?
		
		RQ 3: How does STWD-SFNN perform compared with employing TWD to find network topology?
		
		RQ 4: How does STWD-SFNN perform compared with the model without $ k $-means++ discretization?
		
		RQ 5: How does STWD-SFNN perform compared with other classification models?
		
		RQ 6: How is the process costs of STWD-SFNN related to network topology?
		
	    To handle RQ 1, we utilize the SFNN determined by the empirical formula models in Section \ref{exp_sfnn}. In response to RQ 2, we employ grid search and particle swarm optimization to dynamically select the network topology in Section \ref{exp_gs_pso}. To address RQ 3, we utilize TWD-SFNN in Sections \ref{exp_gs_pso} and \ref{exp_twd_5folds}. To answer RQ 4, we adopt STWD-SFNN without $ k $-means++ as a competitive model in Section \ref{exp_gs_pso}. To answer RQ 5, we adopt SVC, RF, and KNN to verify the generalization ability of STWD-SFNN in section \ref{exp_svc}. To handle RQ 6, we analyze the STWD-SFNN in Section \ref{exp_costs}.

		%Section 4.1 presents the experimental data and evaluation criteria, Section 4.2 introduces the competition models, Sections 4.3 and 4.4 report the preparation and the performance of comparative experiment, and Section 4.5 analyzes the STWD-SFNN.

		\subsection{Benchmark datasets and evaluation criteria }
 { To report the performance of STWD-SFNN, we selected 15 commonly used datasets which can be downloaded from https://archive.ics.uci.edu/ml/index.php.  The characteristics of these datasets are shown in 	Table \ref{tab_uci_data}. All experiments were performed on a computer equipped with an Intel (R) Core (TM) i7-10700K CPU, 16 GB RAM, and a 64-bit version of Windows 10. All models were programmed in Matlab, which can be downloaded from https://github.com/wuc567/Machine-learning/tree/main/STWD-SFNN.}

		\begin{table*}[hb]
			\vspace{0em}
			\setlength{\abovecaptionskip}{0cm}    % 段前
			\setlength{\belowcaptionskip}{-0.2cm} % 段后
			\small
			\captionsetup{font={small}}
			\newcommand{\tabincell}[2]{\begin{tabular}{@{}#1@{}}#2\end{tabular}}
			\caption{ Benchmark datasets }
			\centering{
				\begin{tabular}{cccc}
					\hline
					Name of datasets     & Abbreviation   & \tabincell{c}{Number of instances}  & \tabincell{c}{Number of  attributes}      \\ \hline
					Online news popularity   & ONP           & 39,797            & 61     \\
					QSAR oral toxicity       & QSAR          & 8,992             & 1,024   \\
					Online shoppers purchasing intention & OSP   & 12,330        & 18     \\
					Electrical grid stability simulated  & EGSS  & 10,000        & 14     \\
					Skin segmentation                    & SE    & 245,057       & 4      \\
					HTRU                                 & HTRU  & 17,898        & 9      \\
					Default of credit card clients       & DCC   & 30,000        & 24     \\ 
					Epileptic seizure recognition        & ESR   & 11,500        & 179   \\
					Bank marketing                       & BM    & 45,211        & 17    \\
					Polish companies’ bankruptcy         & PCB   & 10,503        & 64   \\
					Shill bidding                        & SB    & 6,321         & 13    \\
					Estimation of obesity levels         & EOL   & 2,111         & 17  \\
					Occupancy detection                  & OD    & 20,560        & 7  \\
					Room occupancy estimation            & ROE   & 10,129        & 16  \\
					%BLE RSSI for indoor localization     & RSSI  & 23,570        & 5  \\
					Sepsis survival minimal clinical records   & SSMCR        & 110,204       & 4  \\
					%Sepsis survival minimal clinical records for study cohort & SSMCR-study   & 19,051 & 4  \\
					\hline
				\end{tabular}}
			\vspace{-1em}
		\label{tab_uci_data}
		\end{table*}
		We adopted the weighted-f1 metric \cite{Cheng2021INS,Rashid2022ITJ} to evaluate the performance, since weighted-f1 handles the problem of unbalanced classification data and considers the importance of different categories in classification data. Weighted-f1 can be defined as follows.  
		
		\begin{spacing}{0.5} % 行间距
			\begin{normalsize}
				\begin{gather}
				{\rm weighted-f1} = \frac{\sum_{i=1}^{2} \mid d_{i}\mid \ast f_{i} }{\mid d \mid} 	\\
				{f_{i}} = \frac{ 2 \times Precision_{i} \times Recall_{i} }{ Precision _{i} + Recall_{i} } 	 \\
				Precision _{i}  = \frac{ TP_{i} }{ TP_i+FP_i } \\
				Recall_{i}  = \frac{ TP_i }{ TP_i+FN_i }  
				\end{gather}
			\end{normalsize}
		\end{spacing}
  	
	\noindent where $ \mid d_{i} \mid $ and $ \mid d \mid $ represent the numbers of $ i $-th label and all label, respectively; $ f_{i} $ represents the f1-score of the $ i $-th label; $ Precision _{i} $ and $ Recall_{i} $ are the precision rate and the recall rate of $ i $-th label, respectively; $ TP_i $ and $ FP_i $ are the numbers of correct and incorrect predictions for instances with positive labels, respectively; and $ FN_i $ is the number of incorrect predictions for instances with negative labels. The greater the value of weighted-f1, the better the performance of the model.
		
	In addition, we employ other evaluation criteria, including the accuracy, training time, test time, number of hidden layer nodes, receiver operating characteristic (ROC) curve, and area under the curve of ROC (AUC), where the training time is the learning time of the model under a specified set of parameters. The higher the accuracy and AUC, the shorter the training time,  the shorter the test time, the lower number of hidden layer nodes, and the closer the curve to the upper left corner, the better the performance of the model.

		\subsection{Baseline models}
{	To validate the performance of STWD-SFNN, we compared it with ten competitive models, which are listed as follows.}

\begin{enumerate}[(1)]    

\item SFNN-R \cite{Belciug2021ESWA}: The network parameters of SFNN are initialized by random numbers, which cannot guarantee the repeatability of the experimental results.  To overcome this shortage, we set the seed of the random number generator to rng(0), and initialize the hyperparameters in a fixed random number manner to guarantee the repeatability of learning results. We name SFNN with the fixed hyperparameter as SFNN-R. SFNN-R adopts pre-determined hidden layer nodes and learns the optimal network parameters through focal loss and Adam optimizer, thus generating a nonlinear decision boundary. { Moreover, for SFNN-R, the numbers of hidden nodes determined by the empirical formula methods are $ m_{1} $-SFNN-R: $ \sqrt{m+n}+\alpha,(\alpha \in (1,10)) $, $ m_{2} $-SFNN-R: $ log_2 m $, and $ m_{3} $-SFNN-R: $ \sqrt{m \times n} $, where $ m $ and $ n $ are the numbers of nodes in the input and output layers, respectively.}
			% To study the nonlinear relationships between instances, SFNN adopts pre-determined hidden layer nodes and learns the optimal network parameters through forward and backward propagation, thus generating a nonlinear decision boundary.
												         
			\item GS-SFNN \cite{Pontes2016Neurocomputing}: For the hyperparameters of the number of hidden layer nodes in SFNN, we select a common hyperparameters method, i.e., grid search, to find the optimal number of hidden layer nodes of SFNN.
			
			\item PSO-SFNN \cite{Nistor2022ESWA}: To tune the hyperparameters of SFNN, we choose particle swarm optimization, which utilizes collaboration and information sharing among individuals in the group to optimize SFNN.
			
\item TWD-SFNN-R \cite{Cheng2021INS}: The network parameters of TWD-SFNN, the thresholds of three-way decisions, and the cross-validation of the dataset are initialized by random numbers, which cannot guarantee the repeatability of the experimental results. To overcome this shortage, we set the seed of the random number generator to rng(0), and initialize the hyperparameters in a fixed random number manner to guarantee the repeatability of learning results. We name TWD-SFNN with the fixed hyperparameter as TWD-SFNN-R which adds a delayed decision region to determine the number of hidden layer nodes, and aims to minimize the decision risk, thereby improving the performance of the model. 
   
   %Since the network parameters of TWD-SFNN have random numbers when initialized, it cannot guarantee the repeatability of the experiment. 
			
\item STWD-SFNN-NK: To verify the performance of the STWD-SFNN model by adopting discretization techniques, we compare it with STWD-SFNN without $ k $-means++ and name it STWD-SFNN-NK. The parameters affecting STWD-SFNN-NK are consistent with STWD-SFNN.

			\item SVC \cite{Liu2020TPAMI}: To obtain the optimal separated hyperplane, SVC converts the maximization problem to a convex quadratic programming optimization problem, and applies the kernel function to learn the nonlinear classifier.
			
			\item RF \cite{Katuwal2020PR}: To reduce the correlation between decision trees, for  the classification problem, RF randomizes features and instances and adopts a voting mechanism to predict the label for each instance.  
			
			\item KNN \cite{Geler2020ESWA}: To learn the similarities between instances, KNN measures the distances between the different eigenvalues and selects the most frequent categories from the nearest $ K $ instance as the decision category of the model.
		\end{enumerate}

\subsection{ Preparation of the experiment }
For the hyperparameters of STWD-SFNN, we set the seed of the random number generator to rng(0) to guarantee the repeatability of the experimental results. The selection of process cost is explained as follows. We know that the unit test cost $ Cost_{{\rm PPT}_i} $ and the unit delay cost $ Cost_{{\rm PPD}_i} $ at each granular level are part of the input of STWD-SFNN. Without loss of generality, we randomly generate two vectors with dimension $ t $, one is $[Cost_{{\rm PPT}_1}, Cost_{{\rm PPT}_2}, \cdots, Cost_{{\rm PPT}_t}] $, and the other is $[Cost_{{\rm PPD}_1}, Cost_{{\rm PPD}_2}, \cdots, Cost_{{\rm PPD}_t}] $. In addition, the selection of the result cost matrix $ \mathbf \Lambda^{(i)} (1 \le i \le t) $ is illustrated as follows. Firstly, $ \mathbf \Lambda^{(i)} $ of the $ i $-th level is randomly initialized, which needs to satisfy the conditions $ 0 \le \lambda_{PP}^{(i)} < \lambda_{BP}^{(i)} < \lambda_{NP}^{(i)} < 1 $, $ 0 \le \lambda_{NN}^{(i)} < \lambda_{BN}^{(i)} < \lambda_{PN}^{(i)} < 1 $ , and $ (\lambda_{BN}^{(i)} - \lambda_{NN}^{(i)}) \times (\lambda_{BP}^{(i)} - \lambda_{PP}^{(i)}) < (\lambda_{PN}^{(i)} - \lambda_{BN}^{(i)}) \times (\lambda_{NP}^{(i)} - \lambda_{BP}^{(i)}) $. Next, the thresholds $ (\alpha_{i}, \beta_{i} )(1 \le i \le t) $ of the previous $ t $-1 levels and the threshold $ \gamma $ of the $ t $-th level are defined by Definitions \ref{def3_alpha_beta} and \ref{def4_gamma}, respectively. Finally, for the ($ i $+1)-th level, we retain $ \mathbf \Lambda^{(i+1)} $ and  $ (\alpha_{i+1}, \beta_{i+1} ) $ or $ \gamma $ if and only if the thresholds satisfy $ \beta_{i} \le \beta_{i+1} < \gamma \le \alpha_{i+1} \le \alpha_{i} $; otherwise, we iterate this process until $ \mathbf \Lambda^{(i+1)} $ is found. It should be noted that whether there is a better hyperparameter selection method to obtain the result cost matrix is worth investigating in the future.

		The hyperparameters of the competitive models are introduced as follows. For the network models, such as SFNN-R, GS-SFNN, PSO-SFNN, TWD-SFNN-R, and STWD-SFNN-NK, the activation function selects ReLU family (e.g., ReLU, LReLU, and SELU) or tanh family (e.g., tanh, sigmoid, and swish), and the network parameters subject to uniform distribution or normal distribution.  Except for SFNN-R, the number of hidden layer nodes in the rest network model increases from one to ten. Meanwhile, for PSO-SFNN, $ max\_iterations \in (10,100) $ with step size ten, and $ max\_particles \in (1,10) $ with step size one. On the other hand, for SVC, polynomial, RBF, and sigmoid kernel functions are selected,$ \gamma \in (-4,4) $ with step size one, and $ C \in (-4,4) $ with step size one. For RF, $ max\_trees \in (100,1000) $ with step size 100, $ max\_depth \in \{ q, log_2 q, \sqrt{q}, \phi q  \} $, where $ q $ is the number of features, $ \phi \in (0,1) $. For KNN, Euclidean distance and angle cosine are selected to characterize the similarity among instances, $ K \in (10,100) $ with step size ten.	
			
		It is worth noting that the dataset is divided into training set, validation set, and test set according to the ratio of 8:1:1, and the 10-fold cross-validation and grid search are adopted to find the optimal hyperparameters of each model. In addition, we utilize the Focal loss and Adam optimizer in Section \ref{focal_adam} to train the network model, where the regularization factor $ \lambda_{\rm SFNN} $, the learning rate, and batch size of the optimizer are uniformly set to 0.1, 512, and 0.1, respectively.

	\subsection{ Experimental results and analysis}	
	
	    \subsubsection{Comparison with static models} \label{exp_sfnn}
	    To validate the effectiveness of STWD-SFNN in optimizing the network topology, we select the static models SFNN-R as the competitive models, which adopt three empirical formulas to calculate the number of hidden layer nodes. Fig. \ref{fig_sfnn} shows the comparison of ROC curves of the SFNN-R models and  STWD-SFNN. Table \ref{tab_sfnn} reports the comparison of different evaluation criteria of those models. The results give rise to the following observations.
	    
	    \begin{figure*}[hp]
	    	\centering
	    	\subfigure[ONP]{
	    		\includegraphics[width=2.1in]{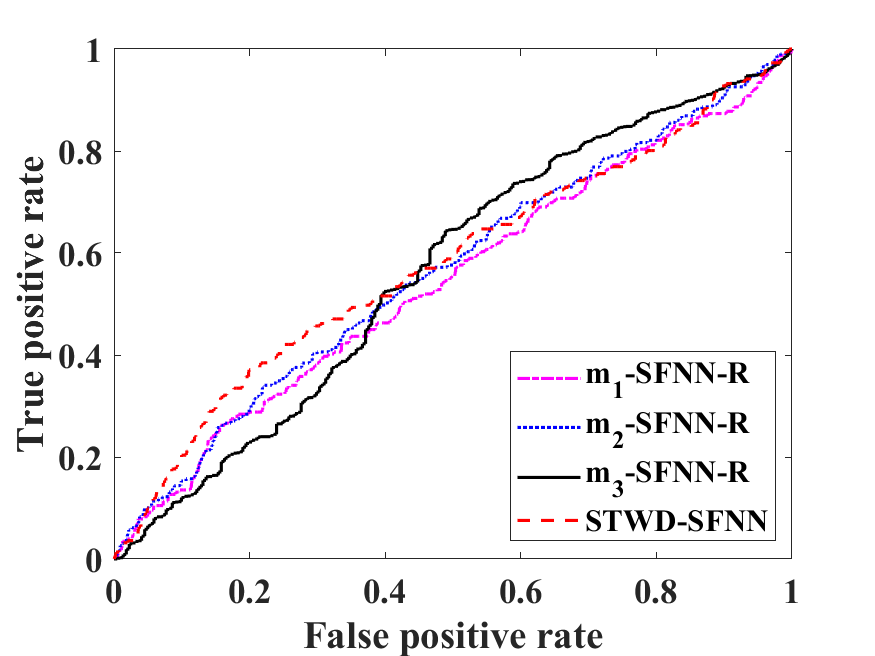}}
	    	\hfill
	    	\centering
	    	\subfigure[QSAR]{
	    		\includegraphics[width=2.1in]{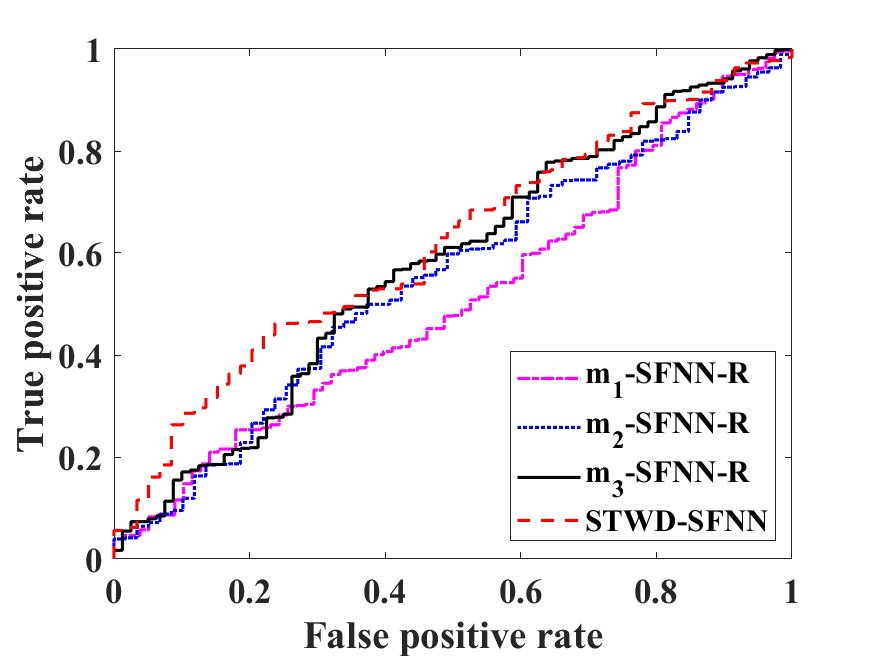}}
	    	\hfill
	    	\centering
	    	\subfigure[OSP]{
	    		\includegraphics[width=2.1in]{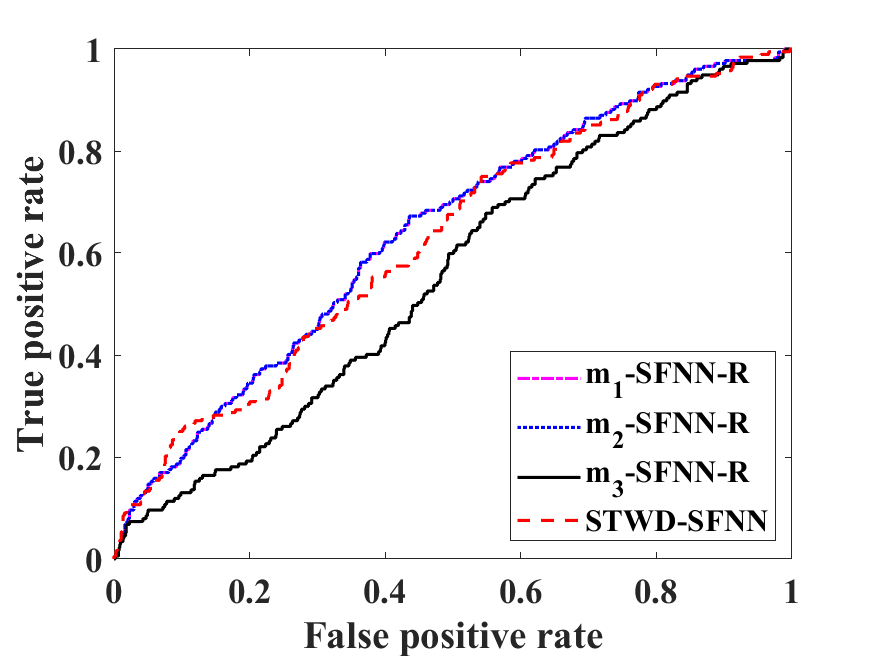}}
	    	\hfill
	    	\centering
	    	\subfigure[EGSS]{
	    		\includegraphics[width=2.1in]{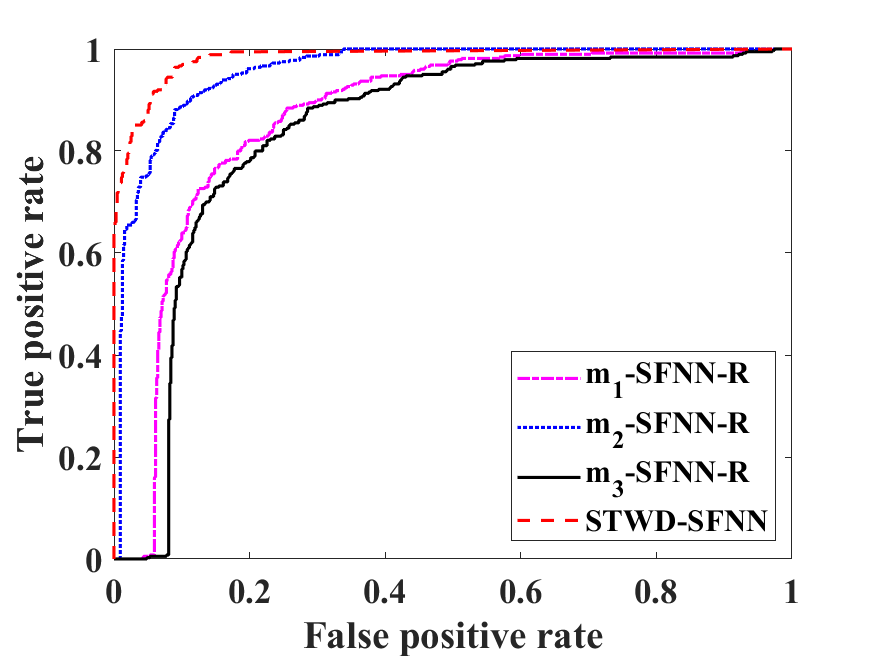}}
	    	\hfill
	    	\centering
	    	\subfigure[SE]{
	    		\includegraphics[width=2.1in]{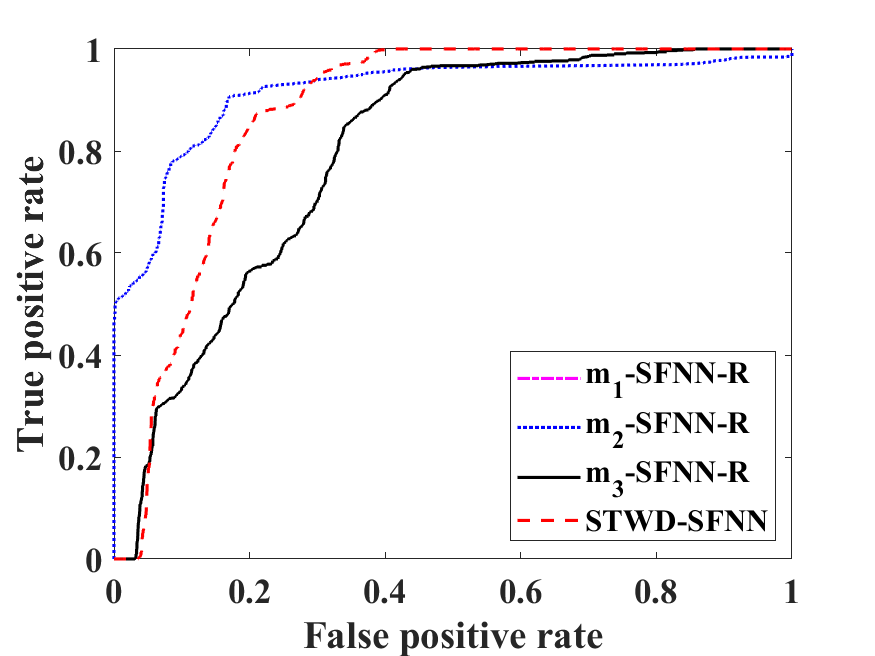}}
	    	\hfill
	    	\centering
	    	\subfigure[HTRU]{
	    		\includegraphics[width=2.1in]{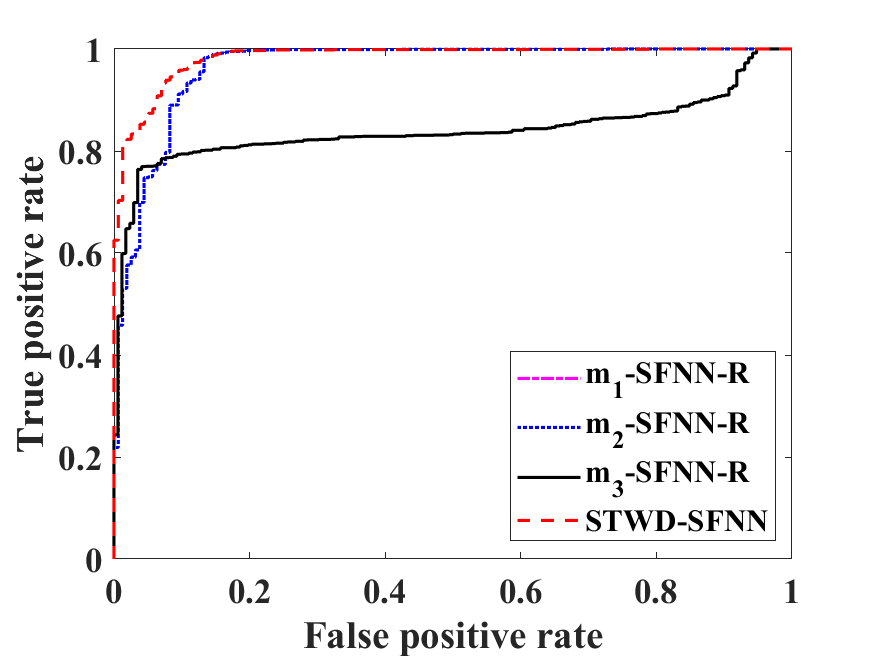}}
	    	\hfill
	    	\centering
	    	\subfigure[DCC]{
	    		\includegraphics[width=2.1in]{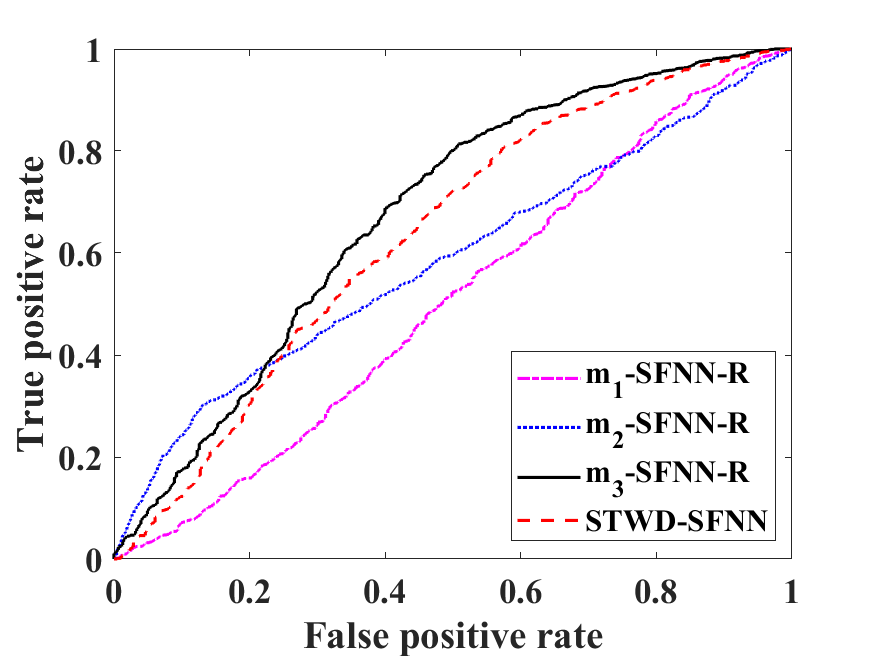}}
	    	\hfill
	    	\centering
	    	\subfigure[ESR]{
	    		\includegraphics[width=2.1in]{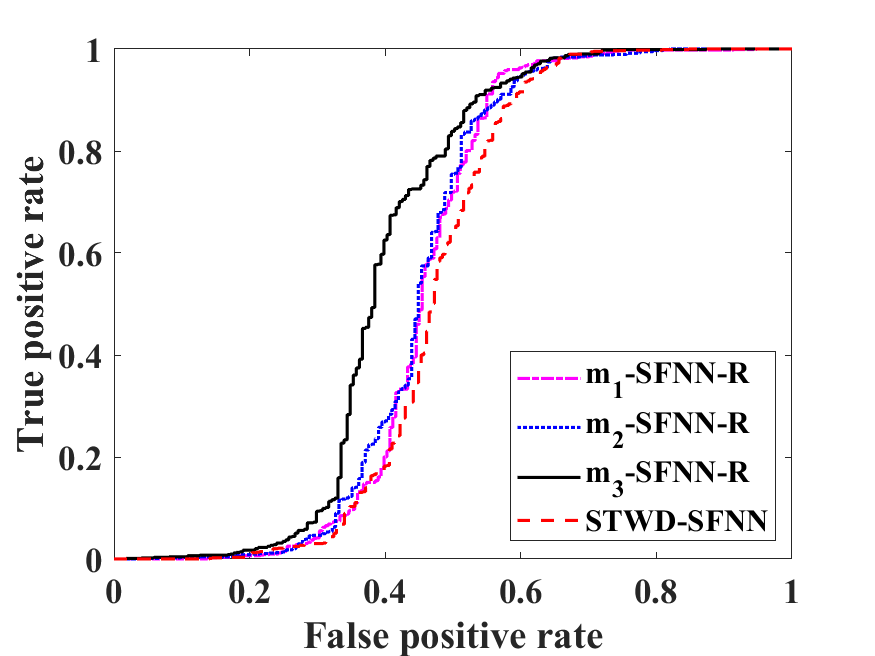}}	
	    	\hfill
	    	\centering
	    	\subfigure[BM]{
	    		\includegraphics[width=2.1in]{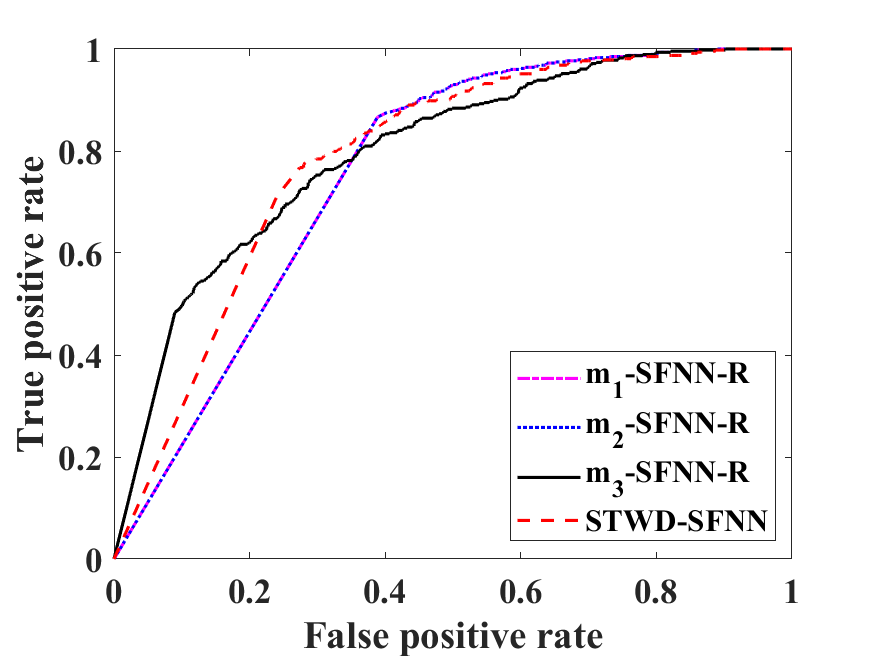}}
	    	\hfill
	    	\centering
	    	\subfigure[PCB]{
	    		\includegraphics[width=2.1in]{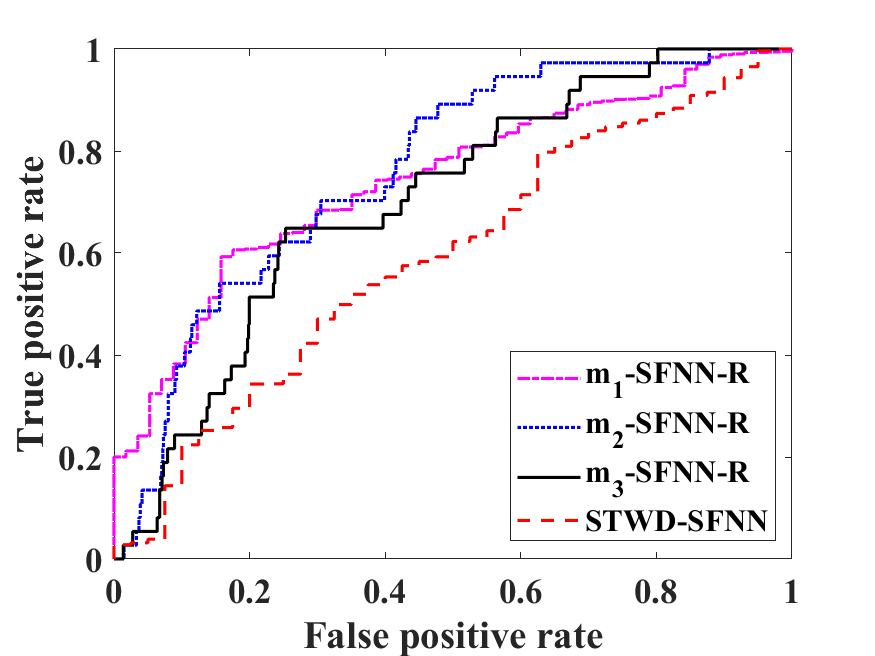}}
	    	\hfill	
	    	\centering
	    	\subfigure[SB]{
	    		\includegraphics[width=2.1in]{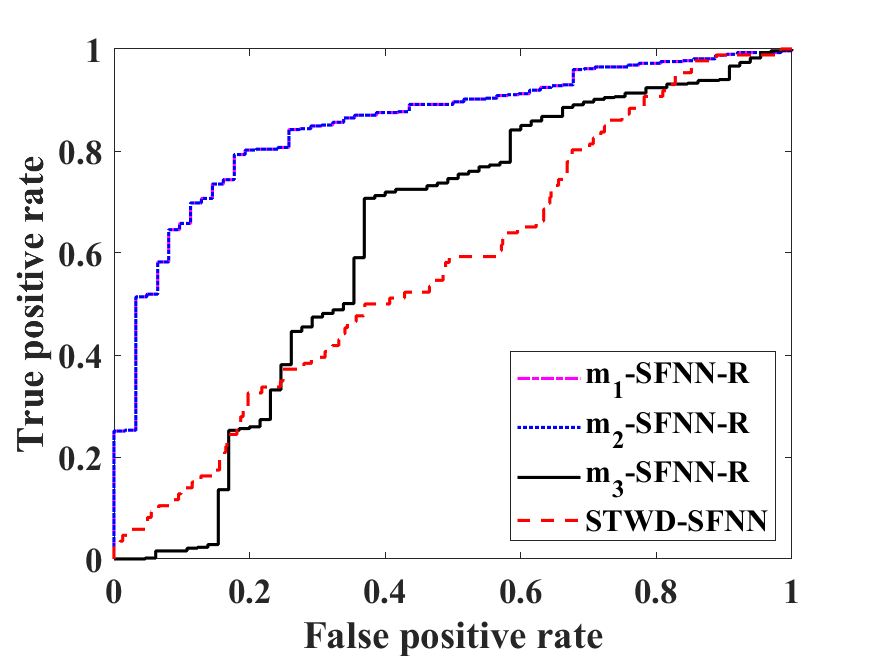}}
	    	\hfill	
	    	\centering
	    	\subfigure[EOL]{
	    		\includegraphics[width=2.1in]{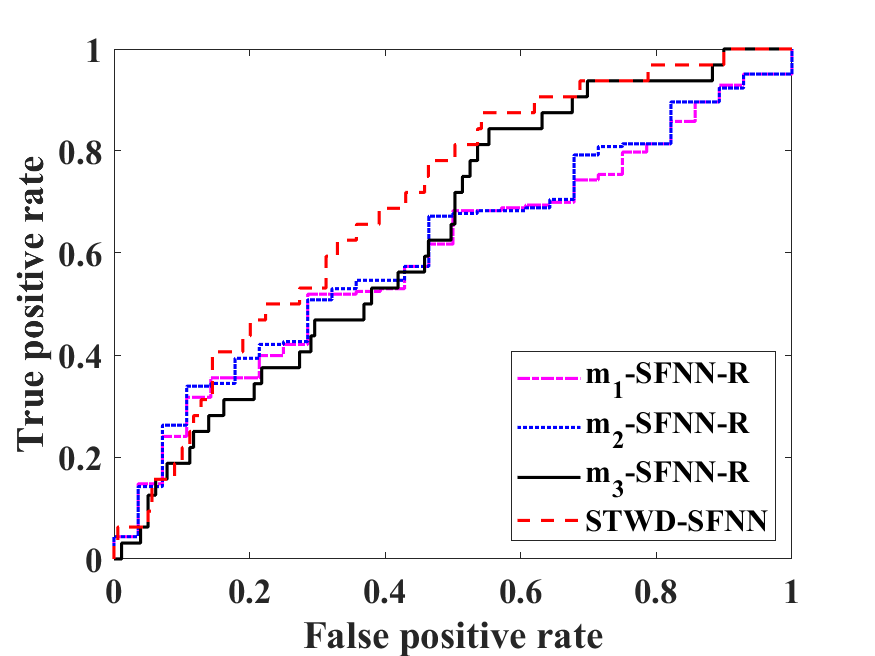}}
	    	\hfill	
	    	\centering	
	    	\subfigure[OD]{
	    		\includegraphics[width=2.1in]{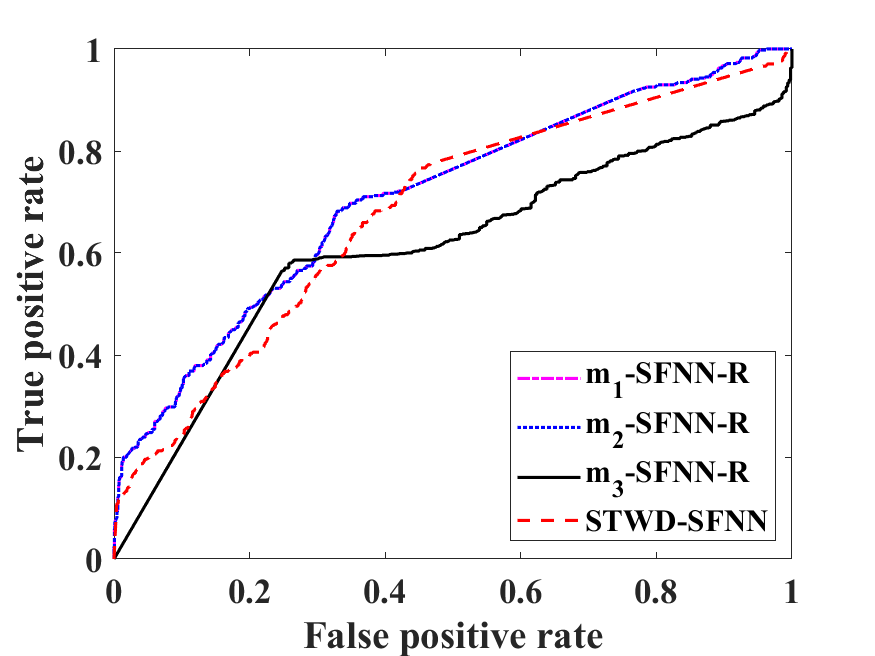}}
	    	\hfill	
	    	\centering
	    	\subfigure[ROE]{
	    		\includegraphics[width=2.1in]{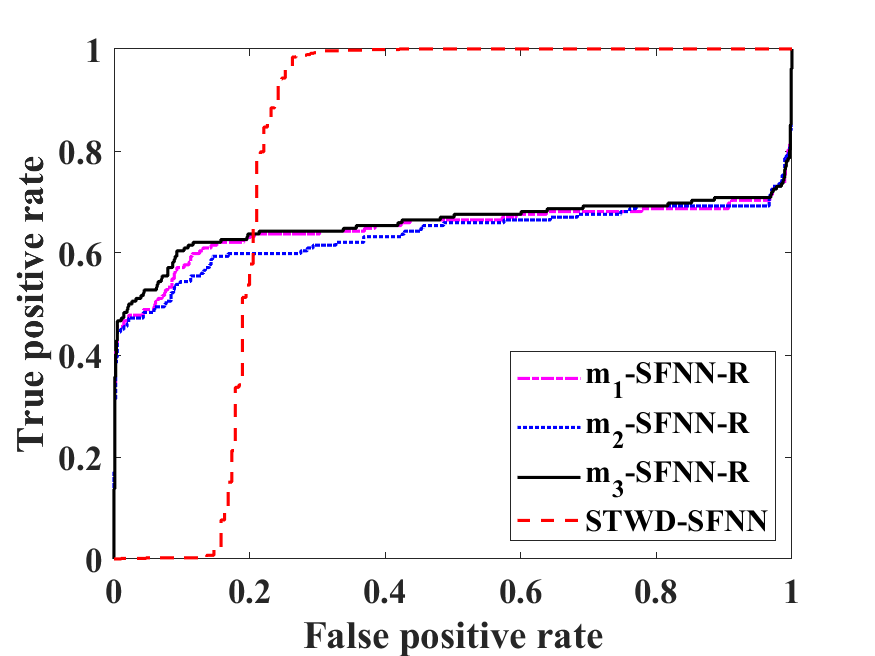}}
	    	\hfill	
	    	\centering
	    	\subfigure[SSMCR]{
	    		\includegraphics[width=2.1in]{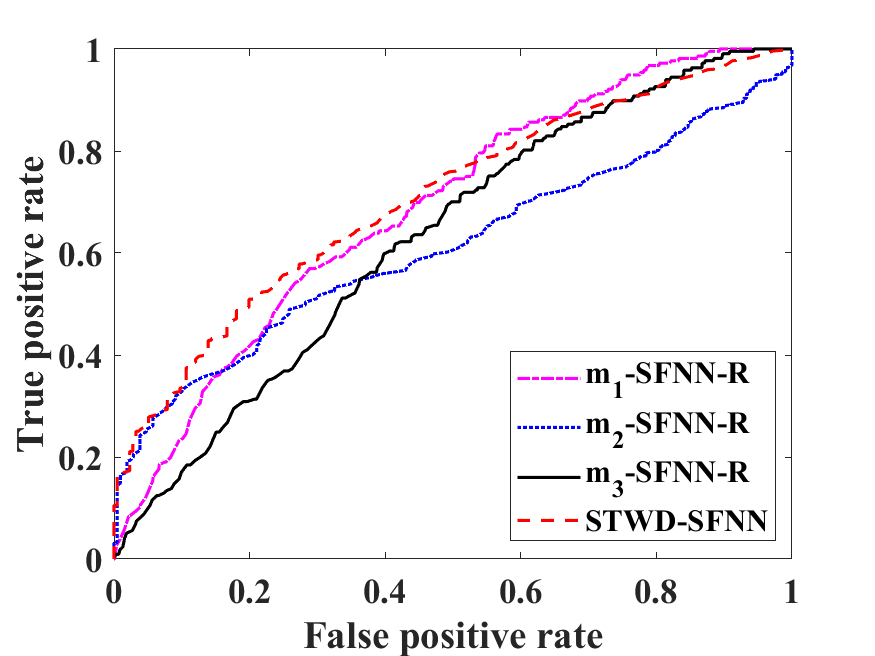}}
	    	\hfill	
	    	\small
	\caption{Comparison of ROC curves of static models and STWD-SFNN } 
	    	\label{fig_sfnn}    		
	    \end{figure*}
	    
		\begin{table*}[hp]
			\setlength{\abovecaptionskip}{0cm}    % 段前
			\setlength{\belowcaptionskip}{0cm}    % 段后
			\small
			\captionsetup{font={small}}
			\newcommand{\tabincell}[2]{\begin{tabular}{@{}#1@{}}#2\end{tabular}}
			\caption{ Comparison of static models and STWD-SFNN }
			\centering
			%            	\resizebox{\textwidth}{60mm}{
			\setlength{\tabcolsep}{1mm}{
				\begin{tabular}{cccccccc}
					\hline
					Dataset   &Model  &\tabincell{c}{Accuracy \\ (\%)}   &\tabincell{c}{Weighted- \\f1 (\%)}  &\tabincell{c}{AUC \\ (\%)}   &\tabincell{c}{Training \\ time(s)}  &\tabincell{c}{Test time (s)} &\tabincell{c}{Nodes}\\ \hline
						
					\multirow{4}{*}{ONP}
					&$m_{1}$-SFNN-R & 85.43$\pm$0.37 &87.42$\pm$0.29 &54.46 &41.76$\pm$0.38 &0.009$\pm$0.000 &8 \\
					&$m_{2}$-SFNN-R & 89.09$\pm$0.21 &\textbf{89.29$\pm$0.23} &56.43 &34.77$\pm$0.62 &0.008$\pm$0.001 &6 \\
					&$m_{3}$-SFNN-R & 86.40$\pm$0.30 &88.02$\pm$0.27 &56.50 &32.04$\pm$0.32 &0.009$\pm$0.005 &11 \\
					&STWD-SFNN &\textbf{94.41$\pm$0.22} &85.87$\pm$0.26 &\textbf{57.87} &\textbf{15.76$\pm$0.76} &\textbf{0.003$\pm$0.000} &\textbf{1.00$\pm$0.00} \\ \hline
						
					\multirow{4}{*}{QSAR}
					&$m_{1}$-SFNN-R & 62.67$\pm$0.83 &71.00$\pm$0.78 &50.89 &22.95$\pm$0.50 &0.015$\pm$0.007 &33 \\
					&$m_{2}$-SFNN-R & 76.57$\pm$1.11 &80.29$\pm$1.15 &55.04 &16.61$\pm$0.24 &0.013$\pm$0.005 &10 \\
					&$m_{3}$-SFNN-R &\textbf{89.57$\pm$0.79} &\textbf{87.17$\pm$1.04} &57.63 &\textbf{11.06$\pm$0.15} &0.010$\pm$0.004 &46 \\
					&STWD-SFNN &\underline{78.26$\pm$1.11} &\underline{81.21$\pm$1.13} &\textbf{61.61} &\underline{12.43$\pm$0.15} &\textbf{0.008$\pm$0.001} &\textbf{2.00$\pm$0.00} \\ \hline
						
					\multirow{4}{*}{OSP}
					&$m_{1}$-SFNN-R & 76.37$\pm$0.79 &74.08$\pm$1.00 &\textbf{63.89} &5.58$\pm$0.22 &0.005$\pm$0.003 &5 \\
					&$m_{2}$-SFNN-R & 76.37$\pm$0.79 &74.08$\pm$1.00 &\textbf{63.89} &\textbf{5.44$\pm$0.26} &0.004$\pm$0.002 &5 \\
					&$m_{3}$-SFNN-R &\textbf{78.41$\pm$0.87} &\textbf{75.33$\pm$1.06} &55.58 &5.73$\pm$0.23 &0.004$\pm$0.002 &6 \\
					&STWD-SFNN &68.43$\pm$0.66 &70.01$\pm$0.81 &\underline{62.52} &\underline{5.55$\pm$0.51} &\textbf{0.002$\pm$0.000} &\textbf{2.00$\pm$0.00} \\ \hline
						
					\multirow{4}{*}{EGSS}
					&$m_{1}$-SFNN-R & 76.38$\pm$0.90 &76.84$\pm$0.86 &86.59 &2.08$\pm$0.24 &0.003$\pm$0.002 &5 \\
					&$m_{2}$-SFNN-R & 84.25$\pm$0.82 &84.56$\pm$0.78 &95.87 &\textbf{1.90$\pm$0.26} &\textbf{0.002$\pm$0.001} &4 \\
					&$m_{3}$-SFNN-R &75.27$\pm$0.80 &75.75$\pm$0.77 &84.03 &2.70$\pm$0.18 &\textbf{0.002$\pm$0.001} &6 \\
					&STWD-SFNN &\textbf{85.63$\pm$0.70} &\textbf{85.91$\pm$0.67} &\textbf{98.30} &7.09$\pm$0.37 &\underline{0.003$\pm$0.000} &\textbf{2.00$\pm$0.00} \\ \hline
						
					\multirow{4}{*}{SE}
					&$m_{1}$-SFNN-R & 79.08$\pm$0.40 &75.18$\pm$0.50 &79.54 &77.22$\pm$1.60 &0.018$\pm$0.007 &3 \\
					&$m_{2}$-SFNN-R & 78.81$\pm$0.38 &69.76$\pm$0.52 &\textbf{91.65} &49.81$\pm$0.92 &\textbf{0.007$\pm$0.004} &\textbf{2} \\
					&$m_{3}$-SFNN-R &79.08$\pm$0.40 &75.18$\pm$0.50 &79.54 &\textbf{26.60$\pm$0.07} &\textbf{0.007$\pm$0.004} &3 \\
					&STWD-SFNN &\textbf{80.89$\pm$0.29} &\textbf{78.31$\pm$0.36} &\underline{87.14} &76.63$\pm$2.09 &\underline{0.011$\pm$0.000} &\textbf{2.00$\pm$0.00} \\ \hline						
						
					\multirow{4}{*}{HTRU}
					&$m_{1}$-SFNN-R & 90.93$\pm$0.37 &86.72$\pm$0.53 &83.95 &\textbf{7.65$\pm$0.26} &0.004$\pm$0.002 &4 \\
					&$m_{2}$-SFNN-R & 84.95$\pm$0.56 &87.59$\pm$0.40 &96.52 &8.45$\pm$0.17 &\textbf{0.003$\pm$0.000} &3 \\
					&$m_{3}$-SFNN-R &90.93$\pm$0.37 &86.72$\pm$0.53 &83.95 &9.04$\pm$0.19 &0.004$\pm$0.002 &4 \\
					&STWD-SFNN &\textbf{91.80$\pm$0.30} &\textbf{92.72$\pm$0.23} &\textbf{98.34} &10.43$\pm$0.61 &\underline{0.004$\pm$0.000} &\textbf{2.00$\pm$0.00} \\ \hline
						
					\multirow{4}{*}{DCC}
					&$m_{1}$-SFNN-R & 77.95$\pm$0.04 &68.52$\pm$0.57 &50.50 &24.36$\pm$1.46 &0.009$\pm$0.005 &6 \\
					&$m_{2}$-SFNN-R & 74.07$\pm$0.26 &68.31$\pm$0.41 &58.51 &11.77$\pm$0.31 &0.004$\pm$0.002 &5 \\
					&$m_{3}$-SFNN-R &76.77$\pm$0.46 &\textbf{75.62$\pm$0.53} &\textbf{67.49} &10.46$\pm$0.33 &0.004$\pm$0.002 &7 \\
					&STWD-SFNN &\textbf{78.13$\pm$0.40} &\underline{69.70$\pm$0.55} &\underline{63.59} &\textbf{7.89$\pm$0.32} &\textbf{0.002$\pm$0.000} &\textbf{2.00$\pm$0.00} \\ \hline
						
					\multirow{4}{*}{ESR}
					&$m_{1}$-SFNN-R & 84.27$\pm$0.58 &80.19$\pm$0.75 &54.48 &4.97$\pm$0.26 &0.004$\pm$0.002 &14 \\
					&$m_{2}$-SFNN-R & 83.82$\pm$0.86 &79.31$\pm$1.14 &54.99 &8.54$\pm$0.30 &0.006$\pm$0.001 &8 \\
					&$m_{3}$-SFNN-R &83.97$\pm$0.67 &81.41$\pm$0.87 &\textbf{59.85} &11.32$\pm$0.18 &0.011$\pm$0.005 &19 \\
					&STWD-SFNN &\textbf{84.35$\pm$0.47} &\textbf{81.55$\pm$0.55} &54.25 &\textbf{4.29$\pm$0.06} &\textbf{0.003$\pm$0.000} &\textbf{2.00$\pm$0.00} \\ \hline
						
					\multirow{4}{*}{BM}
					&$m_{1}$-SFNN-R & \textbf{88.39$\pm$0.35} &\textbf{83.27$\pm$0.49} &75.83 &32.27$\pm$0.62 &0.006$\pm$0.001 &5 \\
					&$m_{2}$-SFNN-R & \textbf{88.39$\pm$0.35} &\textbf{83.27$\pm$0.49} &75.83 &32.35$\pm$0.23 &0.011$\pm$0.006 &5 \\
					&$m_{3}$-SFNN-R &84.31$\pm$0.33 &81.18$\pm$0.47 &\textbf{79.84} &25.75$\pm$1.73 &0.005$\pm$0.001 &7 \\
					&STWD-SFNN &\underline{87.56$\pm$0.34} &\underline{82.85$\pm$0.48} &\underline{78.80} &\textbf{15.19$\pm$0.36} &\textbf{0.003$\pm$0.001} &\textbf{2.00$\pm$0.00} \\ \hline
						
					\multirow{4}{*}{PCB}
					&$m_{1}$-SFNN-R & 63.45$\pm$1.51 &75.70$\pm$1.15 &74.30 &9.26$\pm$0.24 &0.004$\pm$0.002 &9 \\
					&$m_{2}$-SFNN-R & 74.45$\pm$0.76 &83.50$\pm$0.62 &\textbf{76.13} &7.25$\pm$0.25 &0.003$\pm$0.000 &6 \\
					&$m_{3}$-SFNN-R &62.01$\pm$7.90 &74.24$\pm$5.40 &70.31 &7.65$\pm$0.10 &0.003$\pm$0.001 &12 \\
					&STWD-SFNN &\textbf{97.31$\pm$0.49} &\textbf{96.54$\pm$0.52} &59.23 &\textbf{5.44$\pm$0.06} &\textbf{0.002$\pm$0.000} &\textbf{2.00$\pm$0.00} \\ \hline
						
					\multirow{4}{*}{SB}
					&$m_{1}$-SFNN-R & 73.60$\pm$0.96 &78.45$\pm$0.76 &\textbf{85.41} &1.93$\pm$0.29 &\textbf{0.002$\pm$0.000} &4 \\
					&$m_{2}$-SFNN-R & 73.60$\pm$0.96 &78.45$\pm$0.76 &\textbf{85.41} &\textbf{1.89$\pm$0.22} &0.003$\pm$0.002 &4 \\
					&$m_{3}$-SFNN-R &\textbf{76.06$\pm$1.16} &\textbf{79.27$\pm$1.12} &62.58 &2.03$\pm$0.23 &0.003$\pm$0.002 &5 \\
					&STWD-SFNN &65.80$\pm$0.78 &71.61$\pm$0.90 &57.62 &3.48$\pm$0.18 &\underline{0.003$\pm$0.000} &\textbf{2.90$\pm$0.20} \\ \hline
						
					\multirow{4}{*}{EOL}
					&$m_{1}$-SFNN-R & 74.90$\pm$2.11 &75.06$\pm$2.10 &59.66 &\textbf{0.34$\pm$0.17} &0.001$\pm$0.001 &5 \\
					&$m_{2}$-SFNN-R & 74.85$\pm$2.18 &75.14$\pm$2.16 &60.83 &0.36$\pm$0.17 &\textbf{0.001$\pm$0.000} &4 \\
					&$m_{3}$-SFNN-R &75.37$\pm$2.21 &75.24$\pm$2.09 &63.70 &0.44$\pm$0.19 &\textbf{0.001$\pm$0.000} &6 \\
					&STWD-SFNN &\textbf{83.66$\pm$1.95} &\textbf{81.53$\pm$0.01} &\textbf{69.78} &0.75$\pm$0.07 &\textbf{0.001$\pm$0.000} &\textbf{2.00$\pm$0.00} \\ \hline
						
					\multirow{4}{*}{OD}
					&$m_{1}$-SFNN-R & 72.56$\pm$0.54 &64.89$\pm$0.62 &\textbf{70.95} &10.96$\pm$0.31 &0.005$\pm$0.003 &3 \\
					&$m_{2}$-SFNN-R &72.56$\pm$0.54 &64.89$\pm$0.62 &\textbf{70.95} &10.45$\pm$0.23 &0.004$\pm$0.002 &3 \\
					&$m_{3}$-SFNN-R & 64.11$\pm$0.82 &63.84$\pm$0.72 &60.94 &10.89$\pm$0.23 &0.005$\pm$0.003 &4 \\
					&STWD-SFNN &\textbf{76.33$\pm$0.41} &\textbf{66.71$\pm$0.57} &\underline{67.95} &\textbf{4.81$\pm$0.09} &\textbf{0.002$\pm$0.000} &\textbf{2.00$\pm$0.00} \\ \hline
						
					\multirow{4}{*}{ROE}
					&$m_{1}$-SFNN-R & 80.79$\pm$0.54 &79.80$\pm$0.58 &64.81 &2.01$\pm$0.03 &0.002$\pm$0.000 &5 \\
					&$m_{2}$-SFNN-R & 68.86$\pm$0.77 &70.80$\pm$0.73 &63.48 &\textbf{1.78$\pm$0.01} &\textbf{0.001$\pm$0.000} &4 \\
					&$m_{3}$-SFNN-R &\textbf{85.07$\pm$0.53} &83.11$\pm$0.66 &65.74 &2.04$\pm$0.01 &\textbf{0.001$\pm$0.000} &6 \\
					&STWD-SFNN &\underline{84.35$\pm$0.98} &\textbf{85.00$\pm$0.96} &\textbf{80.19} &2.74$\pm$0.11 &\textbf{0.001$\pm$0.000} &\textbf{2.00$\pm$0.00} \\ \hline
						
					\multirow{4}{*}{SSMCR}
					&$m_{1}$-SFNN-R & \textbf{91.35$\pm$0.10} &\textbf{88.37$\pm$0.15} &68.74 &81.55$\pm$1.25 &\textbf{0.007$\pm$0.000} &3 \\
					&$m_{2}$-SFNN-R & 89.74$\pm$0.09 &87.56$\pm$0.19 &68.27 &233.60$\pm$0.37 &0.022$\pm$0.000 &\textbf{2} \\
					&$m_{3}$-SFNN-R &77.03$\pm$0.21 &80.90$\pm$0.17 &60.28 &385.09$\pm$14.33 &0.048$\pm$0.011 &3 \\
					&STWD-SFNN &88.41$\pm$0.32 &\underline{87.22$\pm$0.37} &\textbf{70.63} &\textbf{76.04$\pm$2.53} &0.029$\pm$0.007 &\textbf{2.00$\pm$0.00} \\ \hline										
				\end{tabular}}
				\label{tab_sfnn}
		\end{table*} 	
	    
	    1. The accuracy, weighted-f1, ROC, and AUC of STWD-SFNN are better than those of SFNN-R. For example, on the EOL dataset, Table \ref{tab_sfnn} reports that the accuracy, weighted-f1, and AUC of STWD-SFNN are 83.66\%, 81.53\%, and 69.78\%, respectively. However, the best performance in the SFNN-R models is $ m_{3} $-SFNN-R whose accuracy, weighted-f1, and AUC are 75.37\%, 75.24\%, and 63.70\%, respectively. Meanwhile, as shown in Fig. \ref{fig_sfnn}, the ROC curve of STWD-SFNN is at the top left of $ m_{3} $-SFNN-R. Similar phenomena can be found in other datasets. The reason is as follows. Compared with SFNN-R which only contains two-way decisions with POS and NEG, STWD-SFNN adds the delayed decision region to store the difficult-to-classify instances, and learns the features of these instances more pertinently in each turn, thereby enhancing the performance of STWD-SFNN. Therefore, the generalization ability of STWD-SFNN is better than that of SFNN-R.
	    
	    2. The number of hidden layer nodes of STWD-SFNN is lower than that of SFNN-R. For example, on the OD dataset, Table \ref{tab_sfnn} reports that the number of hidden layer nodes of STWD-SFNN is two. However, the competitive models perform better in $ m_{1} $-SFNN-R and $ m_{2} $-SFNN-R, with a value of three. Similar phenomena can be found in other datasets. The reason is as follows. Compared with SFNN-R, which statically determines the network topology, the compactness of network topology of STWD-SFNN depends on the number of instances in BND. As the number of difficult-to-classify instances decreases, the number of hidden layer nodes of STWD-SFNN gradually decreases. Therefore, the network topology of STWD-SFNN is better than that of SFNN-R.
	    
	    3. The running time of STWD-SFNN is faster than SFNN-R on some datasets. For example, on the BM dataset, Table \ref{tab_sfnn} reports that the training time of STWD-SFNN is 15.19s, and $ m_{3} $-SFNN-R is the least time-consuming model of SFNN-R with a value of 25.75s. Similar phenomena can be found in other datasets. The reason is as follows. Compared with SFNN-R, which needs to repeatedly learn the features of all instances, STWD-SFNN only learns all instances in the first turn. After the second turn, STWD-SFNN focuses on the features of difficult-to-classify instances by constructing granularity layers and sequential threshold parameters, which takes a relatively short time. Therefore, the efficiency of STWD-SFNN is better than that of SFNN-R.

		\subsubsection{ Comparison with dynamic models} \label{exp_gs_pso}
		To verify the effectiveness of STWD-SFNN in dynamically constructing network topology, we add three new datasets on the basis of previous work \cite{Cheng2021INS} and choose GS-SFNN, PSO-SFNN, TWD-SFNN-R, and STWD-SFNN-NK as the competitive models. Fig. \ref{fig_gs_pso} shows the comparison of ROC curves for dynamic models and STWD-SFNN, and Table \ref{tab_gs_pso} lists the indices for those models. The results give rise to the following observations.
		
		\begin{figure*}[hp]
			\centering
			\subfigure[ONP]{
				\includegraphics[width=2.1in]{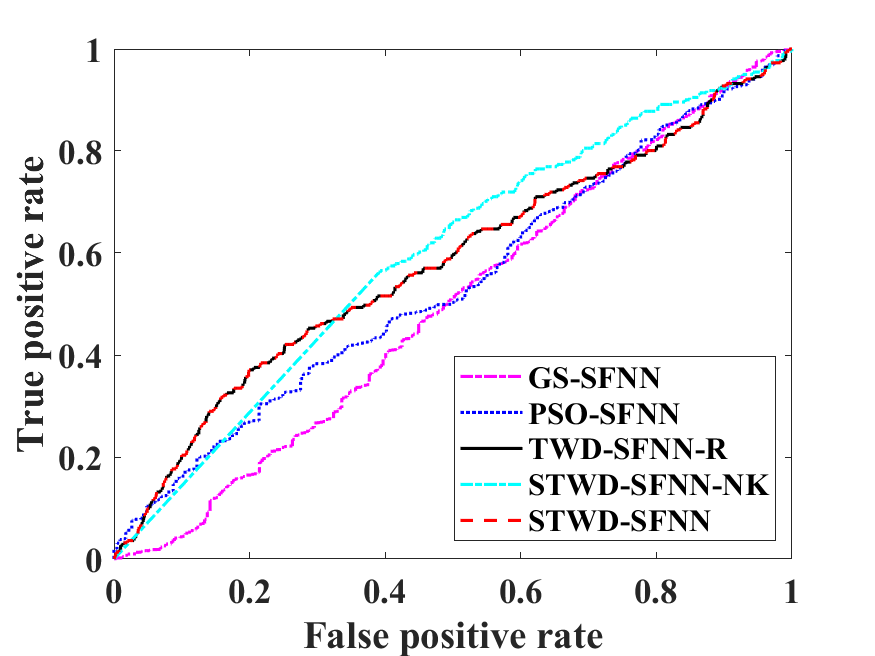}}
			\hfill
			\centering
			\subfigure[QSAR]{
				\includegraphics[width=2.1in]{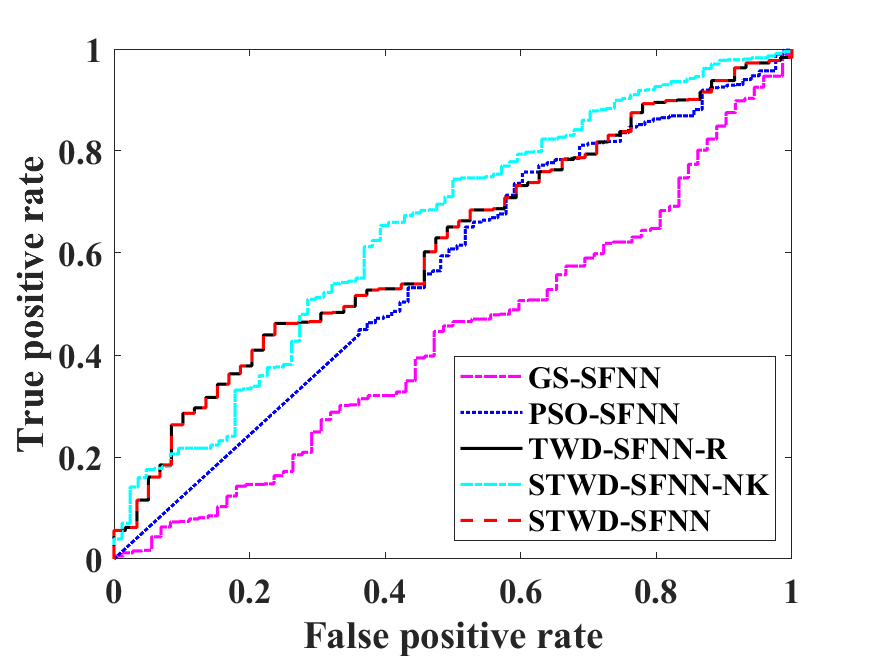}}
			\hfill
			\centering
			\subfigure[OSP]{
				\includegraphics[width=2.1in]{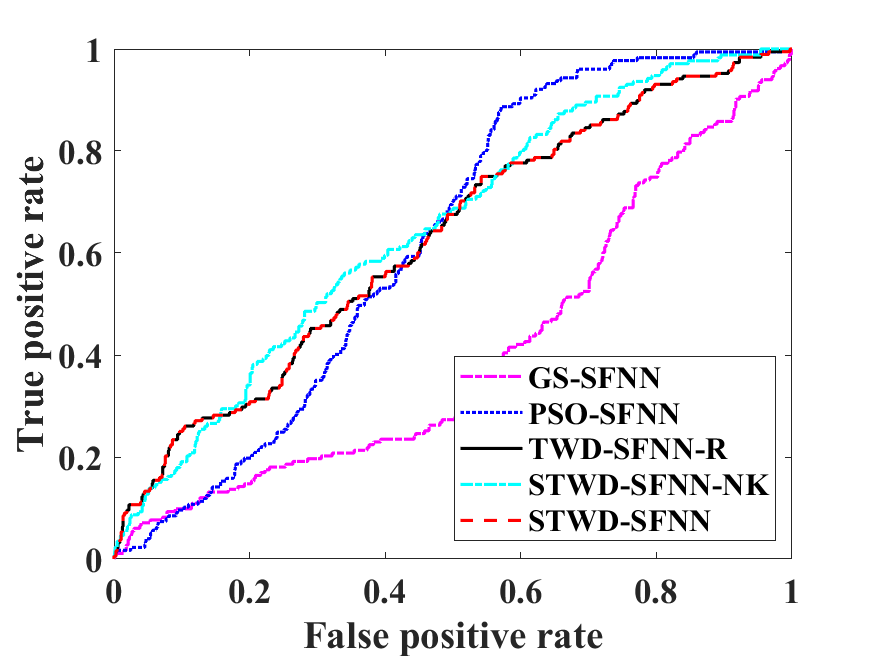}}
			\hfill
			\centering
			\subfigure[EGSS]{
				\includegraphics[width=2.1in]{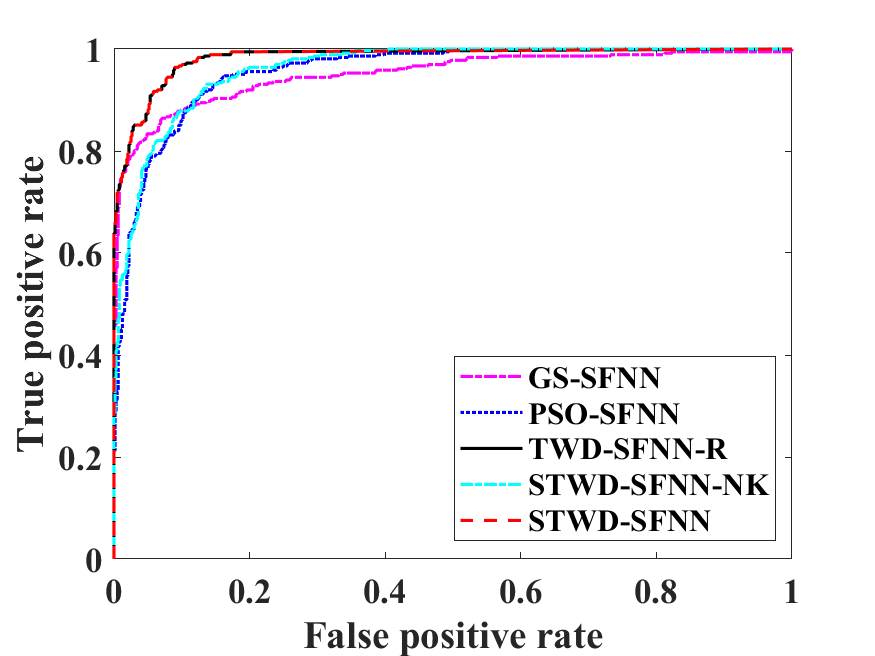}}
			\hfill
			\centering
			\subfigure[SE]{
				\includegraphics[width=2.1in]{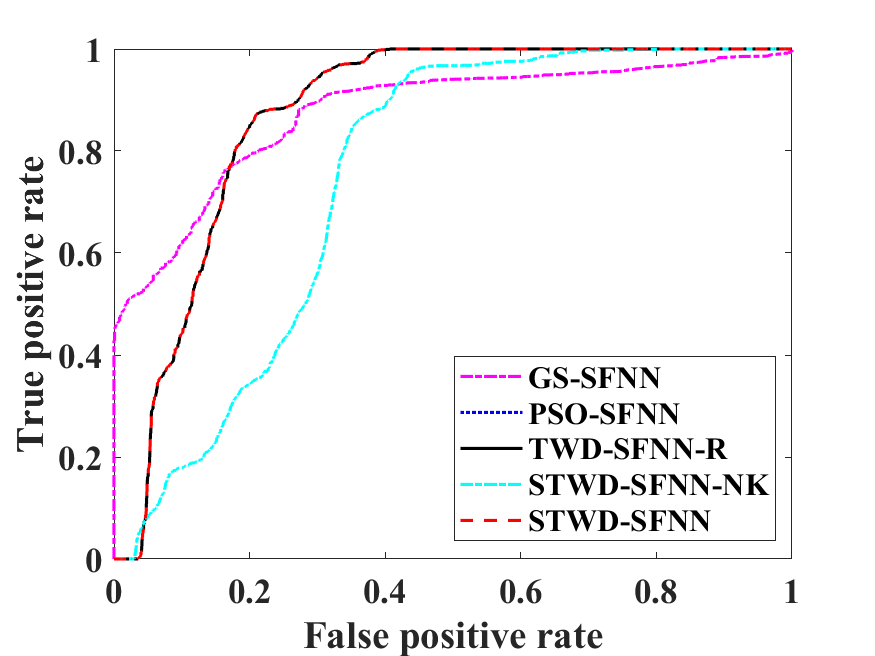}}
			\hfill
			\centering
			\subfigure[HTRU]{
				\includegraphics[width=2.1in]{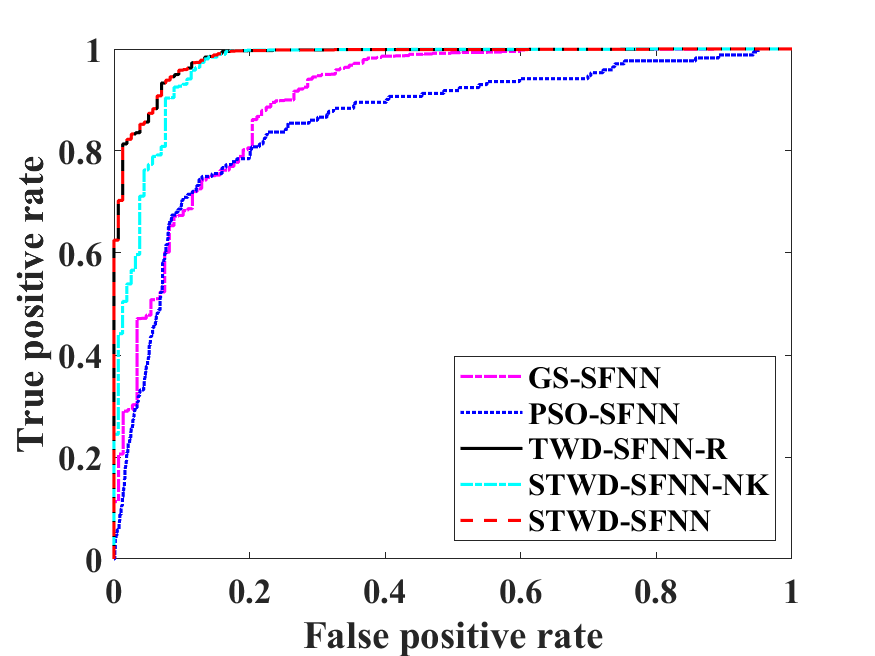}}
			\hfill
			\centering
			\subfigure[DCC]{
				\includegraphics[width=2.1in]{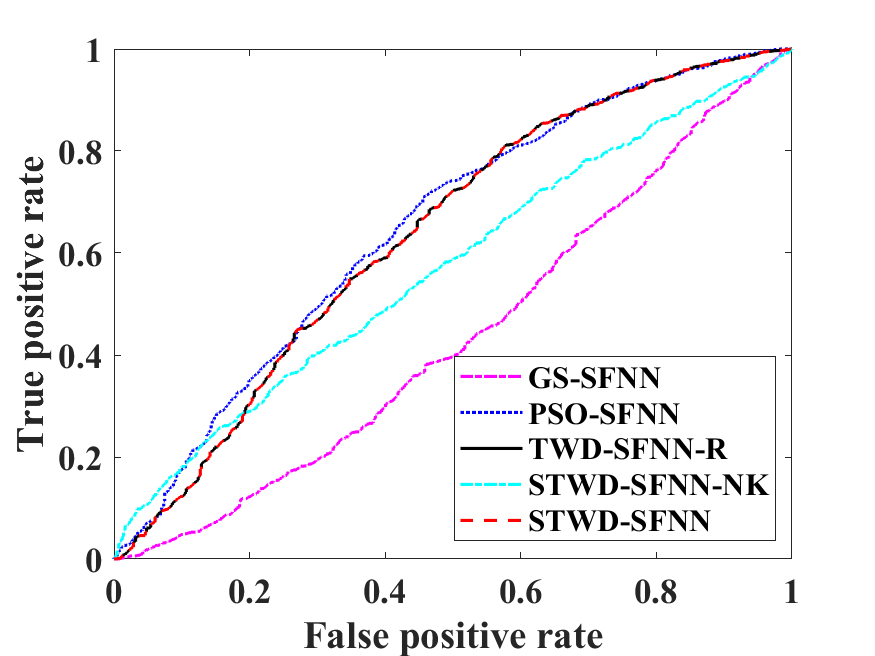}}
			\hfill
			\centering
			\subfigure[ESR]{
				\includegraphics[width=2.1in]{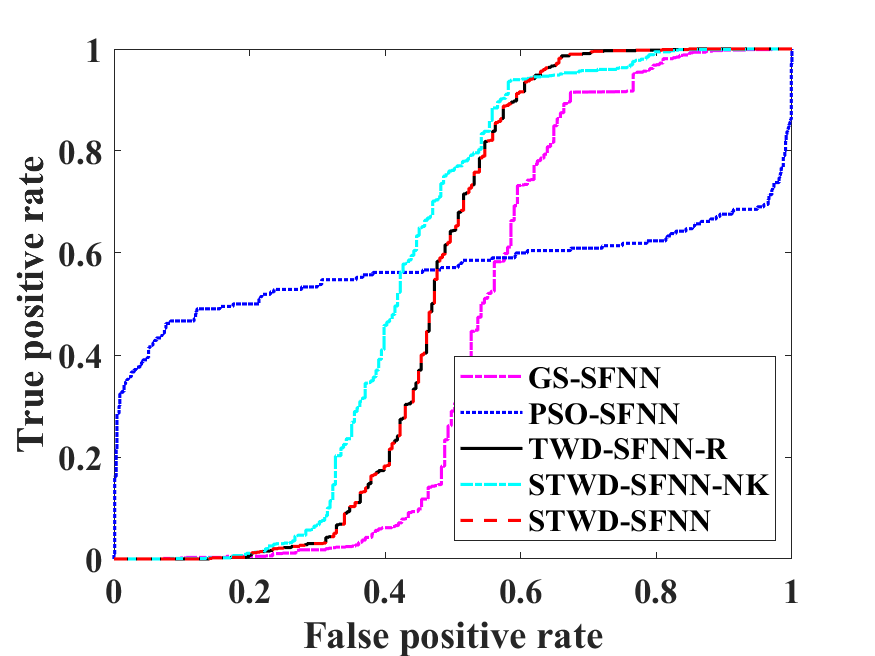}}	
			\hfill
			\centering
			\subfigure[BM]{
				\includegraphics[width=2.1in]{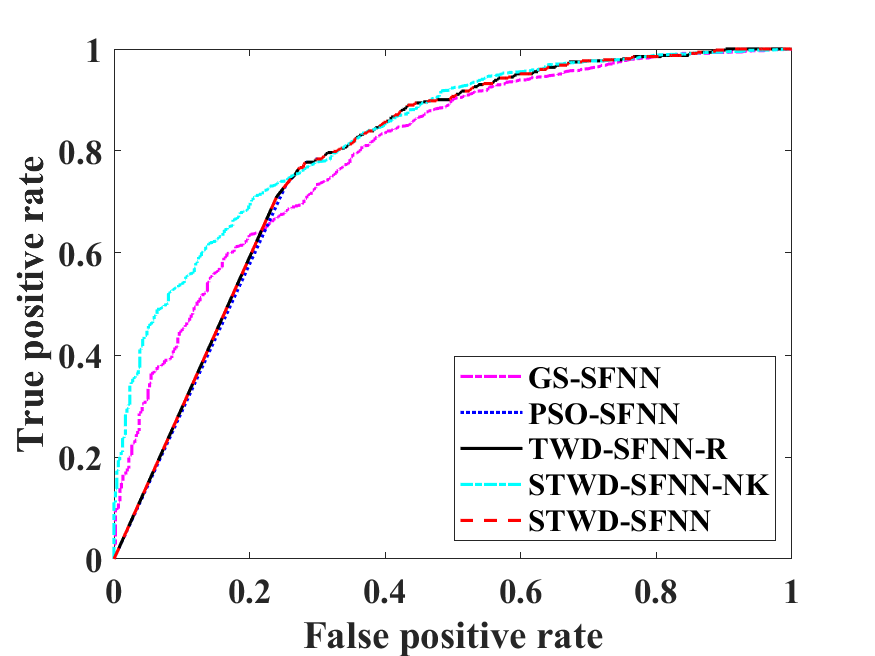}}
			\hfill
			\centering
			\subfigure[PCB]{
				\includegraphics[width=2.1in]{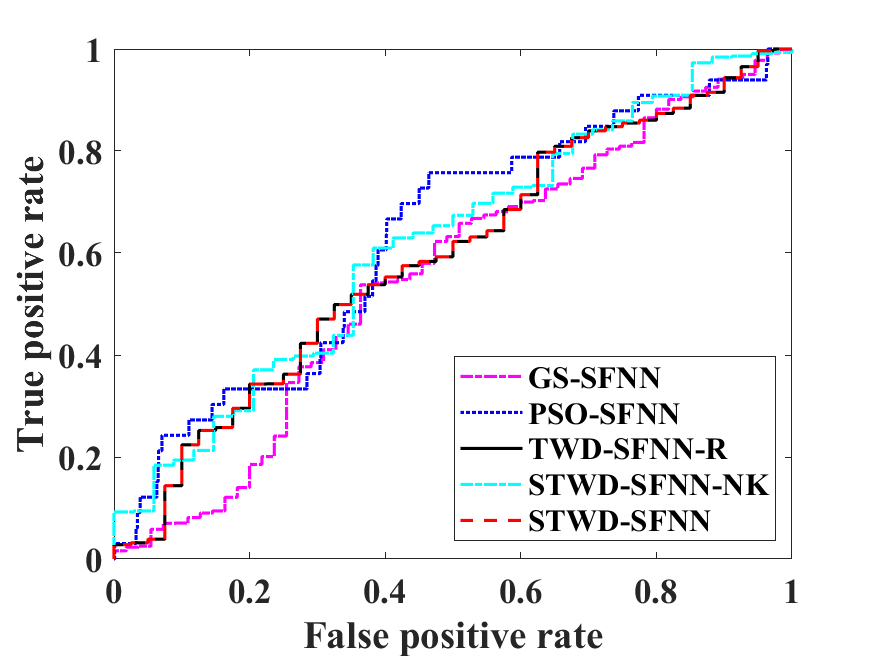}}
			\hfill	
			\centering
			\subfigure[SB]{
				\includegraphics[width=2.1in]{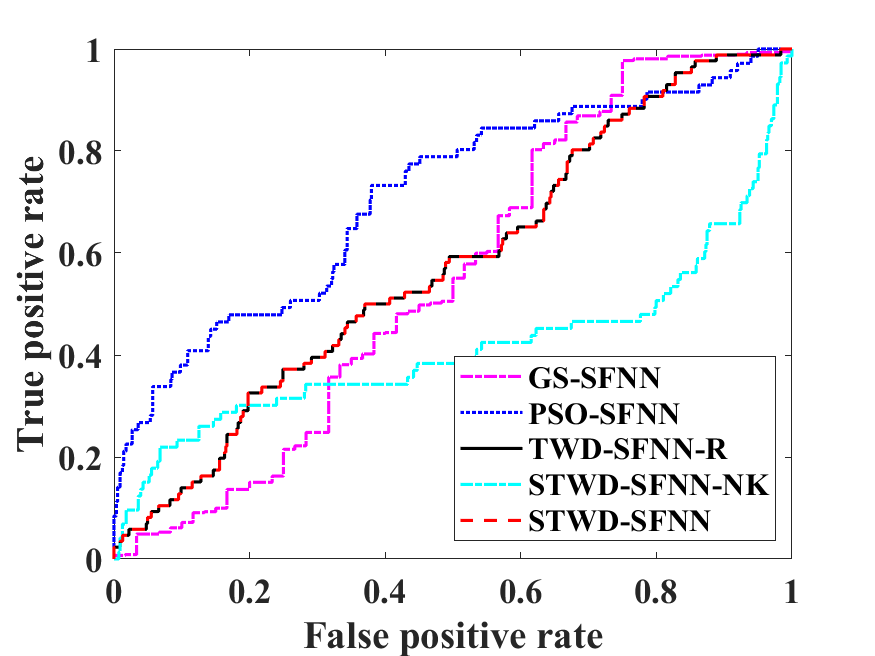}}
			\hfill	
			\centering
			\subfigure[EOL]{
				\includegraphics[width=2.1in]{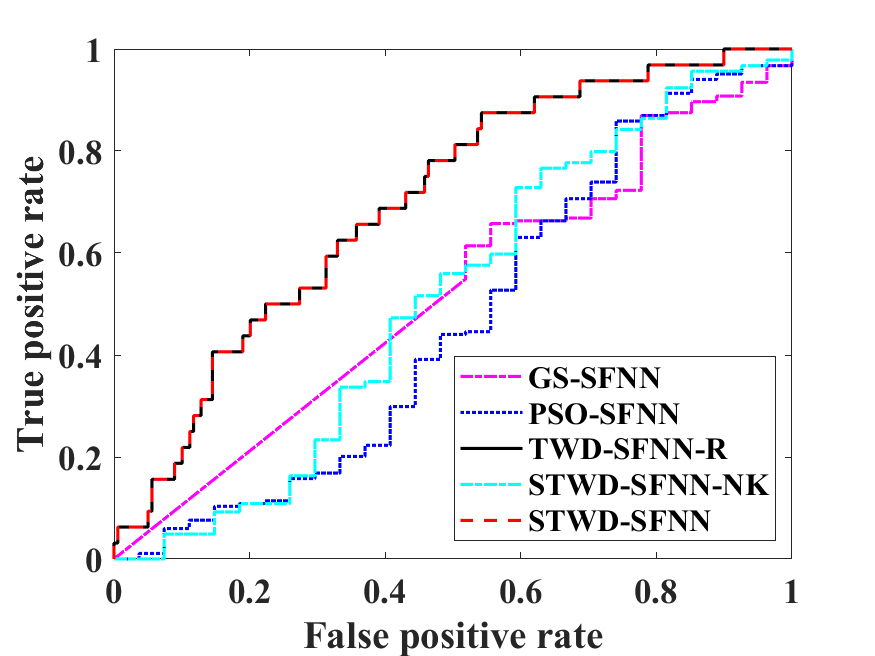}}
			\hfill	
			\centering	
			\subfigure[OD]{
				\includegraphics[width=2.1in]{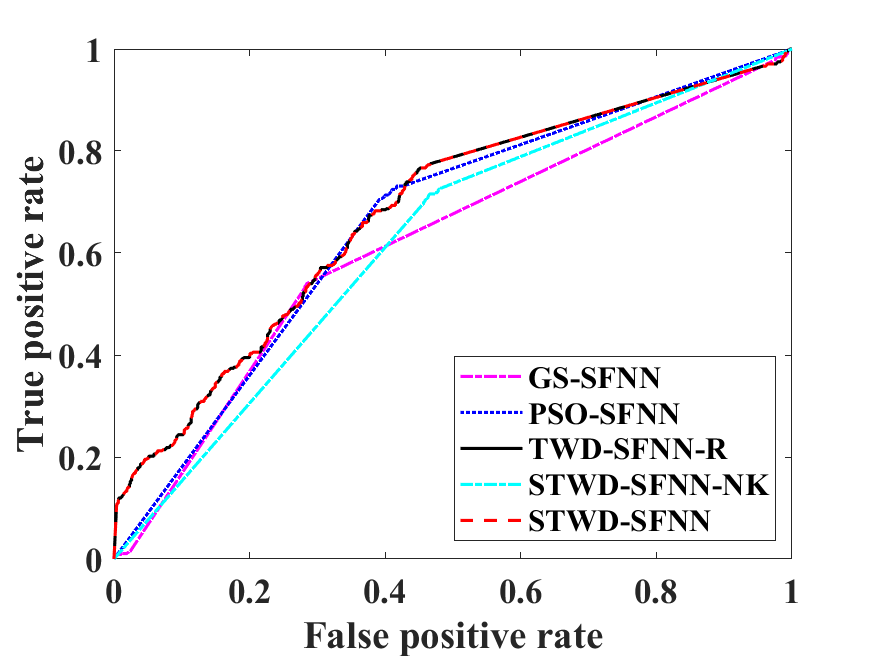}}
			\hfill	
			\centering
			\subfigure[ROE]{
				\includegraphics[width=2.1in]{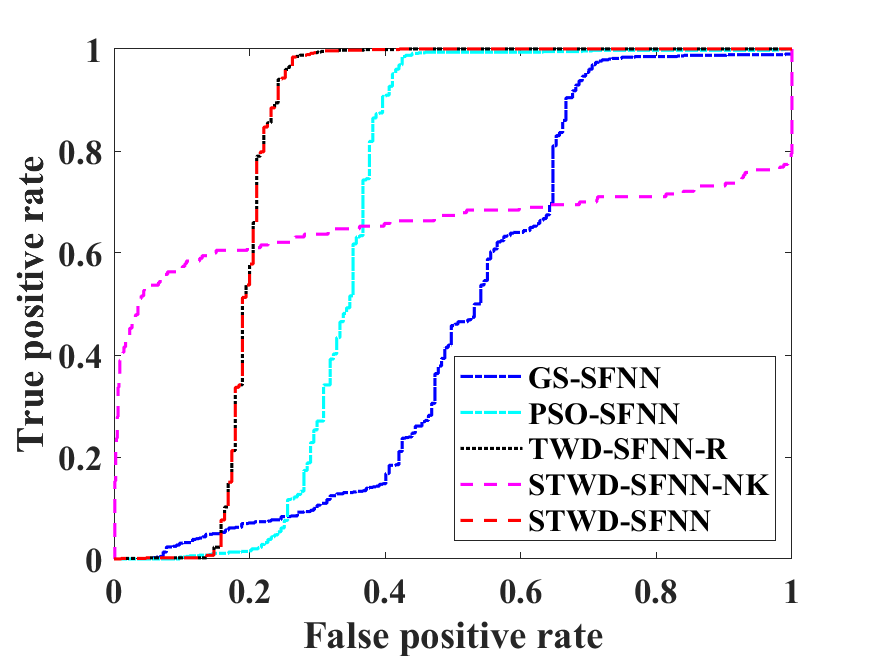}}
			\hfill	
			\centering
			\subfigure[SSMCR]{
				\includegraphics[width=2.1in]{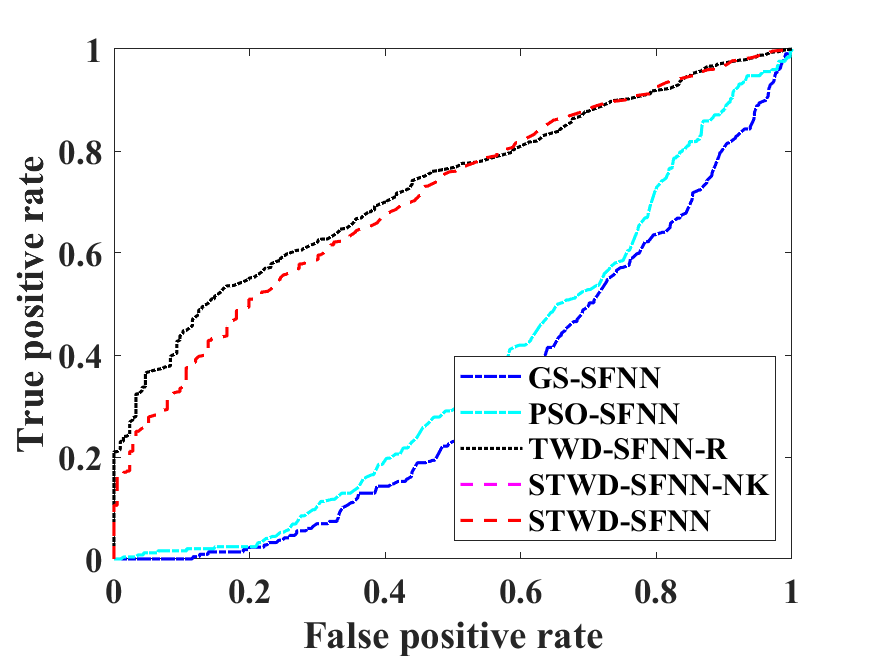}}
			\hfill	
			\small
			\captionsetup{font={small}}
			\caption{Comparison of ROC curves of dynamic models and STWD-SFNN }
			\label{fig_gs_pso}
		\end{figure*}

		\begin{table*}[hp]
			\setlength{\abovecaptionskip}{0cm}    % 段前
			\setlength{\belowcaptionskip}{0cm}    % 段后
			\small
			\captionsetup{font={small}}
			\newcommand{\tabincell}[2]{\begin{tabular}{@{}#1@{}}#2\end{tabular}}
			\caption{ Comparison of dynamic models and STWD-SFNN }
			\centering
			\resizebox{\textwidth}{120mm}{
				\setlength{\tabcolsep}{1mm}{
					\begin{tabular}{cccccccc}
						\hline
						Dataset   &Model  &\tabincell{c}{Accuracy \\ (\%)}   &\tabincell{c}{Weighted- \\ f1(\%)}  &\tabincell{c}{AUC \\ (\%)}   &\tabincell{c}{Training \\ time(s)}  &\tabincell{c}{Test time (s)} &\tabincell{c}{Nodes}\\ \hline
							
						\multirow{5}{*}{ONP}
						&GS-SFNN &47.24$\pm$17.13 &54.65$\pm$17.01 &49.71 &55.15$\pm$11.72 &0.005$\pm$0.001 &3.10$\pm$0.63 \\
						&PSO-SFNN &92.20$\pm$2.34 &81.64$\pm$18.18 &53.83 &91.55$\pm$3.51 &\textbf{0.003$\pm$0.000} &4.35$\pm$1.01 \\
						&TWD-SFNN-R &\textbf{94.41$\pm$0.22} &\underline{85.87$\pm$0.25} &\underline{57.87} &\textbf{9.56$\pm$0.11} &\textbf{0.003$\pm$0.000} &\textbf{1.00$\pm$0.00} \\
						&STWD-SFNN-NK &84.16$\pm$0.24 &\textbf{86.63$\pm$0.25} &\textbf{59.35} &44.70$\pm$6.10 &0.008$\pm$0.001 &\textbf{1.00$\pm$0.00} \\
						&STWD-SFNN &\textbf{94.41$\pm$0.22} &\underline{85.87$\pm$0.26} &\underline{57.87} &\underline{15.76$\pm$0.76} &\textbf{0.003$\pm$0.000} &\textbf{1.00$\pm$0.00} \\ \hline
							
						\multirow{5}{*}{QSAR}
						&GS-SFNN &49.47$\pm$11.29 &57.78$\pm$10.70 &43.16 &40.13$\pm$4.71 &0.011$\pm$0.000 &2.70$\pm$0.31 \\
						&PSO-SFNN &\textbf{85.92$\pm$2.48} &\textbf{85.49$\pm$1.66} &56.32 &190.05$\pm$8.61 &0.017$\pm$0.002 &5.42$\pm$1.61 \\
						&TWD-SFNN-R &\underline{78.26$\pm$1.11} &81.21$\pm$1.13 &\underline{61.61} &14.62$\pm$0.33 &0.011$\pm$0.002 &\underline{2.00$\pm$0.00} \\
						&STWD-SFNN-NK &67.23$\pm$1.15 &74.42$\pm$1.04 &64.54 &13.61$\pm$1.18 &0.009$\pm$0.001 &\textbf{1.00$\pm$0.00} \\
						&STWD-SFNN &\underline{78.26$\pm$1.11} &\underline{81.21$\pm$1.13} &\underline{61.61} &\textbf{12.43$\pm$0.15} &\textbf{0.008$\pm$0.001} &\underline{2.00$\pm$0.00} \\ \hline
							
						\multirow{5}{*}{OSP}
						&GS-SFNN &44.76$\pm$15.67 &42.34$\pm$18.29 &40.62 &28.78$\pm$5.95 &0.004$\pm$0.000 &2.90$\pm$0.63 \\
						&PSO-SFNN &\textbf{80.98$\pm$3.59} &68.07$\pm$15.27 &62.10 &28.88$\pm$2.60 &0.002$\pm$0.001 &5.62$\pm$1.44 \\
						&TWD-SFNN-R &68.42$\pm$0.65 &\textbf{70.01$\pm$0.80} &\underline{62.52} &\textbf{5.04$\pm$0.43} &\textbf{0.002$\pm$0.000} &\underline{2.00$\pm$0.00} \\
						&STWD-SFNN-NK &65.01$\pm$0.65 &66.98$\pm$0.82 &\textbf{64.77} &5.15$\pm$0.60 &\textbf{0.002$\pm$0.000} &\textbf{1.00$\pm$0.00} \\
						&STWD-SFNN &\underline{68.43$\pm$0.66} &\textbf{70.01$\pm$0.81} &\underline{62.52} &5.55$\pm$0.51 &\textbf{0.002$\pm$0.000} &\underline{2.00$\pm$0.00} \\ \hline
							
						\multirow{5}{*}{EGSS}
						&GS-SFNN &62.56$\pm$14.77 &58.48$\pm$16.67 &95.07 &18.36$\pm$2.41 &0.004$\pm$0.001 &2.50$\pm$0.45 \\
						&PSO-SFNN &84.32$\pm$5.05 &83.30$\pm$7.17 &95.60 &36.54$\pm$3.29 &0.004$\pm$0.001 &5.53$\pm$1.51 \\
						&TWD-SFNN-R &85.60$\pm$0.70 &85.88$\pm$0.67 &\textbf{98.30} &6.04$\pm$0.65 &\underline{0.003$\pm$0.000} &\underline{2.00$\pm$0.00} \\
						&STWD-SFNN-NK &82.26$\pm$0.75 &82.61$\pm$0.72 &96.15 &\textbf{4.71$\pm$0.57} &\textbf{0.002$\pm$0.000} &\textbf{1.00$\pm$0.00} \\
						&STWD-SFNN &\textbf{85.63$\pm$0.70} &\textbf{85.91$\pm$0.67} &\textbf{98.30} &7.09$\pm$0.37 &\underline{0.003$\pm$0.000} &\underline{2.00$\pm$0.00} \\ \hline
							
						\multirow{5}{*}{SE}
						&GS-SFNN &55.84$\pm$16.66 &35.60$\pm$19.26 &87.58 &42.24$\pm$8.69 &0.003$\pm$0.000 &2.50$\pm$0.45 \\
						&PSO-SFNN &79.01$\pm$8.96 &61.05$\pm$22.36 &87.14 &192.66$\pm$12.37 &0.004$\pm$0.000 &4.83$\pm$1.82 \\
						&TWD-SFNN-R &\textbf{80.89$\pm$0.29} &\textbf{78.31$\pm$0.36} &\textbf{87.14} &\textbf{13.84$\pm$0.038} &\textbf{0.002$\pm$0.000} &\underline{2.00$\pm$0.00} \\
						&STWD-SFNN-NK &79.12$\pm$0.29 &70.21$\pm$0.03 &74.58 &17.87$\pm$0.25 &\textbf{0.002$\pm$0.000} &\textbf{1.00$\pm$0.00} \\
						&STWD-SFNN &\textbf{80.89$\pm$0.29} &\textbf{78.31$\pm$0.36} &\textbf{87.14} &20.86$\pm$0.57 &\underline{0.003$\pm$0.000} &\underline{2.00$\pm$0.00} \\ \hline
							
						\multirow{5}{*}{HTRU}
						&GS-SFNN &46.78$\pm$17.71 &52.19$\pm$17.94 &90.43 &40.34$\pm$9.08 &0.005$\pm$0.000 &3.20$\pm$0.72 \\
						&PSO-SFNN &89.71$\pm$2.35 &80.99$\pm$18.17 &86.00 &108.00$\pm$9.76 &0.007$\pm$0.001 &5.64$\pm$1.91 \\
						&TWD-SFNN-R &91.75$\pm$0.27 &92.68$\pm$0.21 &\textbf{98.34} &10.88$\pm$0.55 &0.005$\pm$0.001 &\underline{2.00$\pm$0.00} \\
						&STWD-SFNN-NK &84.50$\pm$0.48 &87.26$\pm$0.35 &96.70 &\textbf{5.53$\pm$0.30} &\textbf{0.002$\pm$0.000} &\textbf{1.00$\pm$0.00} \\
						&STWD-SFNN &\textbf{91.80$\pm$0.30} &\textbf{92.72$\pm$0.23} &\textbf{98.34} &10.43$\pm$0.61 &\underline{0.004$\pm$0.000} &\underline{2.00$\pm$0.00} \\ \hline
							
						\multirow{5}{*}{DCC}
						&GS-SFNN &49.15$\pm$16.12 &25.93$\pm$19.33 &43.78 &84.02$\pm$20.61 &0.008$\pm$0.001 &2.60$\pm$0.68 \\
						&PSO-SFNN &\textbf{78.56$\pm$0.87} &50.32$\pm$22.10 &\textbf{65.01} &65.11$\pm$1.78 &0.003$\pm$0.000 &4.57$\pm$1.32 \\
						&TWD-SFNN-R &\underline{78.13$\pm$0.40} &\textbf{69.70$\pm$0.55} &\underline{63.59} &\textbf{7.64$\pm$0.14} &0.002$\pm$0.001 &\underline{2.00$\pm$0.00} \\
						&STWD-SFNN-NK &76.48$\pm$0.34 &68.59$\pm$0.45 &57.08 &8.72$\pm$0.05 &\textbf{0.002$\pm$0.000} &\textbf{1.00$\pm$0.00} \\
						&STWD-SFNN &\underline{78.13$\pm$0.40} &\textbf{69.70$\pm$0.55} &\underline{63.59} &\underline{7.89$\pm$0.32} &\textbf{0.002$\pm$0.000} &\underline{2.00$\pm$0.00} \\ \hline
							
						\multirow{5}{*}{ESR}
						&GS-SFNN &52.82$\pm$15.78 &36.76$\pm$18.92 &44.23 &21.12$\pm$4.28 &0.006$\pm$0.001 &2.50$\pm$0.45 \\
						&PSO-SFNN &84.24$\pm$1.44 &\textbf{81.77$\pm$1.27} &\textbf{56.96} &149.30$\pm$11.62 &0.010$\pm$0.002 &4.79$\pm$1.68 \\
						&TWD-SFNN-R &84.32$\pm$0.48 &81.53$\pm$0.56 &54.25 &\textbf{3.84$\pm$0.15} &\textbf{0.003$\pm$0.000} &\underline{2.00$\pm$0.00} \\
						&STWD-SFNN-NK &80.00$\pm$0.96 &79.76$\pm$0.01 &56.71 &4.42$\pm$0.06 &\textbf{0.003$\pm$0.000} &\textbf{1.00$\pm$0.00} \\
						&STWD-SFNN &\textbf{84.35$\pm$0.47} &\underline{81.55$\pm$0.55} &54.25 &\underline{4.29$\pm$0.06} &\textbf{0.003$\pm$0.000} &\underline{2.00$\pm$0.00} \\ \hline
							
						\multirow{5}{*}{BM}
						&GS-SFNN &78.90$\pm$15.16 &65.95$\pm$21.63 &80.10 &88.38$\pm$13.81 &0.008$\pm$0.001 &2.70$\pm$0.43 \\
						&PSO-SFNN &86.84$\pm$0.47 &81.74$\pm$0.01 &78.59 &102.65$\pm$7.63 &0.003$\pm$0.000 &3.62$\pm$1.56 \\
						&TWD-SFNN-R &\underline{87.56$\pm$0.34} &\textbf{82.85$\pm$0.48} &78.80 &\textbf{8.36$\pm$0.06} &\textbf{0.002$\pm$0.000} &\underline{2.00$\pm$0.00} \\
						&STWD-SFNN-NK &\textbf{88.72$\pm$0.33} &66.93$\pm$22.31 &\textbf{83.51} &19.48$\pm$0.47 &\textbf{0.002$\pm$0.000} &\textbf{1.00$\pm$0.00} \\
						&STWD-SFNN &\underline{87.56$\pm$0.34} &\textbf{82.85$\pm$0.48} &78.80 &\underline{15.19$\pm$0.36} &0.003$\pm$0.001 &\underline{2.00$\pm$0.00} \\ \hline
							
						\multirow{5}{*}{PCB}
						&GS-SFNN &63.46$\pm$26.28 &47.06$\pm$27.92 &55.58 &43.16$\pm$7.45 &0.005$\pm$0.002 &3.00$\pm$0.60 \\
						&PSO-SFNN &84.34$\pm$7.66 &89.12$\pm$4.75 &\textbf{62.77} &273.66$\pm$11.88 &0.015$\pm$0.003 &5.15$\pm$1.49 \\
						&TWD-SFNN-R &\textbf{97.48$\pm$0.32} &\textbf{96.61$\pm$0.46} &59.23 &4.87$\pm$0.14 &0.003$\pm$0.000 &\underline{2.00$\pm$0.00} \\
						&STWD-SFNN-NK &91.56$\pm$4.58 &93.48$\pm$2.63 &61.70 &\textbf{4.43$\pm$0.06} &\textbf{0.002$\pm$0.000} &\textbf{1.00$\pm$0.00} \\
						&STWD-SFNN &\underline{97.31$\pm$0.49} &\underline{96.54$\pm$0.52} &59.23 &5.44$\pm$0.06 &\textbf{0.002$\pm$0.000} &\underline{2.00$\pm$0.00} \\ \hline
							
						\multirow{5}{*}{SB}
						&GS-SFNN &44.62$\pm$15.12 &49.93$\pm$14.92 &55.25 &7.97$\pm$0.77 &0.003$\pm$0.001 &2.30$\pm$0.31 \\
						&PSO-SFNN &\textbf{70.64$\pm$5.83} &\textbf{74.99$\pm$4.38} &\textbf{71.03} &33.03$\pm$5.31 &0.004$\pm$0.001 &5.69$\pm$1.37 \\
						&TWD-SFNN-R &65.39$\pm$1.05 &71.34$\pm$1.04 &\underline{57.62} &2.49$\pm$0.05 &0.002$\pm$0.000 &3.00$\pm$0.00 \\
						&STWD-SFNN-NK &58.58$\pm$1.32 &66.44$\pm$1.14 &41.34 &\textbf{2.26$\pm$0.06} &\textbf{0.001$\pm$0.000} &\textbf{1.00$\pm$0.00} \\
						&STWD-SFNN &\underline{65.80$\pm$0.78} &\underline{71.61$\pm$0.90} &\underline{57.62} &3.48$\pm$0.18 &0.003$\pm$0.000 &2.90$\pm$0.20 \\ \hline
							
						\multirow{5}{*}{EOL}
						&GS-SFNN &54.76$\pm$14.32 &57.84$\pm$11.50 &51.81 &2.35$\pm$0.47 &0.002$\pm$0.000 &2.40$\pm$0.44 \\
						&PSO-SFNN &83.51$\pm$2.65 &72.10$\pm$16.18 &65.14 &13.40$\pm$1.70 &0.003$\pm$0.001 &4.51$\pm$1.71 \\
						&TWD-SFNN-R &\textbf{83.99$\pm$2.10} &\textbf{81.53$\pm$0.01} &\textbf{69.78} &\textbf{0.47$\pm$0.06} &\textbf{0.001$\pm$0.000} &\underline{2.00$\pm$0.00} \\
						&STWD-SFNN-NK &46.52$\pm$2.16 &53.62$\pm$2.53 &51.01 &1.78$\pm$0.20 &0.002$\pm$0.000 &\textbf{1.00$\pm$0.00} \\
						&STWD-SFNN &\underline{83.66$\pm$1.95} &\textbf{81.53$\pm$0.01} &\textbf{69.78} &\underline{0.75$\pm$0.07} &\textbf{0.001$\pm$0.000} &\underline{2.00$\pm$0.00} \\ \hline
							
						\multirow{5}{*}{OD}
						&GS-SFNN &58.13$\pm$2.63 &59.86$\pm$2.71 &62.12 &15.50$\pm$2.70 &0.003$\pm$0.000 &2.60$\pm$0.53 \\
						&PSO-SFNN &67.26$\pm$1.04 &\textbf{67.02$\pm$1.38} &65.97 &106.85$\pm$14.90 &0.007$\pm$0.001 &6.37$\pm$0.96 \\
						&TWD-SFNN-R &\textbf{76.38$\pm$0.43} &60.03$\pm$13.35 &\textbf{67.95} &\textbf{4.18$\pm$0.09} &\textbf{0.001$\pm$0.000} &2.00$\pm$0.00 \\
						&STWD-SFNN-NK &68.15$\pm$0.71 &63.06$\pm$0.65 &62.46 &8.13$\pm$0.10 &\underline{0.002$\pm$0.000} &\textbf{1.00$\pm$0.00} \\
						&STWD-SFNN &\underline{76.33$\pm$0.41} &\underline{66.71$\pm$0.57} &\textbf{67.95} &\underline{4.81$\pm$0.09} &\underline{0.002$\pm$0.000} &\underline{2.00$\pm$0.00} \\ \hline
							
						\multirow{5}{*}{ROE}
						&GS-SFNN &48.72$\pm$17.15 &50.30$\pm$18.63 &48.40 &46.83$\pm$27.58 &0.006$\pm$0.005 &2.80$\pm$0.78 \\
						&PSO-SFNN &82.26$\pm$8.52 &83.14$\pm$7.41 &66.40 &91.05$\pm$16.67 &0.008$\pm$0.004 &5.12$\pm$2.08 \\
						&TWD-SFNN-R &84.19$\pm$0.91 &84.85$\pm$0.90 &\textbf{80.19} &\textbf{1.80$\pm$0.04} &\textbf{0.001$\pm$0.000} &\underline{2.00$\pm$0.00} \\
						&STWD-SFNN-NK &81.74$\pm$1.00 &74.02$\pm$1.42 &65.89 &3.38$\pm$0.13 &\textbf{0.001$\pm$0.000} &\textbf{1.00$\pm$0.00} \\
						&STWD-SFNN &\textbf{84.35$\pm$0.98} &\textbf{85.00$\pm$0.96} &\textbf{80.19} &\underline{2.74$\pm$0.11} &\textbf{0.001$\pm$0.000} &\underline{2.00$\pm$0.00} \\ \hline
							
						\multirow{5}{*}{SSMCR}
						&GS-SFNN &48.68$\pm$27.47 &35.03$\pm$28.61 &32.04 &118.43$\pm$34.05 &0.011$\pm$0.000 &3.00$\pm$0.89 \\
						&PSO-SFNN &\textbf{92.47$\pm$0.51} &26.75$\pm$27.24 &36.72 &421.79$\pm$31.41 &0.147$\pm$0.004 &5.48$\pm$1.55 \\
						&TWD-SFNN-R &91.42$\pm$0.21 &\textbf{88.41$\pm$0.31} &\textbf{72.84} &\textbf{20.10$\pm$0.15} &\textbf{0.004$\pm$0.000} &\underline{2.00$\pm$0.00} \\
						&STWD-SFNN-NK &85.45$\pm$0.21 &85.35$\pm$0.25 &\underline{70.63} &71.08$\pm$1.07 &0.018$\pm$0.004 &\textbf{1.00$\pm$0.00} \\
						&STWD-SFNN &88.41$\pm$0.32 &\underline{87.22$\pm$0.37} &\underline{70.63} &76.04$\pm$2.53 &0.029$\pm$0.007 &\underline{2.00$\pm$0.00} \\ \hline	
					\end{tabular}}
				}
				\label{tab_gs_pso}
			\end{table*}
     
%Similar phenomena can be found in other datasets			

		1. The accuracy, weighted-f1, ROC, and AUC of STWD-SFNN are better than those of GS-SFNN, PSO-SFNN, TWD-SFNN-R, and STWD-SFNN-NK. For example, on the ROE dataset of Table \ref{tab_gs_pso}, the accuracy, weighted-f1, and AUC of STWD-SFNN are 84.35\%, 85.00\%, and 80.19\%, respectively, which are higher than those of TWD-SFNN-R with best performance in the competitive models, and the values are 84.19\%, 84.85\%, and 80.19\%, respectively. Meanwhile, Fig. \ref{fig_gs_pso} shows that the comparison of ROC curves of TWD-SFNN-R and STWD-SFNN overlap. Similar phenomena can be found in other datasets. The reasons are as follows. 
  
\begin{enumerate}[(1)] 
\item        Compared with GS-SFNN and PSO-SFNN, STWD-SFNN adopts the idea of divide-and-conquer and focuses on learning the features of the difficult-to-classify instances in BND. However, GS-SFNN and PSO-SFNN do not perform targeted processing on the difficult-to-classify instances, but repeatedly learn the features of all instances.	
			
\item        Compared with TWD-SFNN-R, STWD-SFNN sets dynamic thresholds and utilizes the sequential property of these thresholds, thereby gradually relaxing the conditions of instances partitioned at different granularity levels. However, TWD-SFNN-R adopts the static threshold in each turn, which is difficult to achieve dynamic adjustment of instance learning.  
			
\item         Compared with STWD-SFNN-NK, STWD-SFNN utilizes the discretization technique to handle the data of instances partitioned into the BND region, thereby improving the learning ability of the model. Therefore, the generalization ability of STWD-SFNN is better than that of the competitive models.		

\end{enumerate}

2. The number of hidden layer nodes of STWD-SFNN is no more than GS-SFNN, PSO-SFNN, and TWD-SFNN-R, but more than STWD-SFNN-NK in most cases. For example, on the SSMCR dataset in Table \ref{tab_gs_pso}, STWD-SFNN has the same number of nodes as TWD-SFNN-R, with a value of two. STWD-SFNN has less nodes than GS-SFNN and PSO-SFNN, with  values of 3.00 and 5.48, respectively, while STWD-SFNN has more nodes than STWD-SFNN-NK, with a value of 1.00. Similar phenomena can be found in other datasets. The reasons are as follows.

\begin{enumerate}[(1)] 
\item        Compared with GS-SFNN and PSO-SFNN, STWD-SFNN adopts the idea of sequential three-way decisions, and determines whether to increase the number of hidden layer nodes according to whether there are instances in BND. However, GS-SFNN and PSO-SFNN use the difference in accuracy between adjacent turns. If the accuracy of this turn is lower than that of the previous turn, the number of hidden layer nodes will continue to increase until it reaches a given maximum number of hidden layer nodes.
			
\item        Compared with TWD-SFNN-R, STWD-SFNN constructs granularity layers and sequential thresholds so that it can pay more attention to instances that are difficult to classify, and gradually partition instances that are difficult to classify from the BND region. In addition, STWD-SFNN is a generalization of TWD-SFNN-R, which can ensure that STWD-SFNN has the same advantages as TWD-SFNN-R in network topology.
			
\item          Compared with STWD-SFNN-NK, STWD-SFNN can repeatedly learn the features of difficult-to-classify instances, thereby expanding the network topology. Therefore, the network topology of STWD-SFNN is superior to the competitive models in most cases.		
\end{enumerate}		
		
3. The training time of STWD-SFNN is less than GS-SFNN and PSO-SFNN, but more than TWD-SFNN-R and STWD-SFNN-NK on some datasets. For example, on the QSAR dataset in Table \ref{tab_gs_pso}, the training time of STWD-SFNN is, less than GS-SFNN, PSO-SFNN, TWD-SFNN-R, and STWD-SFNN-NK, with values of 12.43s, 40.13s, 190.05s, 13.61s, and 14.62s, respectively. Similar phenomena can be found in other datasets. The reasons are as follows.
		\begin{enumerate}[(1)] 
			\item        Compared with GS-SFNN and PSO-SFNN, STWD-SFNN only has a learning process for all instances in the first turn, and the time cost required for the model gradually decreases as the size of BND decreases. However, GS-SFNN and PSO-SFNN need to spend more time on each turn of learning for all instances.
			
			\item        Compared with TWD-SFNN-R which may repeatedly handle the same difficult-to-classify instances in each turn, the sequential thresholds of STWD-SFNN can reduce the condition of partitioned instances in each turn. STWD-SFNN decomposes the more easily classified instances in BND at each granularity level, so as to gradually realize the partition of all instances.
			
			\item         Compared with STWD-SFNN-NK, STWD-SFNN adopts $ k $-means++ discretization technology to increase the learning time of the model. Therefore, STWD-SFNN performs better than the competitive models on most cases.		
		\end{enumerate}

\subsubsection{Comparison with other competitive models} \label{exp_svc}
To measure the efficiency of STWD-SFNN, we selected SVC, RF, and KNN as the competitive models.  To guarantee the repeatability of the experimental results, SVC, RF, and KNN adopted the same random number seed in 10-fold cross-validation, that is, we fixed the random number seed in 10-fold cross-validation by utilizing rng(0). Fig. \ref{fig_svc} shows the comparison of  ROC curves of the classification models and STWD-SFNN. Table \ref{tab_svc} compares the evaluation criteria of each model. The results give rise to the following observations.
		
		\begin{figure*}[hp]
			\centering
			\subfigure[ONP]{
				\includegraphics[width=2.1in]{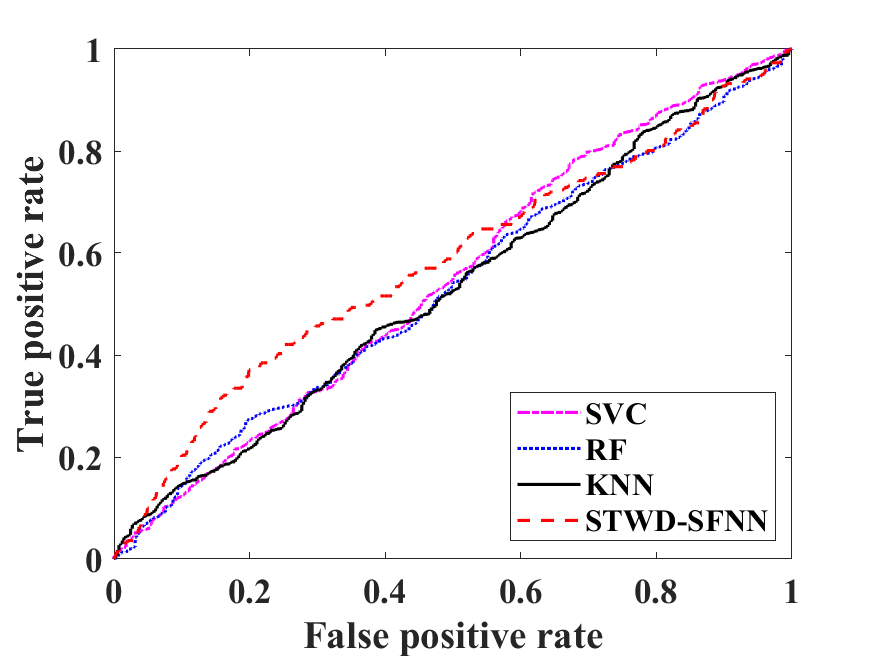}}
			\hfill
			\centering
			\subfigure[QSAR]{
				\includegraphics[width=2.1in]{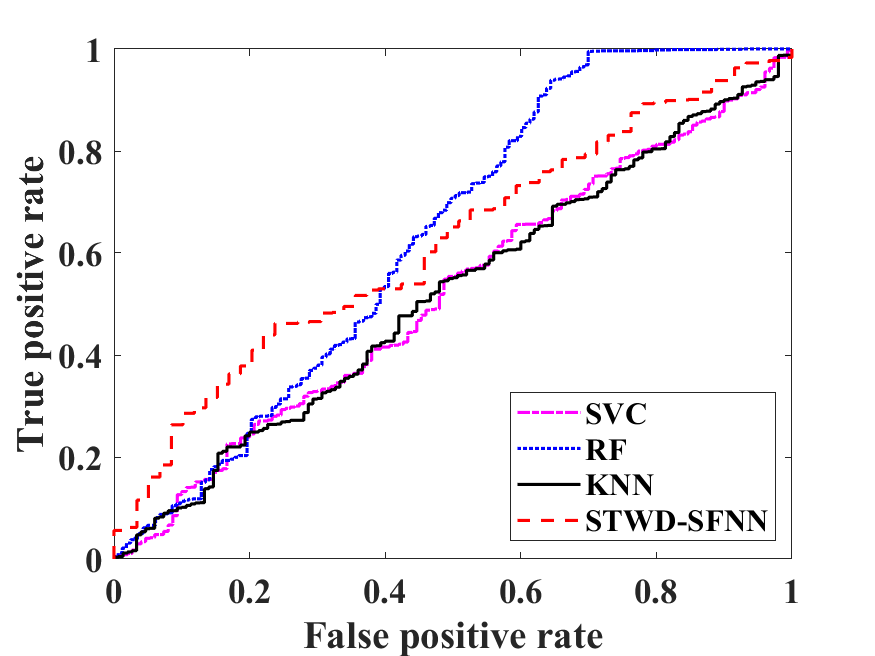}}
			\hfill
			\centering
			\subfigure[OSP]{
				\includegraphics[width=2.1in]{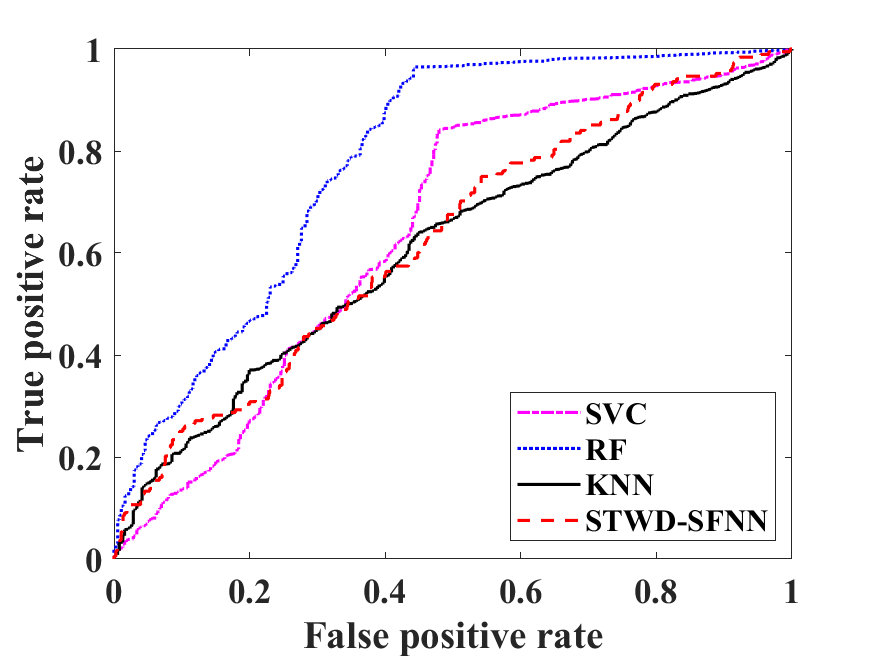}}
			\hfill
			\centering
			\subfigure[EGSS]{
				\includegraphics[width=2.1in]{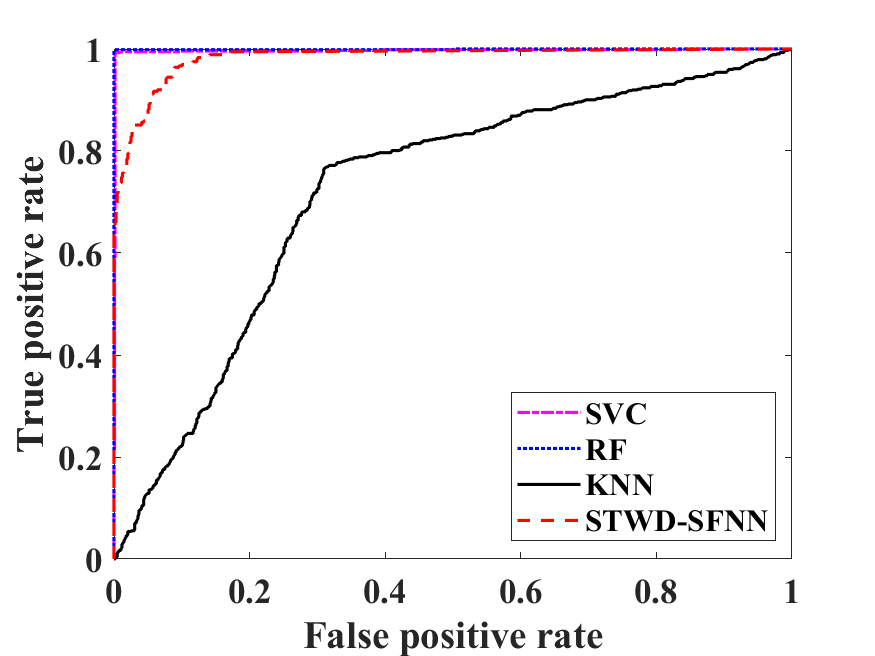}}
			\hfill
			\centering
			\subfigure[SE]{
				\includegraphics[width=2.1in]{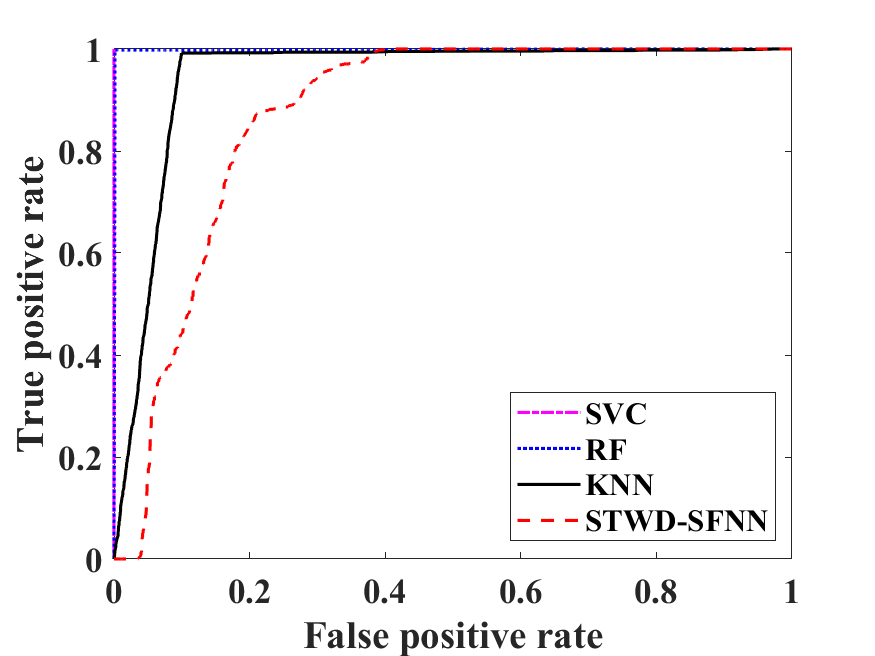}}
			\hfill
			\centering
			\subfigure[HTRU]{
				\includegraphics[width=2.1in]{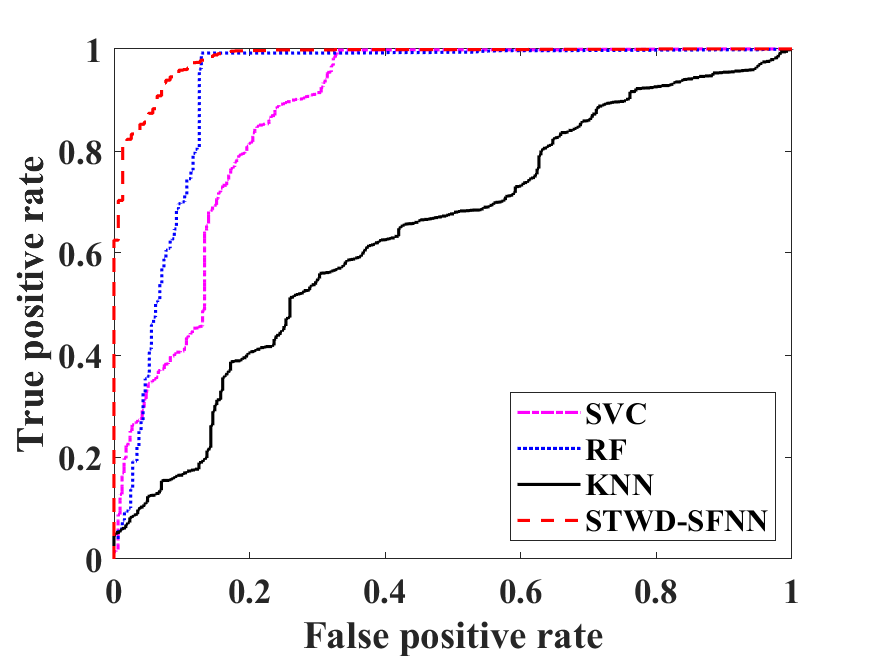}}
			\hfill
			\centering
			\subfigure[DCC]{
				\includegraphics[width=2.1in]{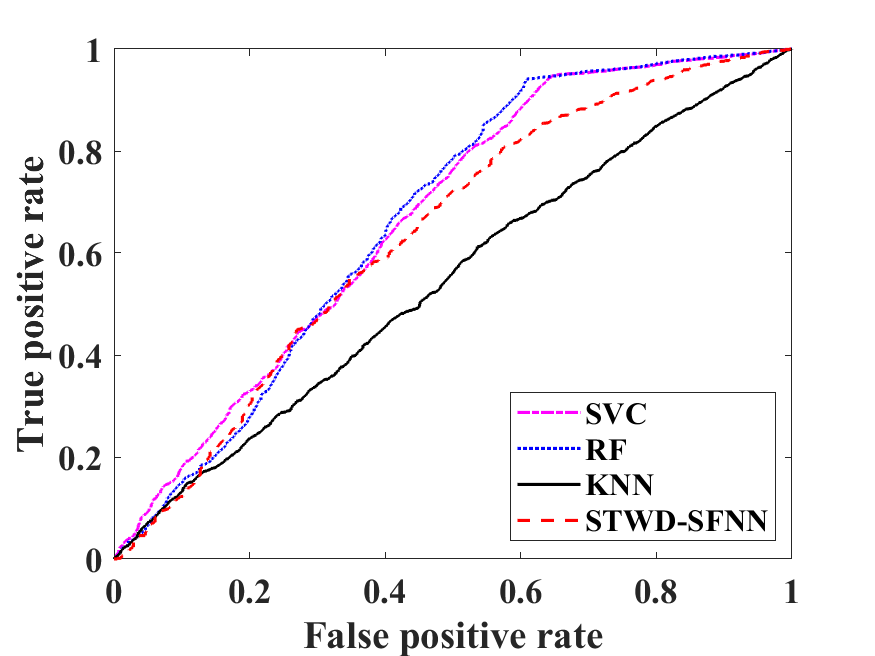}}
			\hfill
			\centering
			\subfigure[ESR]{
				\includegraphics[width=2.1in]{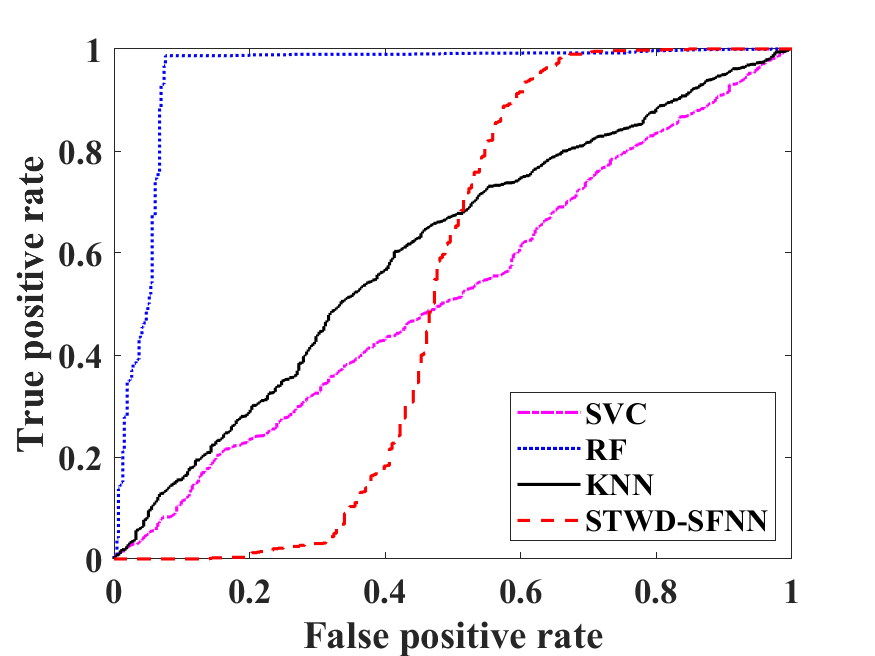}}	
			\hfill
			\centering
			\subfigure[BM]{
				\includegraphics[width=2.1in]{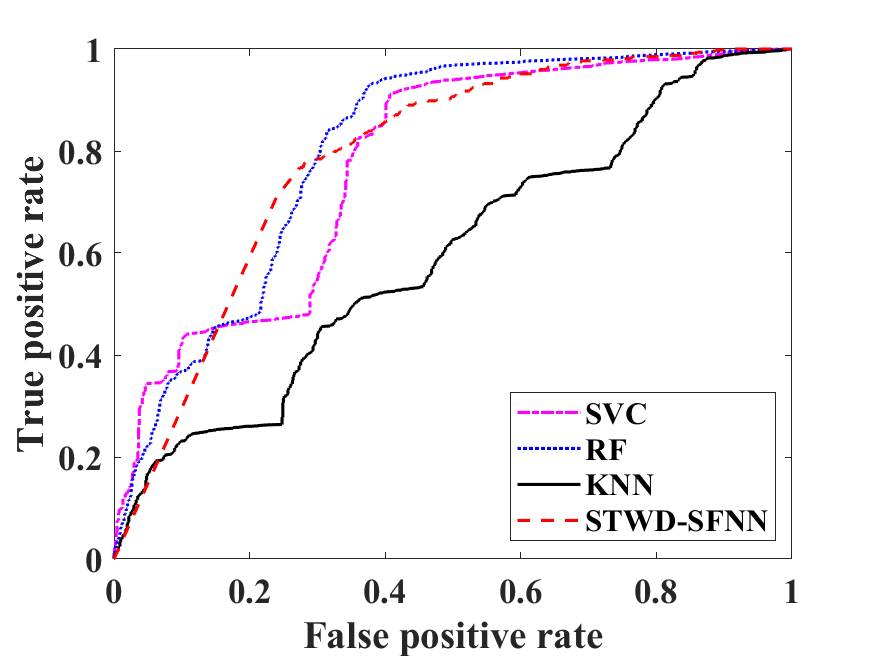}}
			\hfill
			\centering
			\subfigure[PCB]{
				\includegraphics[width=2.1in]{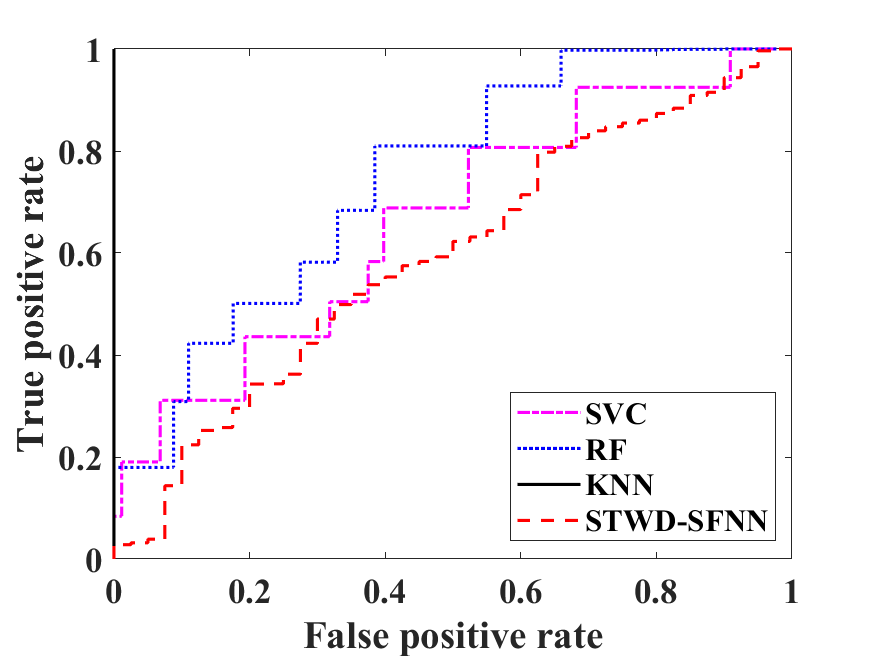}}
			\hfill	
			\centering
			\subfigure[SB]{
				\includegraphics[width=2.1in]{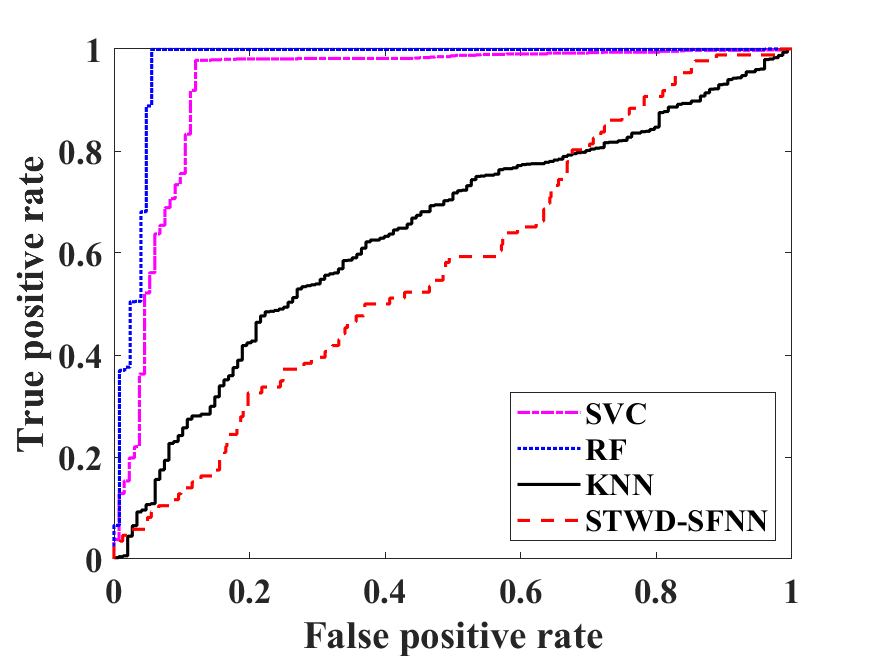}}
			\hfill	
			\centering
			\subfigure[EOL]{
				\includegraphics[width=2.1in]{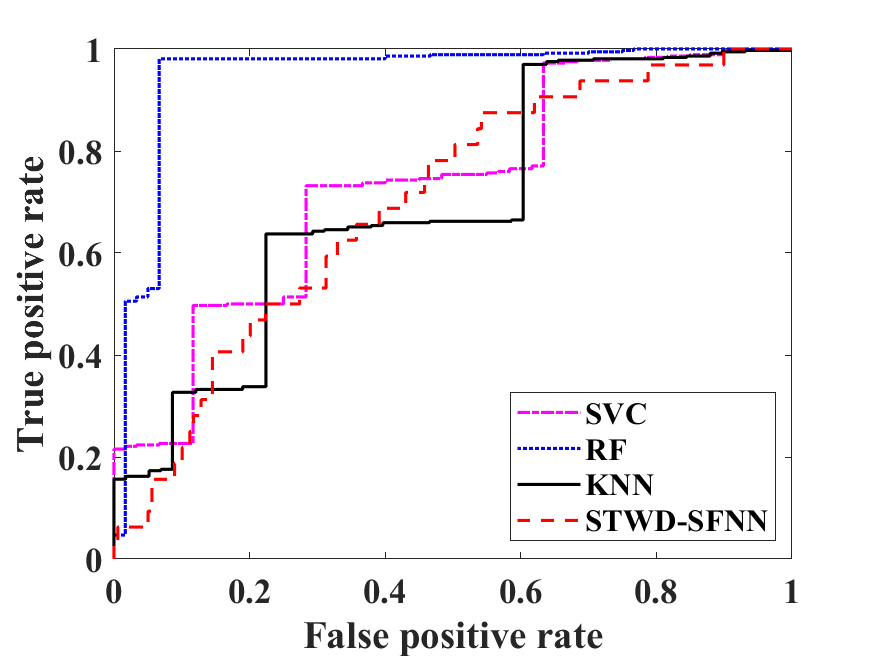}}
			\hfill	
			\centering	
			\subfigure[OD]{
				\includegraphics[width=2.1in]{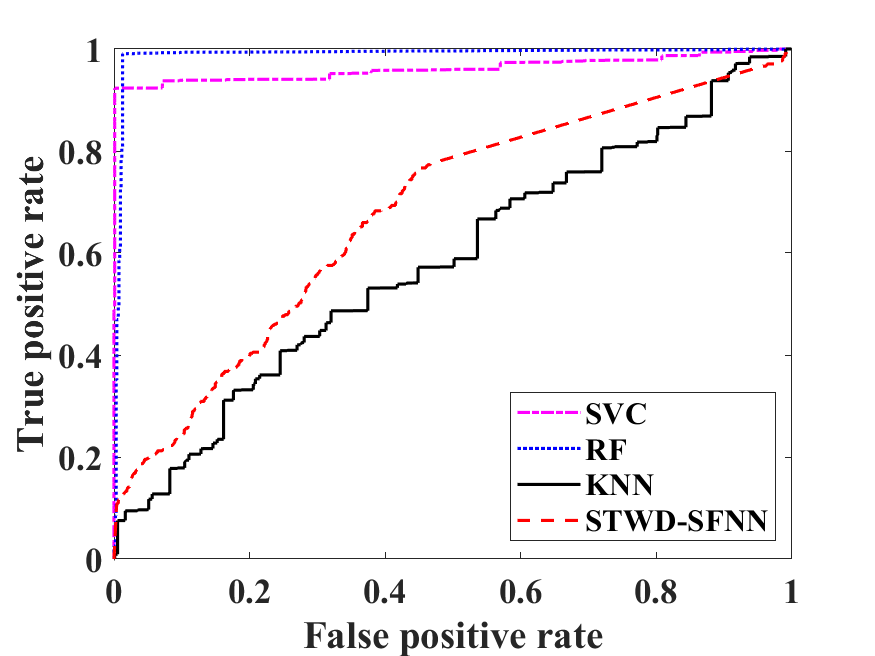}}
			\hfill	
			\centering
			\subfigure[ROE]{
				\includegraphics[width=2.1in]{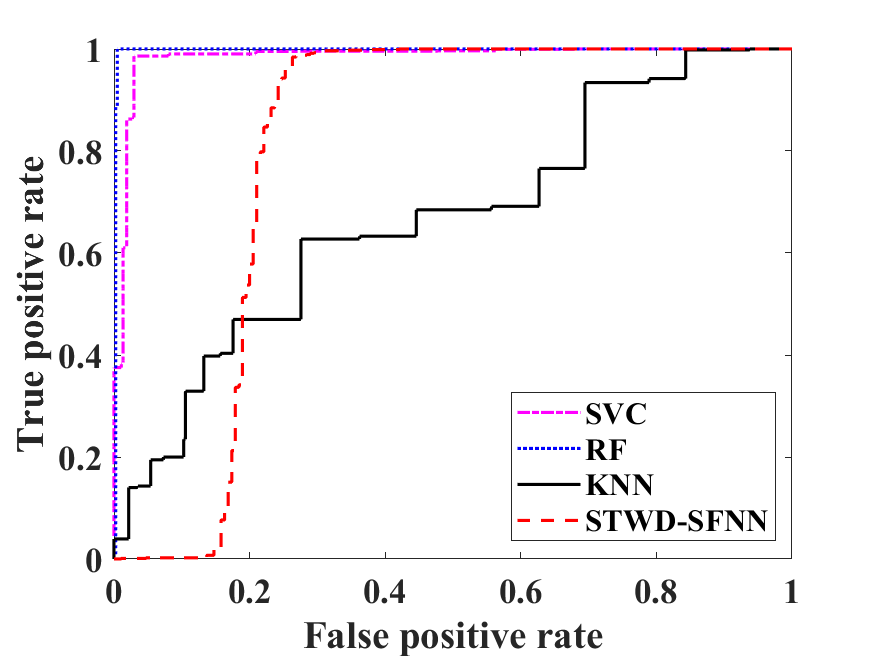}}
			\hfill	
			\centering
			\subfigure[SSMCR]{
				\includegraphics[width=2.1in]{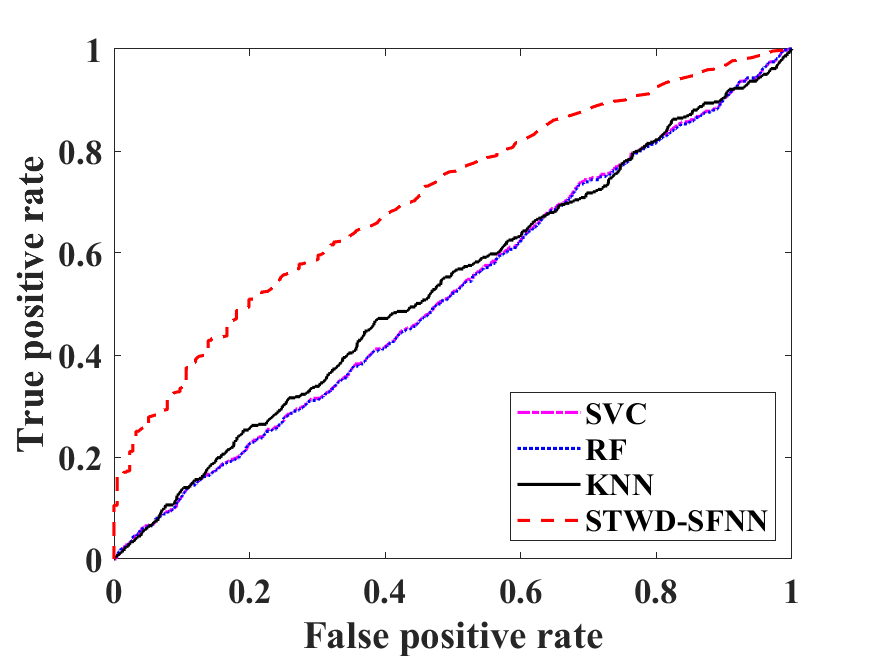}}
			\hfill	
			\small
			\captionsetup{font={small}}
			\caption{Comparison of ROC curves of other classification models and STWD-SFNN }
			\label{fig_svc}
		\end{figure*}
	
		\begin{table*}[hp]
			\setlength{\abovecaptionskip}{0cm}    % 段前
			\setlength{\belowcaptionskip}{0cm}    % 段后
			\small
			\captionsetup{font={small}}
			\newcommand{\tabincell}[2]{\begin{tabular}{@{}#1@{}}#2\end{tabular}}
			\caption{ Comparison of other classification models and STWD-SFNN }
			\centering
			%            	\resizebox{\textwidth}{60mm}{
			\setlength{\tabcolsep}{1mm}{
				\begin{tabular}{ccccccc}
					\hline
					Dataset   &Model  &\tabincell{c}{Accuracy  (\%)}   &\tabincell{c}{Weighted- f1(\%)}  &\tabincell{c}{AUC  (\%)}   &\tabincell{c}{Training time(s)}  &\tabincell{c}{Test time  (s)} \\ \hline
						
					\multirow{4}{*}{ONP}
					&SVC &93.29$\pm$0.56 &\textbf{91.37$\pm$0.30} &54.61 &1284$\pm$117 &1.015$\pm$0.129 \\
					&RF &94.40$\pm$0.43 &82.53$\pm$18.34 &52.96 &124.20$\pm$1.80 &0.059$\pm$0.008 \\
					&KNN &94.39$\pm$0.22 &45.88$\pm$30.59 &53.01 &- &88.43$\pm$17.95 \\
					&STWD-SFNN &\textbf{94.41$\pm$0.22} &\underline{85.87$\pm$0.26} &\textbf{57.87} &\textbf{15.76$\pm$0.76} &\textbf{0.003$\pm$0.000} \\ \hline
						
					\multirow{4}{*}{QSAR}
					&SVC &91.63$\pm$0.46 &90.06$\pm$0.78 &51.78 &1.77$\pm$0.08 &0.093$\pm$0.008 \\
					&RF &\textbf{93.75$\pm$0.72} &\textbf{92.56$\pm$0.94} &\textbf{62.91} &1721.7$\pm$68.20 &0.092$\pm$0.013 \\
					&KNN &93.51$\pm$0.72 &92.26$\pm$0.94 &51.60 &- &4.728$\pm$0.343 \\
					&STWD-SFNN &78.26$\pm$1.11 &81.21$\pm$1.13 &\underline{61.61} &\textbf{12.43$\pm$0.15} &\textbf{0.008$\pm$0.001} \\ \hline
						
					\multirow{4}{*}{OSP}
					&SVC &70.31$\pm$18.50 &63.23$\pm$19.50 &64.46 &349.12$\pm$56.78 &0.048$\pm$0.008 \\
					&RF &\textbf{90.47$\pm$0.49} &\textbf{89.97$\pm$0.51} &\textbf{78.16} &6.31$\pm$0.19 &0.025$\pm$0.001 \\
					&KNN &84.10$\pm$0.73 &78.34$\pm$0.98 &61.05 &- &12.40$\pm$2.24 \\
					&STWD-SFNN &68.43$\pm$0.66 &70.01$\pm$0.81 &62.52 &\textbf{5.55$\pm$0.51} &\textbf{0.002$\pm$0.000} \\ \hline
						
					\multirow{4}{*}{EGSS}
					&SVC &90.17$\pm$0.69 &89.96$\pm$0.73 &99.60 &\textbf{3.15$\pm$0.24} &0.070$\pm$0.009 \\
					&RF &\textbf{99.98$\pm$0.00} &\textbf{99.98$\pm$0.00} &\textbf{99.93} &6.07$\pm$0.44 &0.030$\pm$0.020 \\
					&KNN &99.68$\pm$0.17 &99.68$\pm$0.17 &72.12 &- &81.13$\pm$7.13 \\
					&STWD-SFNN &85.63$\pm$0.70 &85.91$\pm$0.67 &98.30 &7.09$\pm$0.37 &\textbf{0.003$\pm$0.000} \\ \hline
						
					\multirow{4}{*}{SE}
					%&SVC &84.77$\pm$0.20 &78.94$\pm$0.01 &86.75 &\textbf{24.22$\pm$1.14} &1.981$\pm$0.016 \\
					&SVC &84.77$\pm$0.20 &- &- &\textbf{24.22$\pm$1.14} &1.981$\pm$0.016 \\
					&RF &\textbf{99.92$\pm$0.03} &\textbf{99.92$\pm$0.03} &\textbf{99.83} &49.26$\pm$3.10 &0.376$\pm$0.020 \\
					&KNN &99.64$\pm$0.05 &99.64$\pm$0.05 &94.60 &- &10324$\pm$2993 \\
					&STWD-SFNN &80.89$\pm$0.29 &78.31$\pm$0.36 &87.14 &85.20$\pm$2.33 &\textbf{0.012$\pm$0.000} \\ \hline
						
					\multirow{4}{*}{HTRU}
					&SVC &70.15$\pm$0.01 &79.18$\pm$0.01 &87.64 &492.75$\pm$25.27 &0.073$\pm$0.076 \\
					&RF &97.96$\pm$0.09 &\textbf{97.92$\pm$0.10} &92.57 &\textbf{2.52$\pm$0.27} &0.012$\pm$0.005 \\
					&KNN &\textbf{97.97$\pm$0.12} &\textbf{97.92$\pm$0.13} &64.26 &- &92.01$\pm$10.42 \\
					&STWD-SFNN &91.80$\pm$0.30 &\underline{92.72$\pm$0.23} &\textbf{98.34} &\underline{10.43$\pm$0.61} &\textbf{0.004$\pm$0.000} \\ \hline
						
					\multirow{4}{*}{DCC}
					&SVC &72.76$\pm$1.32 &67.00$\pm$0.86 &\textbf{66.67} &1041.81$\pm$27.23 &0.248$\pm$0.064 \\
					&RF &\textbf{81.84$\pm$0.44} &\textbf{79.85$\pm$0.60} &66.65 &131.16$\pm$34.93 &0.339$\pm$0.022 \\
					&KNN &81.36$\pm$0.42 &78.85$\pm$0.55 &54.14 &- &44.49$\pm$4.43 \\
					&STWD-SFNN &78.13$\pm$0.40 &69.70$\pm$0.55 &63.59 &\textbf{7.89$\pm$0.32} &\textbf{0.002$\pm$0.000} \\ \hline
						
					\multirow{4}{*}{ESR}
					&SVC &80.00$\pm$0.96 &80.35$\pm$0.01 &51.99 &12.99$\pm$0.51 &1.327$\pm$0.051 \\
					&RF &\textbf{97.58$\pm$0.18} &\textbf{97.58$\pm$0.18} &\textbf{95.12} &610.17$\pm$138.52 &0.113$\pm$0.010 \\
					&KNN &84.43$\pm$0.97 &80.15$\pm$1.33 &60.18 &- &73.41$\pm$7.48 \\
					&STWD-SFNN &84.35$\pm$0.47 &81.55$\pm$0.55 &54.25 &\textbf{4.29$\pm$0.06} &\textbf{0.003$\pm$0.000} \\ \hline
						
					\multirow{4}{*}{BM}
					&SVC &65.62$\pm$24.03 &61.57$\pm$25.08 &77.27 &908.79$\pm$149.27 &0.361$\pm$0.022\\
					&RF &\textbf{91.67$\pm$0.19} &\textbf{91.16$\pm$0.24} &\textbf{80.41} &68.11$\pm$3.32 &0.173$\pm$0.004 \\
					&KNN &89.41$\pm$0.32 &86.98$\pm$0.42 &59.52 &- &234.01$\pm$35.35 \\
					&STWD-SFNN &87.56$\pm$0.34 &82.85$\pm$0.48 &\underline{78.80} &\textbf{15.19$\pm$0.36} &\textbf{0.003$\pm$0.001} \\ \hline
						
					\multirow{4}{*}{PCB}
					&SVC &97.77$\pm$0.32 &96.80$\pm$0.45 &66.26 &200.21$\pm$69.92 &0.138$\pm$0.011 \\
					&RF &\textbf{98.09$\pm$0.27} &\textbf{97.56$\pm$0.38} &\textbf{74.52} &222.54$\pm$61.80 &0.049$\pm$0.004 \\
					%&KNN &97.84$\pm$0.32 &96.90$\pm$0.01 &61.27 &- &29.27$\pm$5.60 \\
					&KNN &97.84$\pm$0.32 &- &- &- &29.27$\pm$5.60 \\
					&STWD-SFNN &97.31$\pm$0.49 &96.54$\pm$0.52 &59.23 &\textbf{5.44$\pm$0.06} &\textbf{0.002$\pm$0.000} \\ \hline
						
					\multirow{4}{*}{SB}
					&SVC &98.24$\pm$0.30 &98.25$\pm$0.30 &92.29 &11.36$\pm$0.46 &0.025$\pm$0.034\\
					&RF &\textbf{99.81$\pm$0.16} &\textbf{99.81$\pm$0.16} &\textbf{97.06} &\textbf{0.24$\pm$0.01} &0.007$\pm$0.005 \\
					&KNN &93.43$\pm$0.39 &92.88$\pm$0.48 &63.47 &- &1.237$\pm$0.067 \\
					&STWD-SFNN &65.80$\pm$0.78 &71.61$\pm$0.90 &57.62 &\underline{3.48$\pm$0.18} &\textbf{0.003$\pm$0.000} \\ \hline
						
					\multirow{4}{*}{EOL}
					&SVC &88.58$\pm$1.00 &87.68$\pm$1.21 &71.60 &4.07$\pm$0.55 &0.032$\pm$0.045 \\
					&RF &\textbf{97.44$\pm$0.45} &\textbf{97.46$\pm$0.44} &\textbf{94.99} &1.41$\pm$0.25 &0.006$\pm$0.000 \\
					&KNN &87.02$\pm$1.19 &81.85$\pm$1.65 &69.84 &- &0.086$\pm$0.006 \\
					&STWD-SFNN &83.66$\pm$1.95 &81.53$\pm$0.01 &69.78 &\textbf{0.75$\pm$0.07} &\textbf{0.001$\pm$0.000} \\ \hline
						
					\multirow{4}{*}{OD}
					&SVC &79.37$\pm$11.88 &77.91$\pm$12.17 &96.05 &452.67$\pm$77.46 &0.037$\pm$0.027 \\
					&RF &\textbf{99.27$\pm$0.16} &\textbf{99.27$\pm$0.16} &\textbf{98.93} &7.36$\pm$0.34 &0.037$\pm$0.004 \\
					&KNN &79.07$\pm$0.43 &74.24$\pm$0.49 &58.29 &- &274.68$\pm$23.79 \\
					&STWD-SFNN &76.33$\pm$0.41 &66.71$\pm$0.57 &67.95 &\textbf{4.81$\pm$0.09} &\textbf{0.002$\pm$0.000} \\ \hline
						
					\multirow{4}{*}{ROE}
					&SVC &98.77$\pm$0.23 &98.75$\pm$0.23 &98.19 &\textbf{0.55$\pm$0.07} &0.020$\pm$0.020 \\
					&RF &\textbf{99.97$\pm$0.03} &\textbf{99.97$\pm$0.03} &\textbf{99.73} &1.24$\pm$0.02 &0.005$\pm$0.000 \\
					&KNN &97.20$\pm$0.42 &97.15$\pm$0.43 &66.56 &- &94.56$\pm$6.56 \\
					&STWD-SFNN &84.35$\pm$0.98 &85.00$\pm$0.96 &80.19 &2.74$\pm$0.11 &\textbf{0.001$\pm$0.000} \\ \hline
						
					\multirow{4}{*}{SSMCR}
					&SVC &83.34$\pm$6.20 &48.28$\pm$26.37 &52.03 &1070$\pm$91 &0.290$\pm$0.081 \\
					&RF &\textbf{92.80$\pm$0.27} &53.55$\pm$29.15 &51.84 &\textbf{4.44$\pm$0.07} &0.059$\pm$0.004 \\
					&KNN &92.66$\pm$0.37 &35.99$\pm$29.06 &53.34 &- &241.58$\pm$48.77 \\
					&STWD-SFNN &88.41$\pm$0.32 &\textbf{87.22$\pm$0.37} &\textbf{70.63} &\underline{76.04$\pm$2.53} &\textbf{0.029$\pm$0.007} \\ \hline
				\end{tabular}}
				\begin{tablenotes}
					\footnotesize
%\begin{enumerate}[(1)] 
\item  We do not count the training time of KNN since the training time cost of KNN is zero. 
					
\item  The weighted-f1 and AUC are null in SVC on the SE dataset and KNN on the PCB dataset, since they divide all instances into positive classes.
%\end{enumerate}
\end{tablenotes}
				
		\label{tab_svc}
			\end{table*}

		1. The accuracy, weighted-f1, ROC, and AUC of STWD-SFNN are better than SVC, RF, and KNN on some datasets. For example, on the ONP dataset of Table \ref{tab_svc}, the accuracy, weighted-f1, and AUC of STWD-SFNN are 94.41\%, 85.87\%, and 57.87\%, respectively, while the best-performing competitive model is RF, with 94.40\%, 82.53\%, and 52.96\%, respectively. Meanwhile, as shown in Fig.5, the ROC curve of STWD-SFNN is at the left top of RF. Similar phenomena can be found in the other dataset. The reason is as follows. Compared with SVC, RF, and KNN, STWD-SFNN retains the ability of the neural network to capture the nonlinear relationship of data, and further enhances the performance of the model by adopting sequential three-way decisions. Therefore, in most cases, the generalization ability of STWD-SFNN is better than the competitive models.

		2. STWD-SFNN has less training time and test time than SVC, RF, and KNN on some datasets. For example, on the PCB data set of Table \ref{tab_svc}, the training time and test time of STWD-SFNN are 5.44s and 0.002s, respectively. The competitive models with the least training time and test time are SVC and RF, with values of 200.21s and 0.049s, respectively. Similar phenomena can be found in other datasets. The reason is as follows. Compared with SVC, RF, and KNN, STWD-SFNN can greatly promote the learning process by constructing granularity layers and sequential thresholds, which can improve the operation efficiency of the model. Therefore, in most cases, STWD-SFNN is more efficient than other competitive models.

       \subsubsection{Comparison with TWD-SFNN with 5-fold cross-validation} \label{exp_twd_5folds}
       From the above-mentioned experimental results, STWD-SFNN has relatively better accuracy, network topology, and operating efficiency. However, the ROC and AUC of TWD-SFNN and STWD-SFNN are not significantly different in Fig. \ref{fig_gs_pso} and Table \ref{tab_gs_pso}. To further illustrate the differences between TWD-SFNN and STWD-SFNN, we adopt a 5-fold cross-validation technique to conduct a further study on 15 datasets. Fig. \ref{fig_twd_5folds} shows the comparison of ROC curves of TWD-SFNN and STWD-SFNN on 5-fold cross-validation. Table \ref{tab_twd_5folds} reports the comparison of different evaluation criteria of those models. The results give rise to the following observations.
               
		\begin{figure*}[htp]
				\centering
				\subfigure[ONP]{
					\includegraphics[width=2.1in]{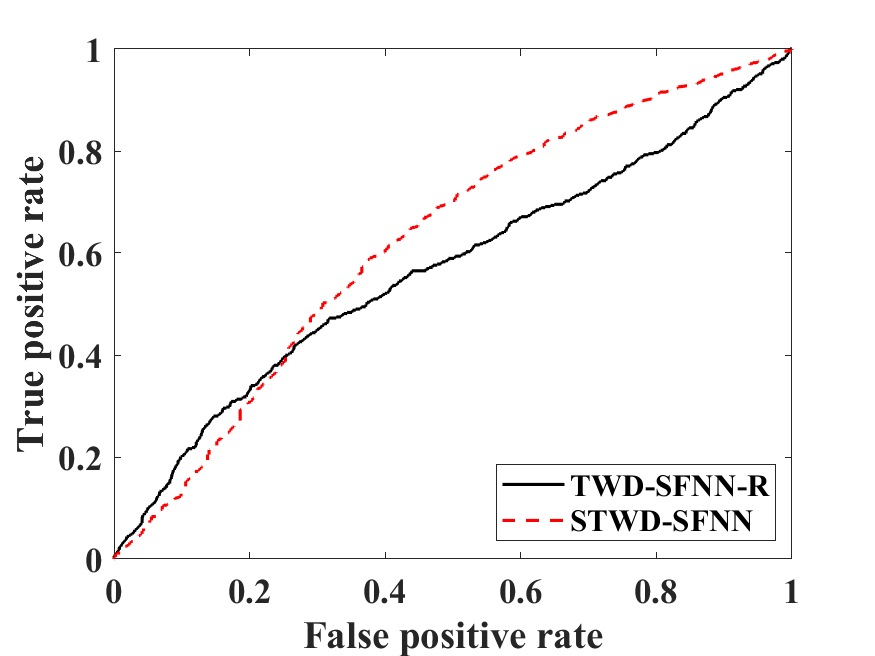}}
				\hfill
				\centering
				\subfigure[QSAR]{
					\includegraphics[width=2.1in]{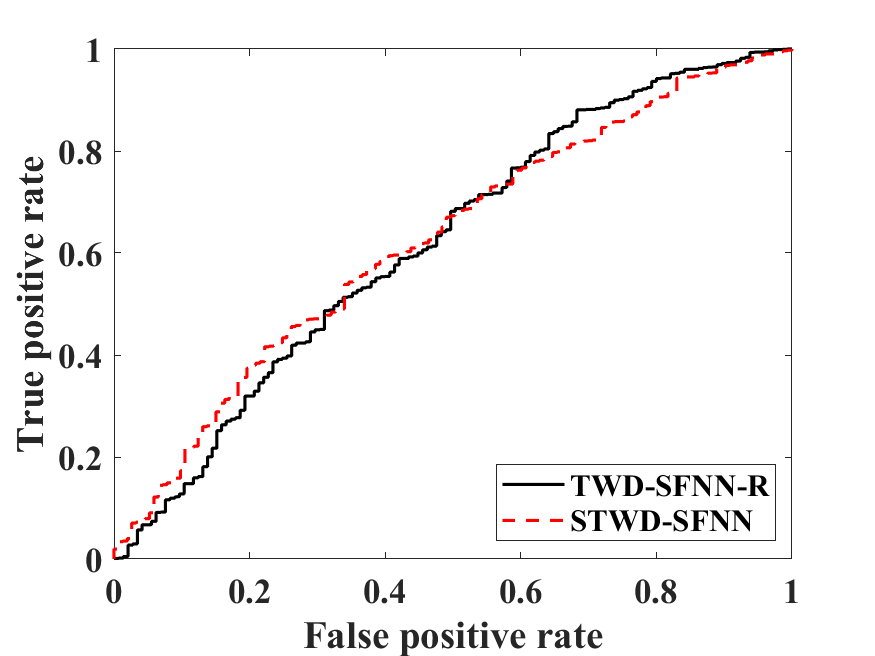}}
				\hfill
				\centering
				\subfigure[OSP]{
					\includegraphics[width=2.1in]{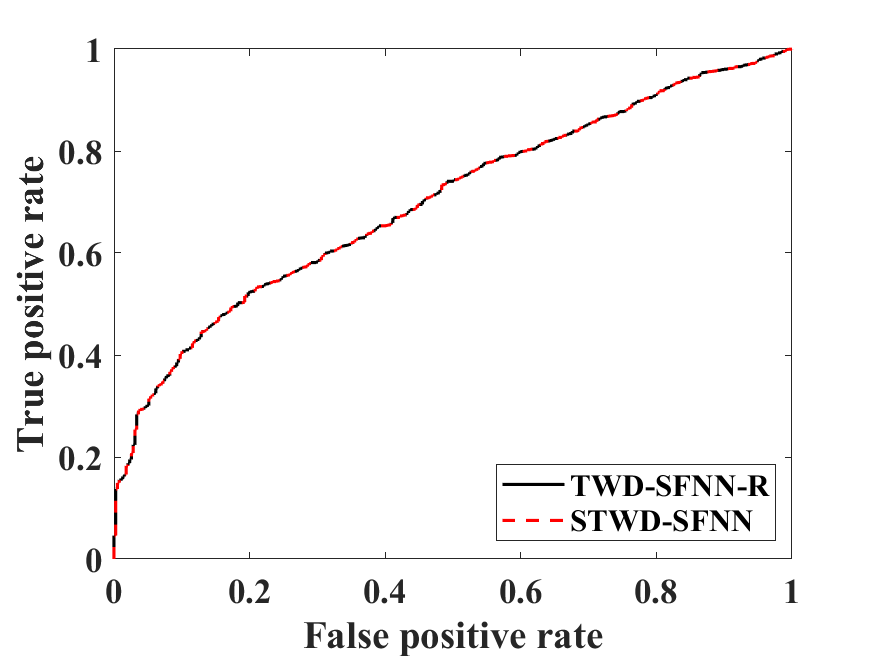}}
				\hfill
				\centering
				\subfigure[EGSS]{
					\includegraphics[width=2.1in]{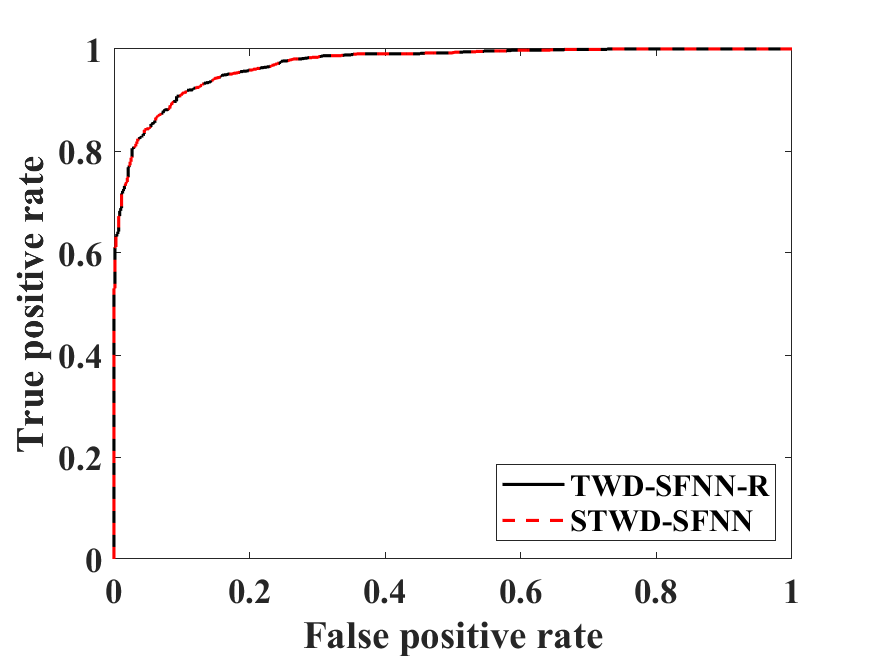}}
				\hfill
				\centering
				\subfigure[SE]{
					\includegraphics[width=2.1in]{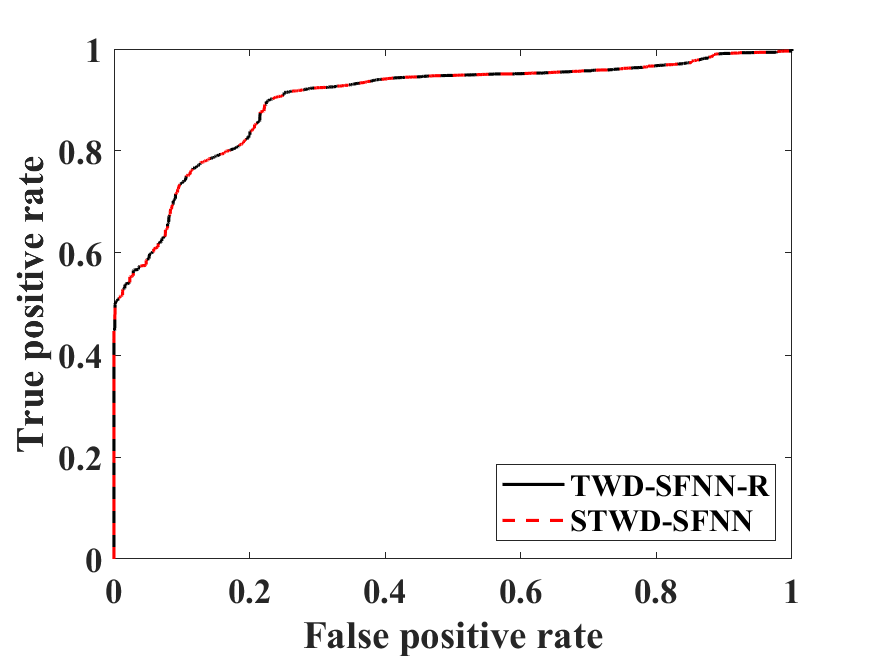}}
				\hfill
				\centering
				\subfigure[HTRU]{
					\includegraphics[width=2.1in]{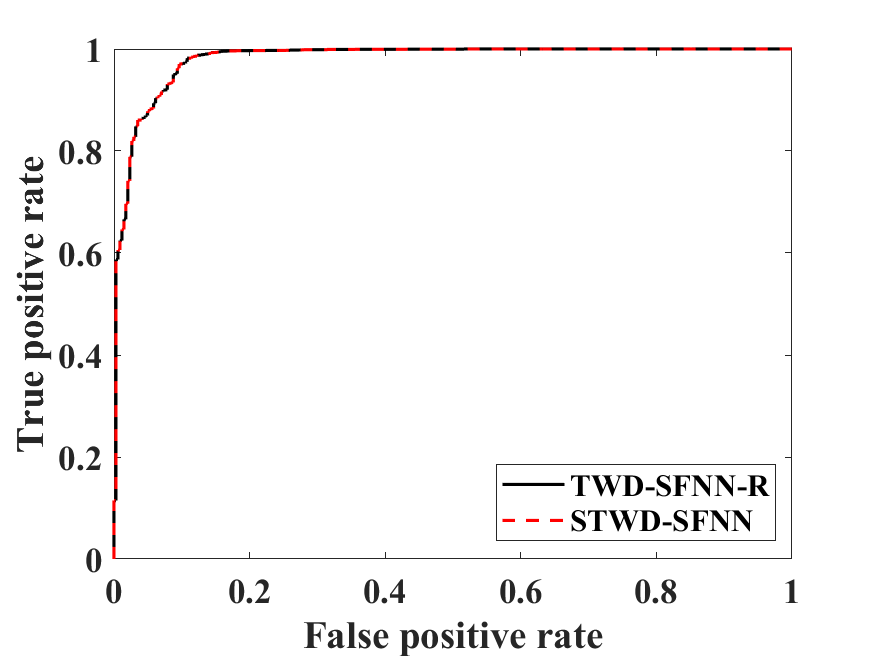}}
				\hfill
				\centering
				\subfigure[DCC]{
					\includegraphics[width=2.1in]{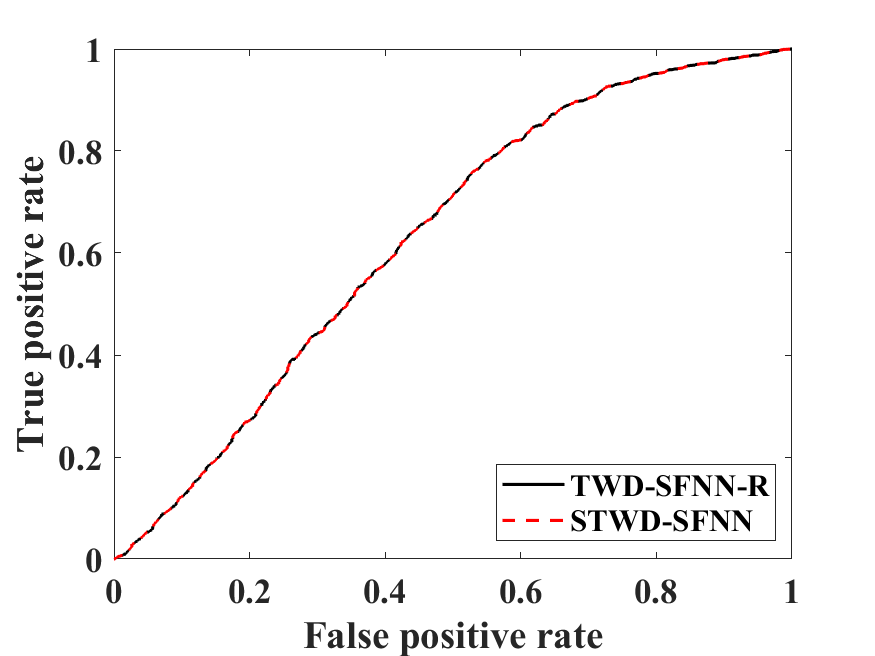}}
				\hfill
				\centering
				\subfigure[ESR]{
					\includegraphics[width=2.1in]{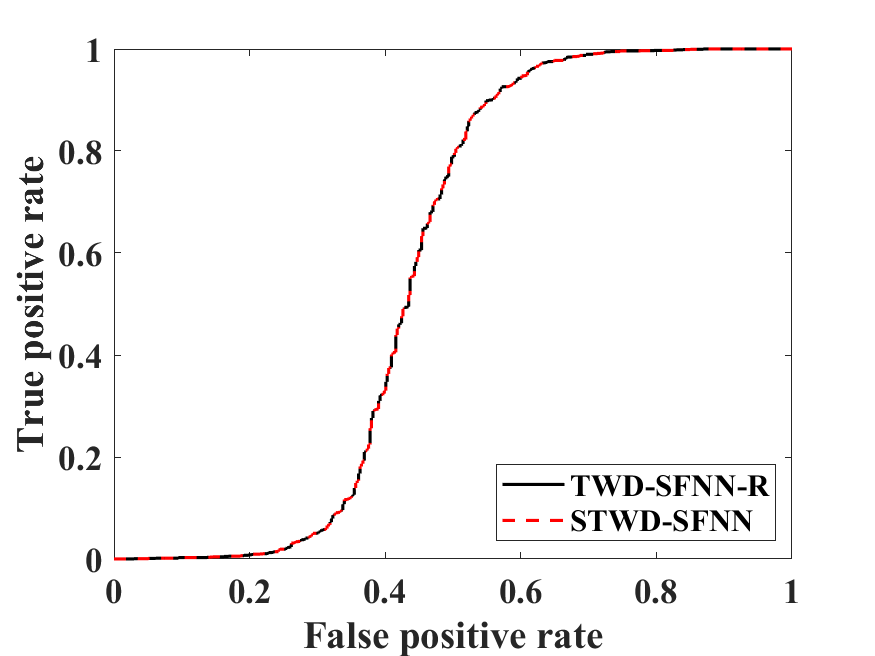}}	
				\hfill
				\centering
				\subfigure[BM]{
					\includegraphics[width=2.1in]{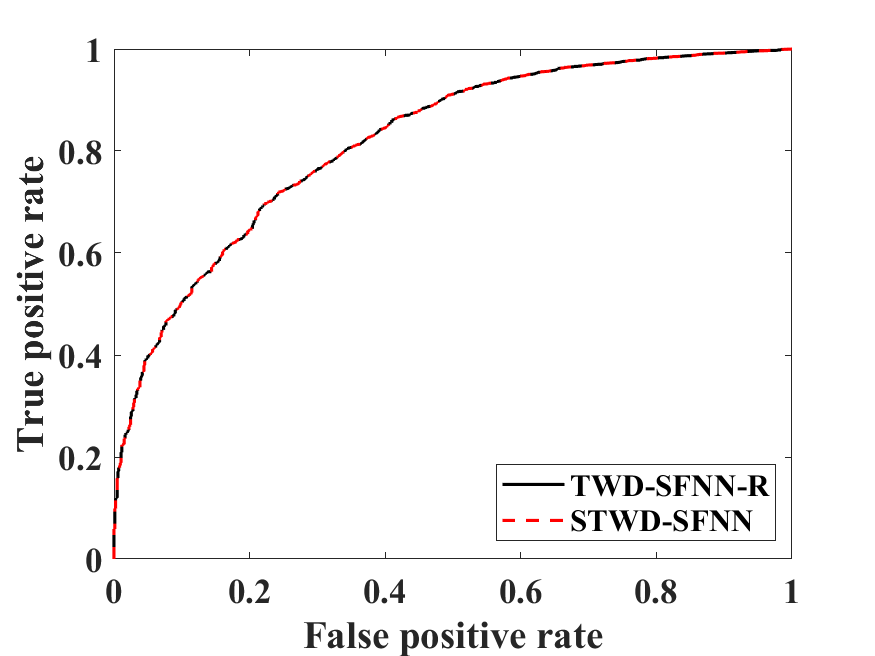}}
				\hfill
				\centering
				\subfigure[PCB]{
					\includegraphics[width=2.1in]{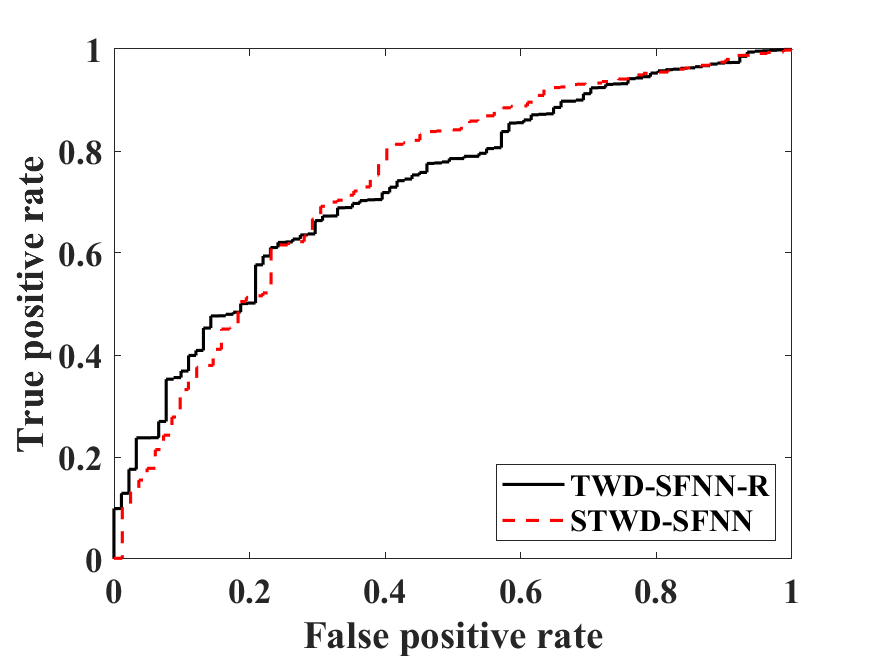}}
				\hfill	
				\centering
				\subfigure[SB]{
					\includegraphics[width=2.1in]{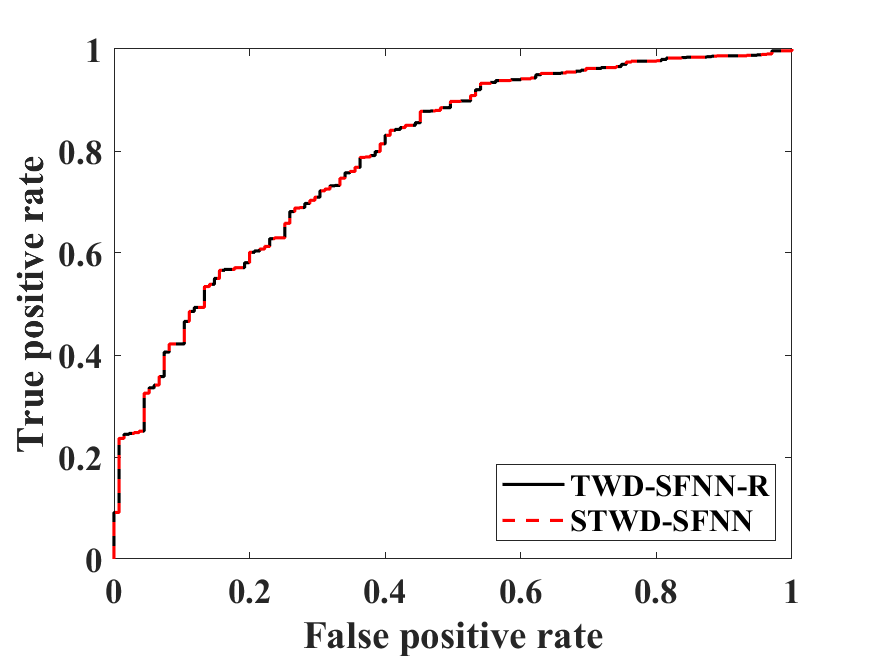}}
				\hfill	
				\centering
				\subfigure[EOL]{
					\includegraphics[width=2.1in]{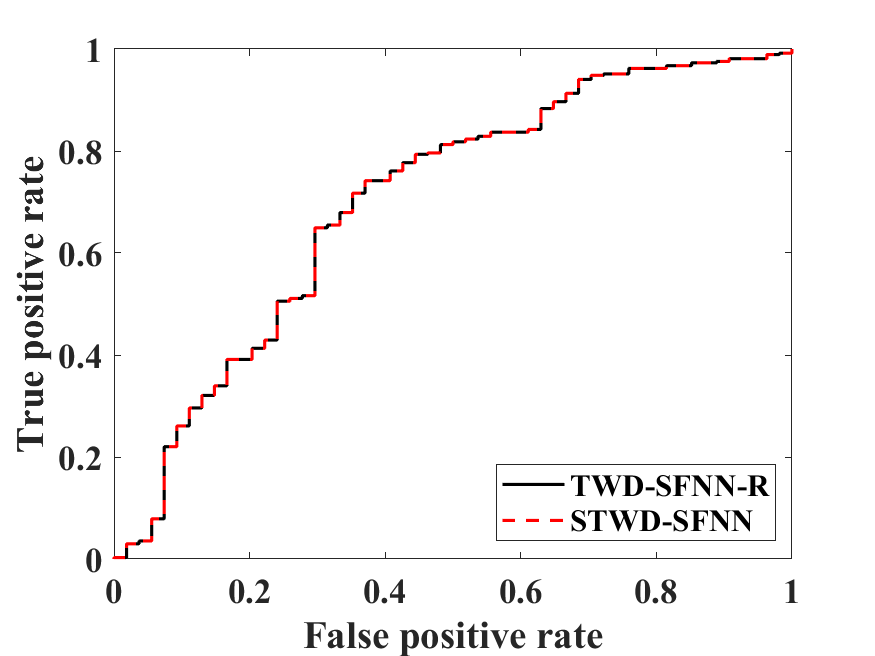}}
				\hfill	
				\centering	
				\subfigure[OD]{
					\includegraphics[width=2.1in]{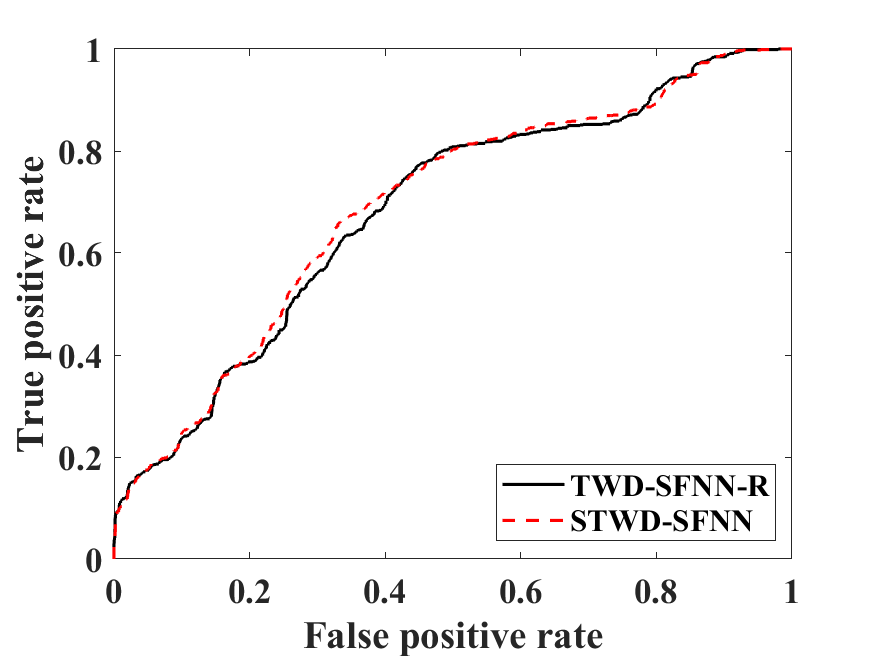}}
				\hfill	
				\centering
				\subfigure[ROE]{
					\includegraphics[width=2.1in]{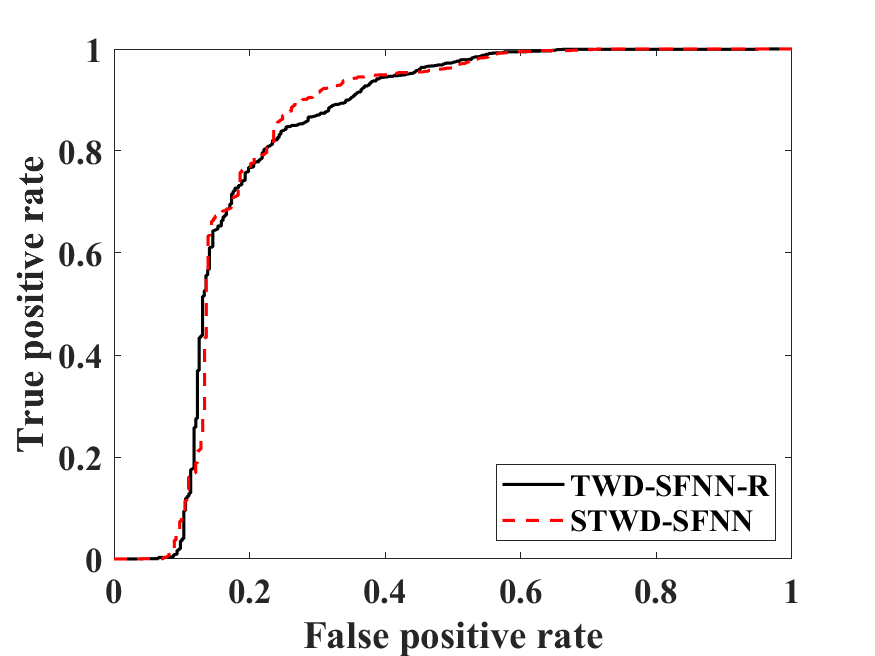}}
				\hfill	
				\centering
				\subfigure[SSMCR]{
					\includegraphics[width=2.1in]{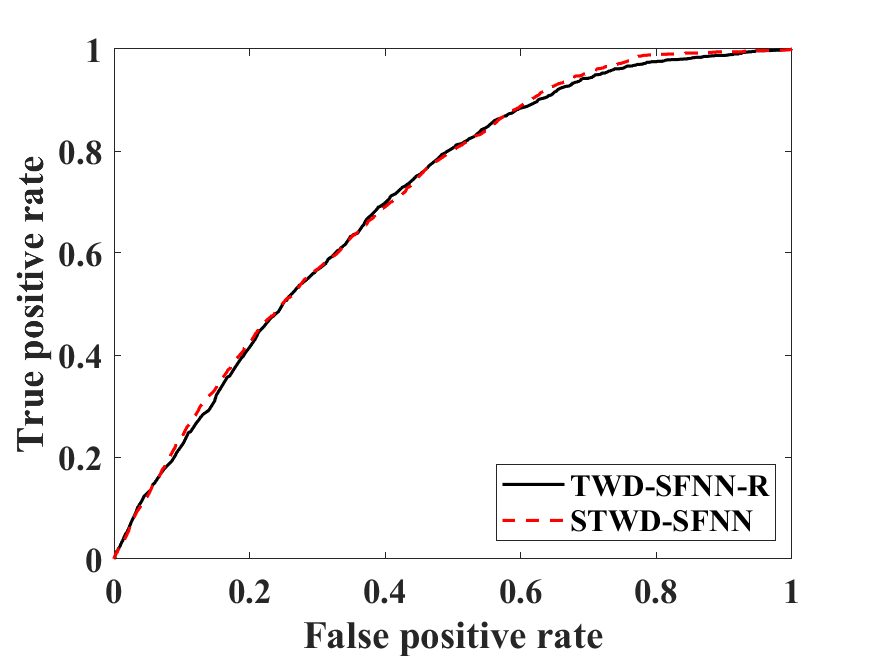}}
				\hfill	
				\small
				\captionsetup{font={small}}
				\caption{Comparison of  ROC curves of TWD-SFNN and STWD-SFNN by adopting 5-fold cross-validation }
				\label{fig_twd_5folds}
		\end{figure*}	        
	        
        \begin{table*}[hp]
                \setlength{\abovecaptionskip}{0cm}    % 段前
	        	\setlength{\belowcaptionskip}{0cm}    % 段后
	        	\small
	        	\captionsetup{font={small}}
	        	\newcommand{\tabincell}[2]{\begin{tabular}{@{}#1@{}}#2\end{tabular}}
	        	\caption{ Comparison of TWD-SFNN and STWD-SFNN by adopting 5-fold cross-validation }
	        	\centering
	        	%            	\resizebox{\textwidth}{60mm}{
	        	\setlength{\tabcolsep}{1mm}{
	        		\begin{tabular}{cccccccc}
	        			\hline
	        			Dataset   &Model  &\tabincell{c}{Accuracy \\ (\%)}   &\tabincell{c}{Weighted- \\f1 (\%)}  &\tabincell{c}{AUC \\ (\%)}   &\tabincell{c}{Training \\ time(s)}  &\tabincell{c}{Test time (s)} &\tabincell{c}{Nodes}\\ \hline
	        			
	        			\multirow{2}{*}{ONP}
	        			&TWD-SFNN-R &92.33$\pm$0.20 &90.94$\pm$0.23 &56.87 &35.29$\pm$0.29 &0.015$\pm$0.002 &3.00$\pm$0.00 \\
	        			&STWD-SFNN &\textbf{94.39$\pm$0.24} &\textbf{91.70$\pm$0.35} &\textbf{62.45} &\textbf{30.63$\pm$0.90} &\textbf{0.008$\pm$0.001} &\textbf{1.00$\pm$0.00} \\ \hline
	        			
	        			\multirow{2}{*}{QSAR}
	        			&TWD-SFNN-R &72.63$\pm$0.71 &78.29$\pm$0.60 &61.89 &\textbf{6.71$\pm$0.22} &0.010$\pm$0.002 &\textbf{2.00$\pm$0.00} \\
	        			&STWD-SFNN &\textbf{82.14$\pm$0.36} &\textbf{83.60$\pm$0.44} &\textbf{62.44} &9.13$\pm$0.66 &\textbf{0.009$\pm$0.001} &2.80$\pm$0.40 \\ \hline
	        			
	        			\multirow{2}{*}{OSP}
	        			&TWD-SFNN-R &\textbf{83.14$\pm$0.57} &\textbf{77.24$\pm$0.62} &\textbf{70.10} &2.29$\pm$0.05 &\textbf{0.003$\pm$0.001} &\textbf{2.00$\pm$0.00} \\
	        			&STWD-SFNN &\textbf{83.14$\pm$0.57} &\textbf{77.24$\pm$0.62} &\textbf{70.10} &\textbf{2.04$\pm$0.17} &\textbf{0.003$\pm$0.001} &\textbf{2.00$\pm$0.00} \\ \hline
	        			
	        			\multirow{2}{*}{EGSS}
	        			&TWD-SFNN-R &\textbf{87.67$\pm$0.60} &\textbf{87.91$\pm$0.56} &\textbf{97.05} &\textbf{4.23$\pm$1.14} &\textbf{0.004$\pm$0.000} &\textbf{1.00$\pm$0.00} \\
	        			&STWD-SFNN &\textbf{87.67$\pm$0.60} &\textbf{87.91$\pm$0.56} &\textbf{97.05} &4.76$\pm$1.16 &0.004$\pm$0.001 &\textbf{1.00$\pm$0.00} \\ \hline
	        			
	        			\multirow{2}{*}{SE}
	        			&TWD-SFNN-R &\textbf{86.00$\pm$0.31} &\textbf{80.58$\pm$0.46} &\textbf{90.01} &365.87$\pm$8.21 &\textbf{0.008$\pm$0.000} &\textbf{2.00$\pm$0.00} \\
	        			&STWD-SFNN &\textbf{86.00$\pm$0.31} &\textbf{80.58$\pm$0.46} &\textbf{90.01} &\textbf{296.15$\pm$3.47} &0.008$\pm$0.004 &\textbf{2.00$\pm$0.00} \\ \hline
	        			
	        			\multirow{2}{*}{HTRU}
	        			&TWD-SFNN-R &\textbf{91.84$\pm$0.40} &\textbf{92.76$\pm$0.29} &\textbf{98.10} &\textbf{9.02$\pm$0.34} &\textbf{0.007$\pm$0.002} &\textbf{2.00$\pm$0.00} \\
	        			&STWD-SFNN &\textbf{91.84$\pm$0.40} &\textbf{92.76$\pm$0.29} &\textbf{98.10} &18.68$\pm$2.99 &0.012$\pm$0.006 &\textbf{2.00$\pm$0.00} \\ \hline
	        			
	        			\multirow{2}{*}{DCC}
	        			&TWD-SFNN-R &\textbf{78.13$\pm$0.30} &\textbf{69.70$\pm$0.32} &\textbf{63.07} &\textbf{7.81$\pm$0.07} &\textbf{0.004$\pm$0.000} &\textbf{2.00$\pm$0.00} \\
	        			&STWD-SFNN &\textbf{78.13$\pm$0.30} &\textbf{69.70$\pm$0.32} &\textbf{63.07} &14.12$\pm$1.30 &0.007$\pm$0.002 &\textbf{2.00$\pm$0.00} \\ \hline
	        			
	        			\multirow{2}{*}{ESR}
	        			&TWD-SFNN-R &\textbf{84.33$\pm$0.26} &\textbf{81.52$\pm$0.37} &\textbf{56.23} &\textbf{6.18$\pm$0.73} &\textbf{0.006$\pm$0.001} &\textbf{2.00$\pm$0.00} \\
	        			&STWD-SFNN &\textbf{84.33$\pm$0.26} &\textbf{81.52$\pm$0.37} &\textbf{56.23} &14.88$\pm$1.26 &0.011$\pm$0.003 &\textbf{2.00$\pm$0.00} \\ \hline
	        			
	        			\multirow{2}{*}{BM}
	        			&TWD-SFNN-R &\textbf{89.27$\pm$0.46} &\textbf{84.24$\pm$0.66} &\textbf{81.91} &\textbf{11.79$\pm$0.38} &\textbf{0.003$\pm$0.000} &\textbf{1.00$\pm$0.00} \\
	        			&STWD-SFNN &\textbf{89.27$\pm$0.48} &\textbf{84.24$\pm$0.69} &\textbf{81.91} &13.80$\pm$0.55 &\textbf{0.003$\pm$0.000} &\textbf{1.00$\pm$0.00} \\ \hline
	        			
	        			\multirow{2}{*}{PCB}
	        			&TWD-SFNN-R &97.75$\pm$0.12 &96.76$\pm$0.01 &73.22 &\textbf{6.06$\pm$0.16} &\textbf{0.006$\pm$0.001} &\textbf{2.00$\pm$0.00} \\
	        			&STWD-SFNN &\textbf{97.85$\pm$0.74} &\textbf{96.95$\pm$0.48} &\textbf{73.99} &11.76$\pm$0.45 &0.009$\pm$0.002 &\textbf{2.00$\pm$0.00} \\ \hline
	        			
	        			\multirow{2}{*}{SB}
	        			&TWD-SFNN-R &\textbf{89.51$\pm$1.20} &\textbf{67.61$\pm$30.85} &\textbf{79.32} &\textbf{0.96$\pm$0.19} &\textbf{0.002$\pm$0.001} &\textbf{1.80$\pm$0.40} \\
	        			&STWD-SFNN &\textbf{89.51$\pm$1.20} &\textbf{67.61$\pm$30.85} &\textbf{79.32} &1.16$\pm$0.23 &\textbf{0.002$\pm$0.001} &\textbf{1.80$\pm$0.40} \\ \hline
	        			
	        			\multirow{2}{*}{EOL}
	        			&TWD-SFNN-R &\textbf{86.26$\pm$1.30} &\textbf{82.68$\pm$1.33} &\textbf{70.07} &\textbf{0.50$\pm$0.06} &\textbf{0.001$\pm$0.000} &\textbf{2.00$\pm$0.00} \\
	        			&STWD-SFNN &\textbf{86.26$\pm$1.30} &\textbf{82.68$\pm$1.33} &\textbf{70.07} &1.46$\pm$0.23 &0.002$\pm$0.001 &\textbf{2.00$\pm$0.00} \\ \hline
	        			
	        			\multirow{2}{*}{OD}
	        			&TWD-SFNN-R &76.56$\pm$0.33 &66.80$\pm$0.43 &68.33 &\textbf{2.85$\pm$0.42} &\textbf{0.001$\pm$0.000} &\textbf{2.00$\pm$0.00} \\
	        			&STWD-SFNN &\textbf{78.50$\pm$0.72} &\textbf{71.48$\pm$1.06} &\textbf{69.07} &12.23$\pm$2.76 &0.003$\pm$0.001 &\textbf{2.00$\pm$0.00} \\ \hline
	        			
	        			\multirow{2}{*}{ROE}
	        			&TWD-SFNN-R &83.82$\pm$0.90 &85.01$\pm$0.81 &82.14 &\textbf{1.13$\pm$0.26} &\textbf{0.002$\pm$0.001} &\textbf{2.00$\pm$0.00} \\
	        			&STWD-SFNN &\textbf{87.30$\pm$1.38} &\textbf{87.91$\pm$1.24} &\textbf{82.48} &3.55$\pm$0.08 &0.003$\pm$0.000 &2.20$\pm$0.40 \\ \hline
	        			
	        			\multirow{2}{*}{SSMCR}
	        			&TWD-SFNN-R &91.93$\pm$0.13 &\textbf{88.76$\pm$0.15} &70.27 &347.07$\pm$7.16 &\textbf{0.048$\pm$0.011} &\textbf{2.00$\pm$0.00} \\
	        			&STWD-SFNN &\textbf{92.65$\pm$0.09} &87.36$\pm$0.17 &\textbf{70.70} &\textbf{338.22$\pm$1.13} &0.048$\pm$0.015 &\textbf{2.00$\pm$0.00} \\ \hline	
	        	\end{tabular}}
	        	\label{tab_twd_5folds}
	   \end{table*}  	               	
       
       Table \ref{tab_twd_5folds} shows that the accuracy, weighted-f1, and AUC of TWD-SFNN and STWD-SFNN have significant differences on some datasets. For example, on the ROE dataset, the accuracy, weighted-f1, and AUC of STWD-SFNN are 87.30\%, 87.91\%, and 82.48\%, respectively, while those of TWD-SFNN are 83.82\%, 85.01\%, and 82.14\%, respectively. Moreover, as shown in Fig. \ref{fig_twd_5folds}, the ROC curve of STWD-SFNN is located at the top left of TWD-SFNN, which means that the performance of STWD-SFNN is better than that of TWD-SFNN. The same conclusion can be found in ONP, QSAR, PCB, OD, and SSMCR datasets. Therefore, TWD-SFNN and STWD-SFNN have some differences under 5-fold cross-validation. It should be noted that whether TWD-SFNN and STWD-SFNN have significant differences under other folds is worth developing in the future.

		\subsubsection{ Analysis of STWD-SFNN } \label{exp_costs}
		In this section, we analyze the relationship between the process costs and the number of hidden layer nodes in STWD-SFNN, including the costs of test and delay. 
		According to Definition \ref{def6_process_costs}, without loss of generality, we randomly generate 10 incremental unit test costs in the real range [1.00, 50.00].  To ensure the repeatability of experimental results, the random number seed is 0, i.e., rng (0), and 10 values are generated as unit test costs: [1.00, 6.44, 11.89, 17.33, 22.78, 28.22, 33.67, 39.11, 44.56, 50.00]. The parameter settings of unit delay cost are the same as those of unit test cost. Fig. \ref{fig_stwdsfnn_costs} shows the relationship between the process costs of STWD-SFNN on the training set and the number of hidden layer nodes. Since STWD-SFNN has only one hidden layer node on the ONP dataset, we analyze the process cost on the remaining fourteen datasets. The results give rise to the following observations according to Fig \ref{fig_stwdsfnn_costs}.
		
		\begin{figure*}[hp]
			\centering
			\subfigure[ONP]{
				\includegraphics[width=2.1in]{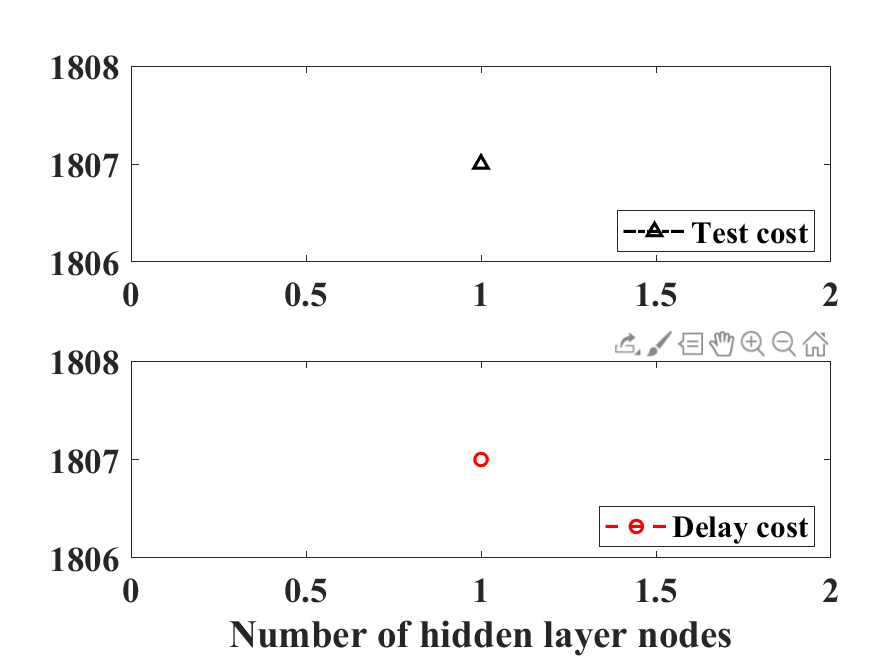}}
			\hfill
			\centering
			\subfigure[QSAR]{
				\includegraphics[width=2.1in]{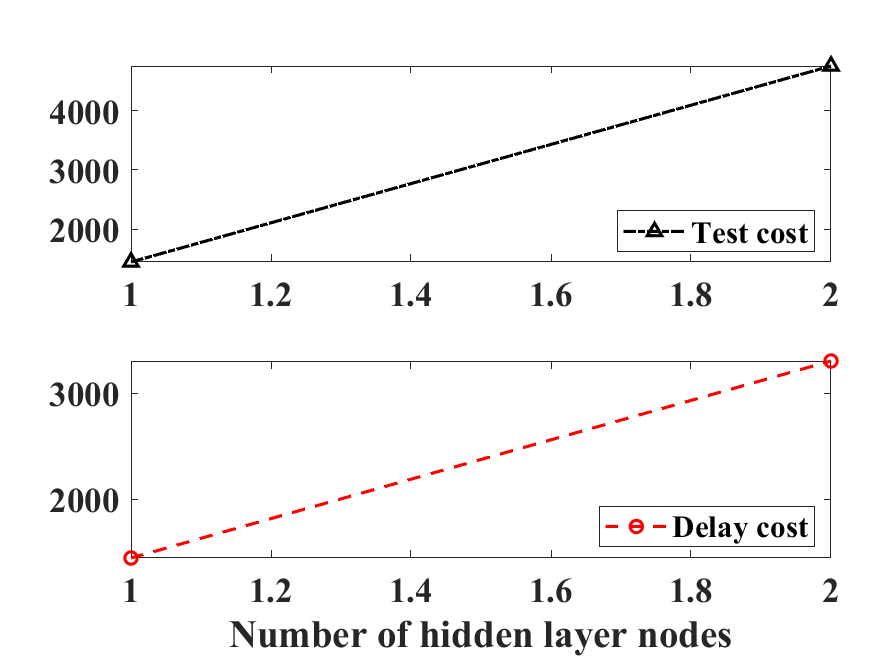}}
			\hfill
			\centering
			\subfigure[OSP]{
				\includegraphics[width=2.1in]{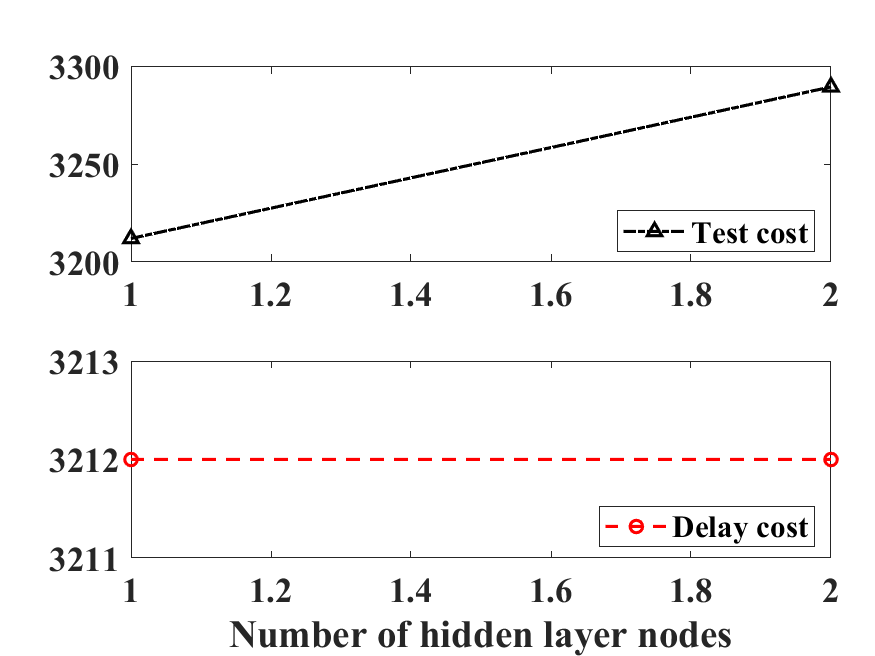}}
			\hfill
			\centering
			\subfigure[EGSS]{
				\includegraphics[width=2.1in]{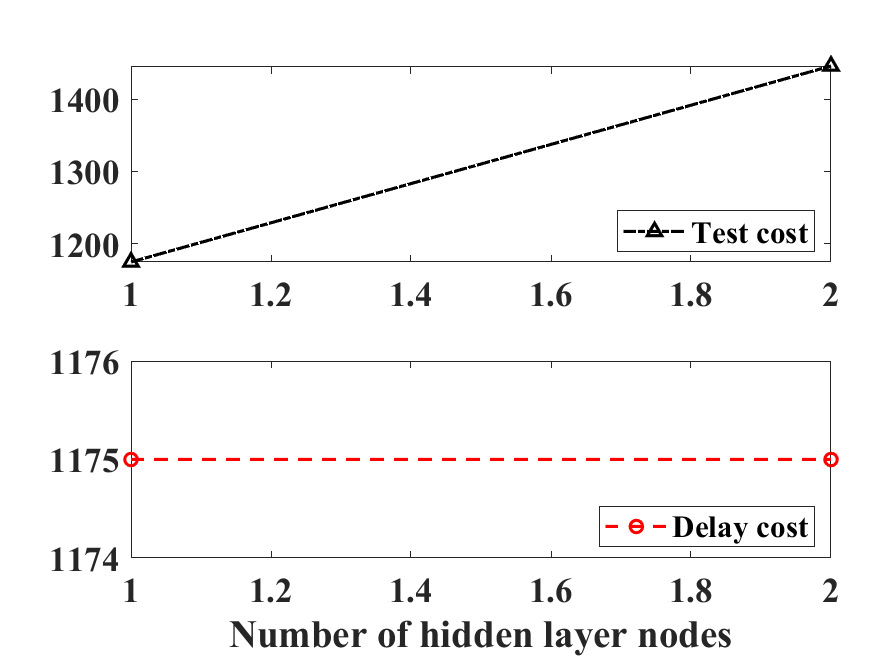}}
			\hfill
			\centering
			\subfigure[SE]{
				\includegraphics[width=2.1in]{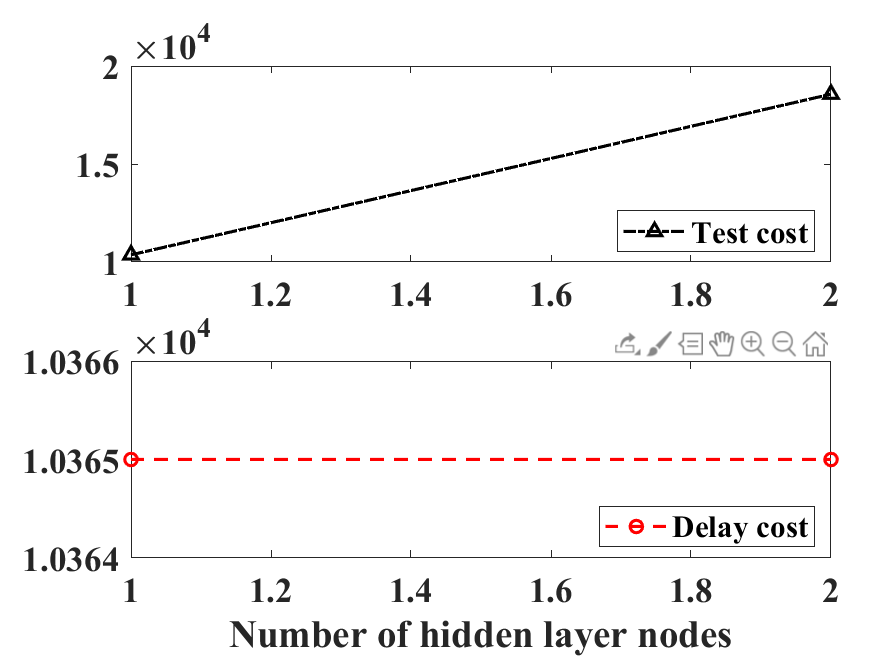}}
			\hfill
			\centering
			\subfigure[HTRU]{
				\includegraphics[width=2.1in]{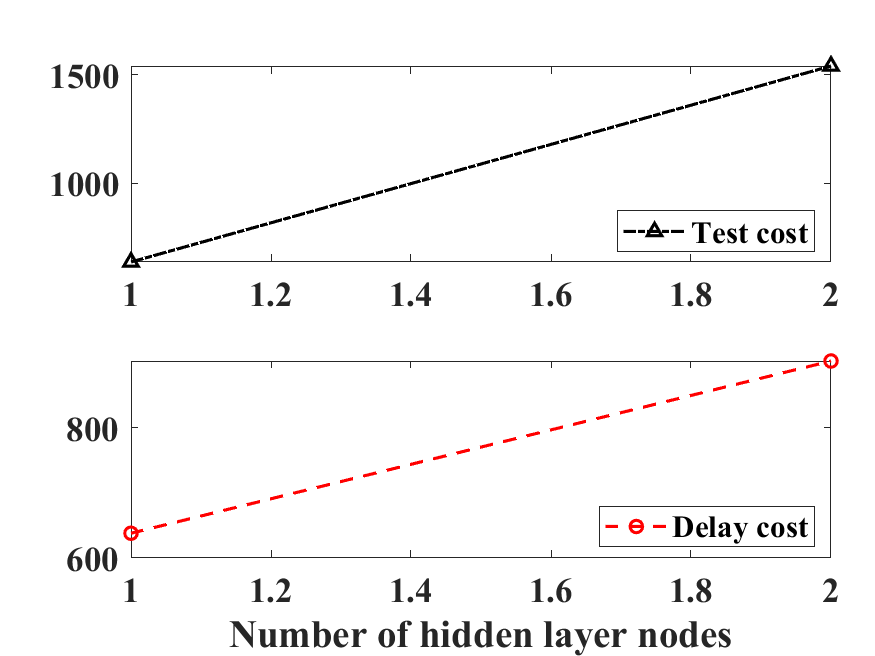}}
			\hfill
			\centering
			\subfigure[DCC dataset]{
				\includegraphics[width=2.1in]{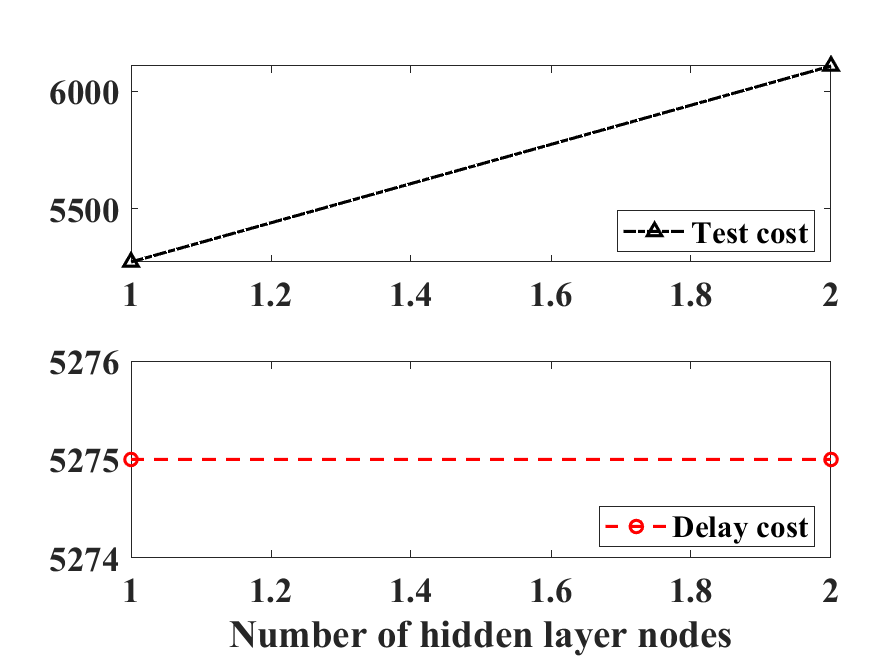}}
			\hfill
			\centering
			\subfigure[ESR]{
				\includegraphics[width=2.1in]{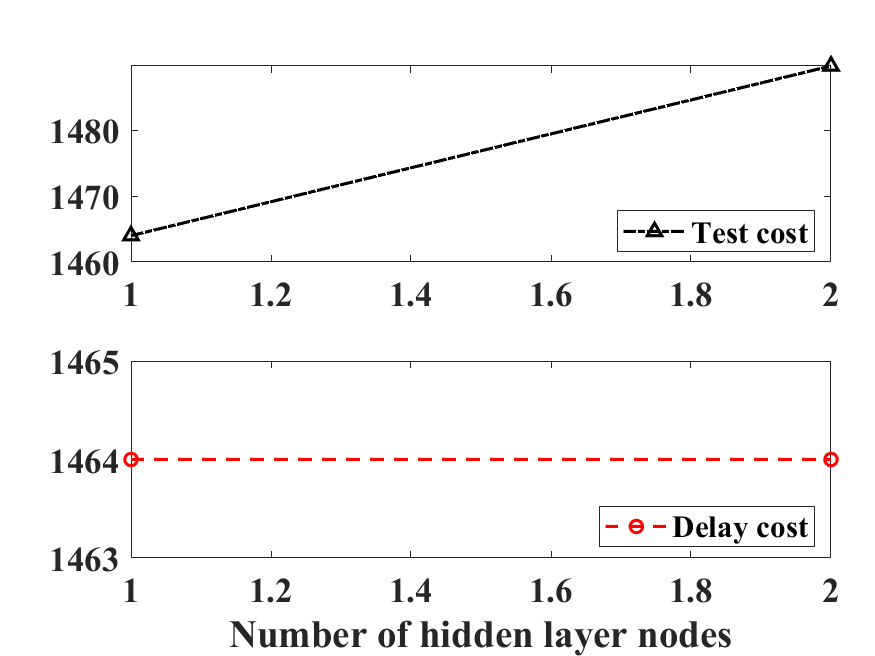}}	
			\hfill
			\centering
			\subfigure[BM]{
				\includegraphics[width=2.1in]{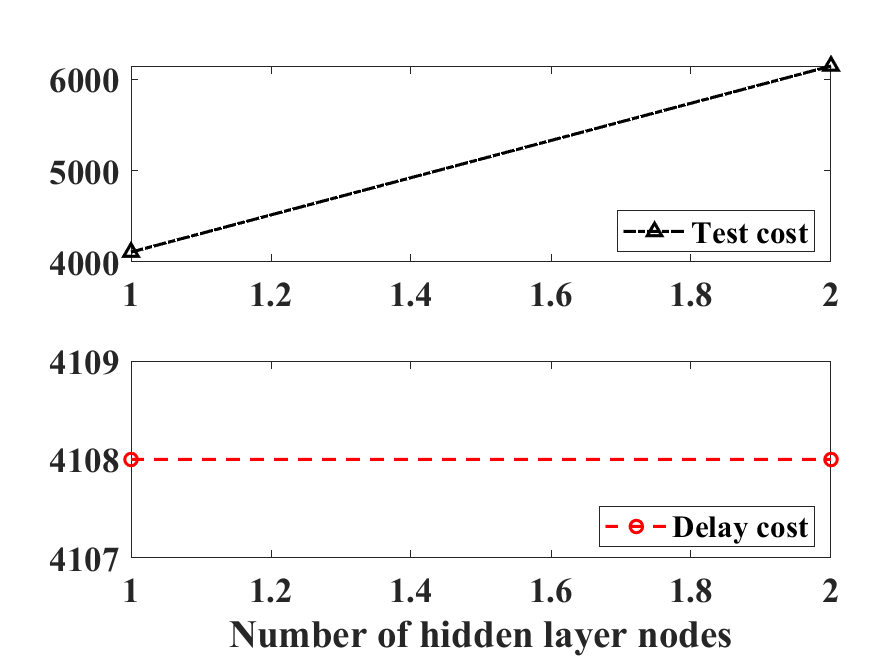}}
			\hfill
			\centering
			\subfigure[PCB]{
				\includegraphics[width=2.1in]{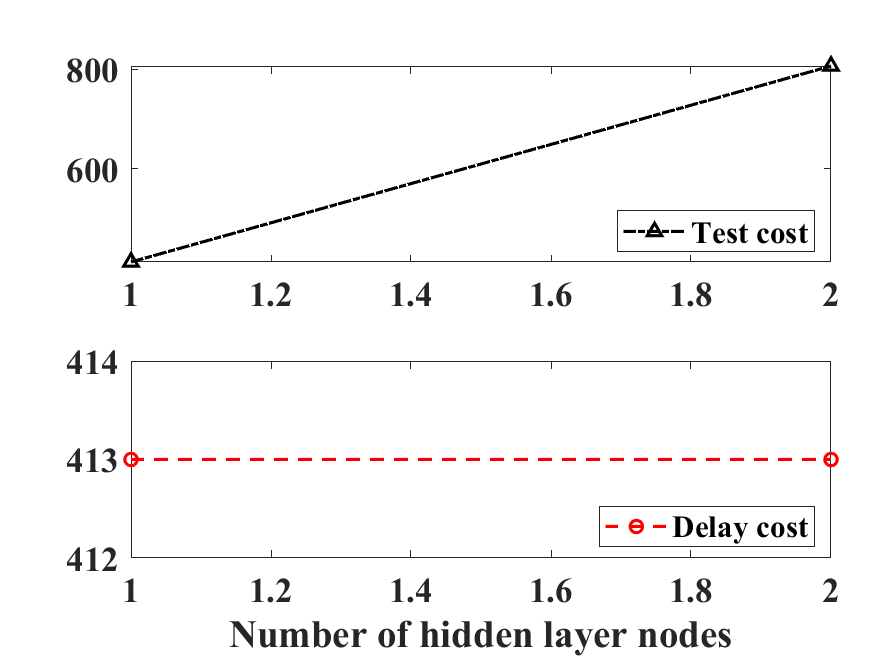}}
			\hfill	
			\centering
			\subfigure[SB]{
				\includegraphics[width=2.1in]{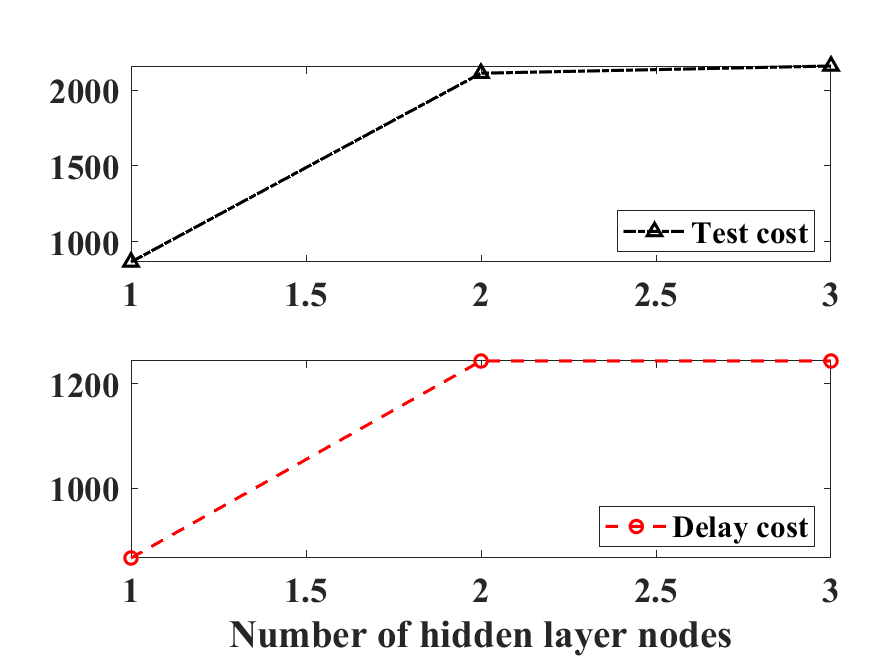}}
			\hfill	
			\centering
			\subfigure[EOL]{
				\includegraphics[width=2.1in]{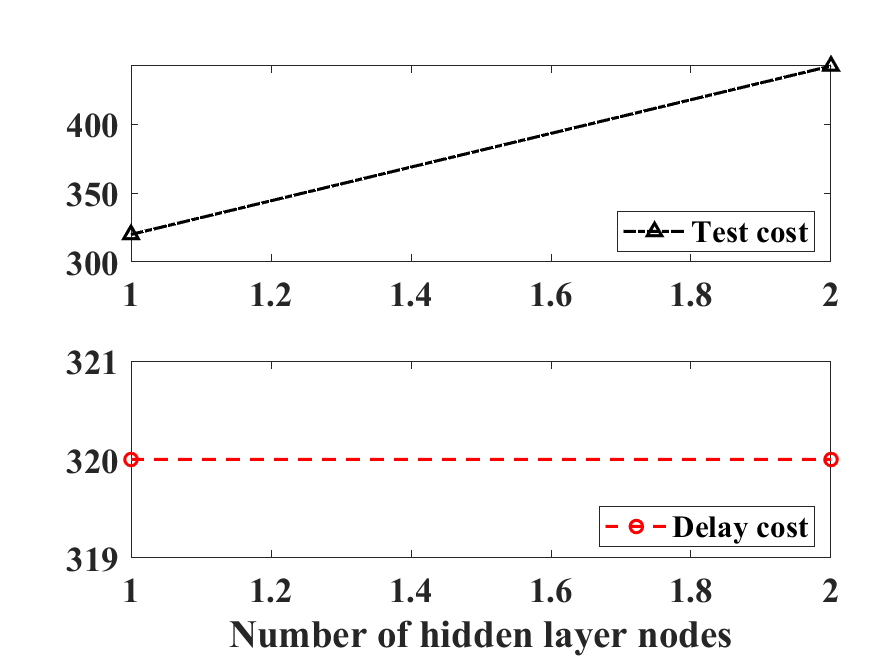}}
			\hfill	
			\centering	
			\subfigure[OD]{
				\includegraphics[width=2.1in]{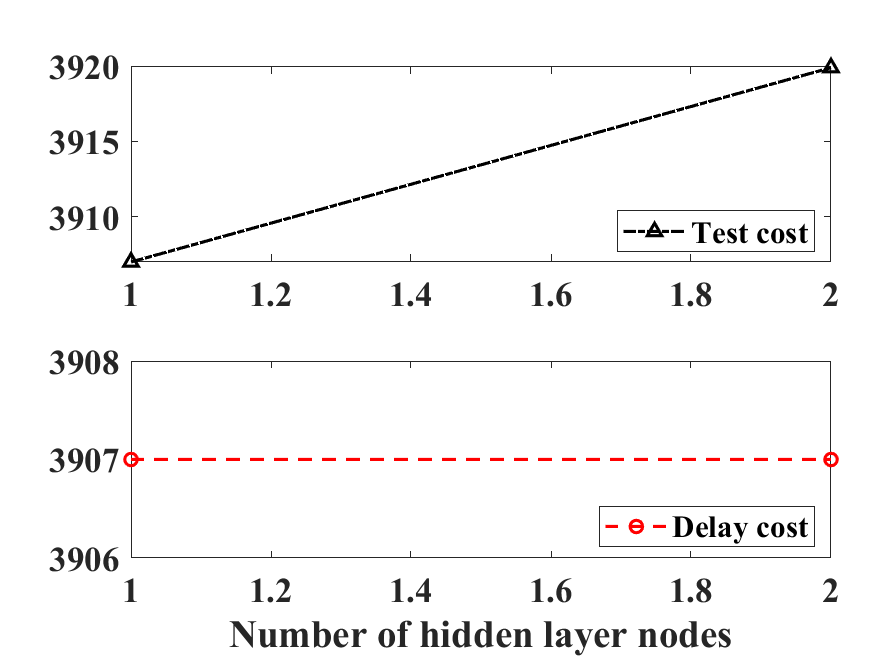}}
			\hfill	
			\centering
			\subfigure[ROE]{
				\includegraphics[width=2.1in]{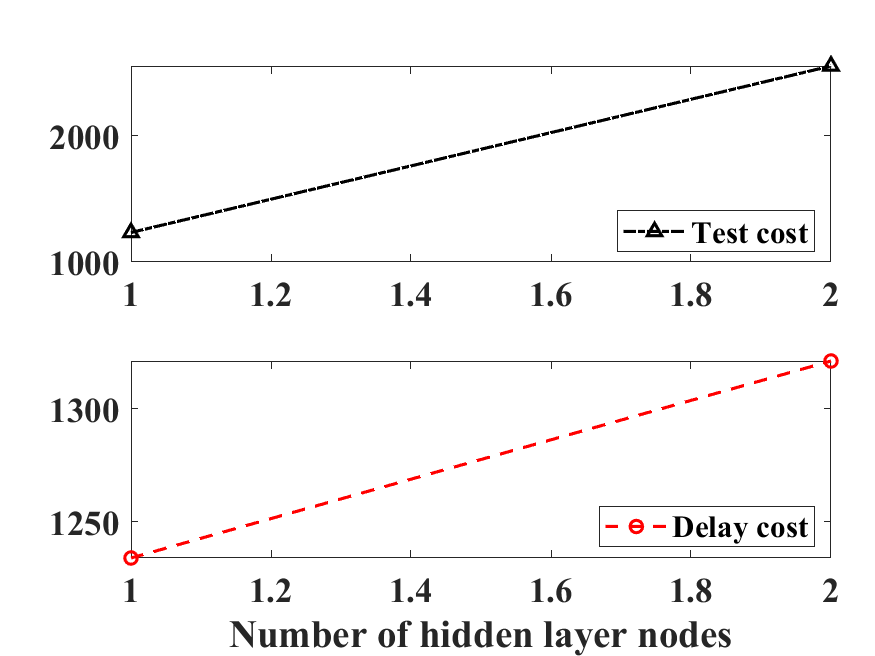}}
			\hfill	
			\centering
			\subfigure[SSMCR]{
				\includegraphics[width=2.1in]{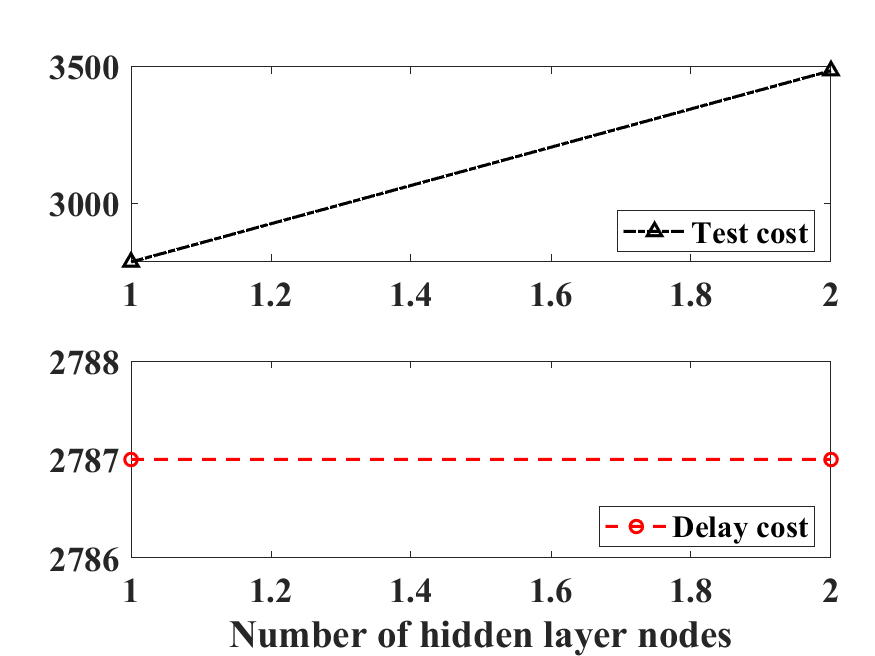}}
			\hfill	
			\small
			\captionsetup{font={small}}
			\caption{  Relationship between process costs and number of hidden layer nodes }
			\label{fig_stwdsfnn_costs}
		\end{figure*}
		
		1. The test cost of STWD-SFNN increases with the number of hidden layer nodes. For example, on the HTRU dataset, according to Definition \ref{def6_process_costs} and $ [Cost_{{\rm PPT}_1}, Cost_{{\rm PPT}_2}] = [1.00, 6.44] $, the test costs of the first and second levels are 638.00 and 1540.22, respectively. Similar phenomena can be found in the rest 13 datasets. The reason is that the test cost of STWD-SFNN in the $ t $-granular level is equivalent to the sum of the test costs of the first ($t$-1)-granular level and the $t$-granular level.
		
		2. The delay cost of STWD-SFNN and the number of hidden layer nodes report two cases. Case 1: The delay cost tends to be constant with the increase of hidden layer nodes. For example, on the DCC dataset, according to Definition \ref{def6_process_costs} and $ [Cost_{{\rm PPD}_1}, Cost_{{\rm PPD}_2}] = [1.00, 6.44] $, the delay costs of the first and second levels are both 5275.00. The same conclusion can be found in OSP, EGSS, SE, ESR, BM, PCB, EOL, OD, and SSMCR. Case 2: The delay cost of STWD-SFNN increases with the number of hidden layer nodes. For example, on the HTRU dataset, the delay costs of the first and second levels are 638.00 and 902.22, respectively. This phenomenon can be found in QSAR and ROE datasets. Specifically, the SE dataset has both Case 1 and Case 2. The reason is that the delay cost of STWD-SFNN in the $t$ granular level is equivalent to the maximum of the delay cost of the ($t$-1)-granular level and the $t$-granular level. Specifically, if the delay cost of the $t$ granular level is not greater than that of the ($t$-1)-granular level, the delay cost shows a constant trend; otherwise, it shows an increasing trend.

    \section{ Conclusion }\label{section5} 
 {Although TWD-SFNN can overcome the shortage of traditional SFNN and obtain a simpler network structure and better performance, it uses fixed threshold parameters to train the network model. Inspired by STWD, this paper proposes a new network topology optimization model, called STWD-SFNN, to enhance the performance of the network on structured datasets. STWD-SFNN adopts a sequential strategy to dynamically determine the number of hidden layer nodes. STWD-SFNN has two parts: discretion modular and training modular. Discretion modular adopts $k$-means++ to convert numerical data into discrete data to meet the requirements of STWD for data types. The training modular, as the main part of STWD-SFNN, adopts a sequential strategy to dynamically determine the number of hidden layer nodes. Thus, STWD-SFNN has $t$ levels. Each level has two parts: SFNN and STWD. SFNN is a network topology with one hidden layer node and gets correctly classified instances and misclassification instances. The first $t$-1 part of STWD utilizes the strategy of three ways, and the last part applies the strategy of two ways. The experimental results show that STWD-SFNN has higher operational efficiency and a more compact network structure than SFNN using empirical formulas, GS-SFNN, PSO-SFNN, TWD-SFNN-R, and STWD-SFNN-NK, and has a better generalization ability on structured datasets than the competitive models.}
    
%    A new network topology optimization model, called STWD-SFNN, adopts a sequential strategy to dynamically determine the number of hidden layer nodes. Focus loss and Adam optimizer are adopted to realize the learning of neural networks, and the difficult-to-classify instances are classified into BND. Moreover, BND is divided into POS and NEG by utilizing granular levels and sequential thresholds. Finally, we increase the number of hidden layer nodes and relax the thresholds in turn until the BND is empty. The experimental results show that STWD-SFNN has higher operational efficiency and a more compact network structure than SFNN using empirical formulas, GS-SFNN, PSO-SFNN, TWD-SFNN-R, and STWD-SFNN-NK, and has a better generalization ability on structured datasets than the competitive models. 
    
    In this paper, we propose STWD-SFNN which improves the performance of SFNN. However, many aspects should be investigated in the future to enrich the theory of neural networks. 
    
\begin{enumerate}[(1)] 
    \item        	 STWD-SFNN employs sequential three-way decisions to guide the growth of the topology of a single hidden layer neural network. However, for multilayer neural networks, how to adopt our model to handle XOR and other problems to further promote deep learning is worthy of investigation.
    	
    \item        	 STWD-SFNN has good performance in binary classification problems by comparing the conditional probability of each instance belonging to the positive region and threshold parameters. For a multi-class classification problem, one method is to convert it into multiple binary classification problems. Obviously, this method is not very effective. How to adopt STWD-SFNN to solve multi-class classification is worthy of consideration. 
    	
    \item        	  STWD-SFNN adopts $ k $-means++ to discretize the data when calculating the conditional probability of instances, since STWD is a classification model established on the decision-theoretic rough sets and derived from nominal attributes. However, whether there is a better discretization model to transform non-integer data into integer data, or whether there are other rough set models that can essentially mine numerical or mixed data features.
    \end{enumerate}

\section*{Acknowledgement}
This work was supported by the National Social Science Fund of China under grant number 18BGL191.


\begin{thebibliography}{99}

\bibitem{Aziz2021ESWA}
	M. F. Ab Aziz, S. A Mostafa, C. F. M. Foozy, M. A. Mohammed, M. Elhoseny, A. Z. Abualkishik, Integrating Elman recurrent neural network with particle swarm optimization algorithms for an improved hybrid training of multidisciplinary datasets, Expert Systems with Applications 183(2021) 115441.
  
\bibitem{Akyol2020ESWA}
	K. Akyol, Comparing of deep neural networks and extreme learning machines based on growing and pruning approach, Expert Systems with Applications 140(2020) 112875.
		
\bibitem{Arqub2014INs}
	O. A. Arqub, Z. Abo-Hammour, Numerical solution of systems of second-order boundary value problems using continuous genetic algorithm, Information Sciences 279(2014) 396--415.
		
\bibitem{Arqub2021MMAS}
	O. A. Arqub, J. Singh, M. Alhodaly, Adaptation of kernel functions-based approach with Atangana–Baleanu–Caputo distributed order derivative for solutions of fuzzy fractional Volterra and Fredholm integrodifferential equations, Mathematical Methods in the Applied Sciences (2021) 1--28.
	
		\bibitem{Belciug2021ESWA}
		S. Belciug, Parallel versus cascaded logistic regression trained single-hidden feedforward neural network for medical data, Expert Systems with Applications 170(2021) 114538.
		
		\bibitem{Cao2016INS}
		F. Cao, D. Wang, H.-Y. Zhu, Y. Wang, An iterative learning algorithm for feedforward neural networks with random weights, Information Sciences 328(2016) 546--557.
		
		\bibitem{Cheng2021INS}
		S. Cheng, Y. Wu, Y. Li, F. Yao, F. Min, TWD-SFNN: Three-way decisions with a single hidden layer feedforward neural network, Information Sciences 579(2021) 15--32.
		
		\bibitem{Chu2020CIE}
		X. Chu, B. Sun, Q. Huang, Y. Zhang, Preference degree-based multi-granularity sequential three-way group conflict decisions approach to the integration of TCM and Western medicine, Computers \& Industrial Engineering 143(2020) 106393.
		
		\bibitem{Ciucci2020INS}
		D. Ciucci, Y. Yao, Synergy of granular computing, shadowed sets, and three-way decisions, Information Sciences 508(2020) 422--425.
		
		\bibitem{Escovedo2020APIN}
		T. Escovedo, A. S. Koshiyama, A. A. da Cruz, M. M.B.R. Vellasco, Neuroevolutionary learning in nonstationary environments, Applied Intelligence 50(2020) 1590--1608.
		
		\bibitem{Fang2020INS}
		Y. Fang, C. Gao, Y. Yao, Granularity-driven sequential three-way decisions: A cost-sensitive approach to classification, Information Sciences 507 (2020) 644--664.
		
		\bibitem{Geler2020ESWA}
		Z. Geler, V. Kurbalija, M. Ivanovi\'{c}, M. Radovanovi\'{c}, Weighted $k$NN and constrained elastic distances for time-series classification, Expert Systems with Applications 162(2020) 113829.
		
		%\bibitem{Hwang2017JE}
		%J. Hwang, Y. Sun, Asymptotic $F$ and $t$ tests in an efficient GMM setting, Journal of Econometrics 198(2017) 277--295.
		
		\bibitem{Katuwal2020PR}
		R. Katuwal, P. N. Suganthan, L. Zhang, Heterogeneous oblique random forest, Pattern Recognition 99(2020) 107078.
		
		\bibitem{Khan2020JAS}
		A. H. Khan, X. Cao, S. Li, V. N. Katsikis, L. Liao, BAS-ADAM: An ADAM based approach to improve the performance of beetle antennae search optimizer, IEEE/CAA Journal of Automatica Sinica 7(2020) 461--471.
		
		\bibitem{Kim2020TFS}
		E.-H. Kim, S.-K. Oh, W. Pedrycz, Z. Fu, Reinforced fuzzy clustering-based ensemble neural networks, IEEE Transactions on Fuzzy Systems 28(2020) 569--582.
		
		\bibitem{Li2020ASC}
		L.-J. Li, M.-Z. Li, J.-S. Mi, B. Xie, Dynamic granularity selection based on local weighted accuracy and local likelihood ratio, Applied Soft Computing 89(2020) 106087.
		
		\bibitem{Li2022ESWA}
		M. Li, W. Li, J. Qiao, Design of a modular neural network based on an improved soft subspace clustering algorithm, Expert Systems with Applications 209(2022) 118219.
		
		\bibitem{Li2022KBS}
		H. Li, J. Wang, Collaborative annealing power $k$-means++ clustering, Knowledge-Based Systems (2022) 109593. %doi: 10.1016/j.knosys.2022.109593.
		
		\bibitem{Li2016KBS}
		H. Li, L. Zhang, B. Huang, X. Zhou, Sequential three-way decision and granulation for cost-sensitive face recognition, Knowledge-Based Systems 91(2016) 241--251.
	
		\bibitem{Liu2020TPAMI}
		Y. Liu, S. Liao, S. Jiang, L. Ding, H. Lin, W. Wang, Fast cross-validation for kernel-based algorithms, IEEE Transactions on Pattern Analysis and Machine Intelligence 42(2020) 1083--1096.
		
		\bibitem{Liu2022KBS}
		X. Liu, Z. Pan, W. Tao, Provable convergence of Nesterov’s accelerated gradient method for over-parameterized neural networks, Knowledge-Based Systems 251(2022) 109277.
		
		\bibitem{Ma2022TKDD}
		P. Ma, Y. Wu, Y. Li, L. Guo, H. Jiang, X. Zhu, X. Wu, HW-Forest: Deep forest with hashing screening and window screening, ACM Transactions on Knowledge Discovery from Data 16(2022) 1--24.
		
		\bibitem{Ma2022Neurocomputing}
		P. Ma, Y. Wu, Y. Li, L. Guo, Z. Li, DBC-Forest: Deep forest with binning confidence screening, Neurocomputing 475(2022) 112--122.
		
		\bibitem{Nistor2022ESWA}
		S. Nistor, G. Czibula, IntelliSwAS: Optimizing deep neural network architectures using a particle swarm-based approach, Expert Systems with Applications 187(2022) 115945.
		
		\bibitem{Niu2022INS}
		J. Niu, D. Chen, J. Li, H. Wang, A dynamic rule-based classification model via granular computing, Information Sciences 584(2022) 325--341.
		
		\bibitem{Pontes2016Neurocomputing}
		F. J. Pontes, G. F. Amorim, P. P. Balestrassi, A. P. Paiva, J. R. Ferreira, Design of experiments and focused grid search for neural network parameter optimization, Neurocomputing 186(2019) 22--34.	
		
		\bibitem{Qian2017IJAR}
		J. Qian, C. Dang, X. Yue, N. Zhang, Attribute reduction for sequential three-way decisions under dynamic granulation, International Journal of Approximate Reasoning 85(2017) 196--216.
		
		\bibitem{Qian2022IJAR}
		W. Qian, Y. Zhou, J. Qian, Y. Wang, Cost-sensitive sequential three-way decision for information system with fuzzy decision, International Journal of Approximate Reasoning 149(2022) 85--103.
		
		\bibitem{Rashid2022ITJ}
		N. Rashid, B. U. Demirel, M. A. Faruque, AHAR: Adaptive CNN for energy-efficient human activity recognition in low-power edge devices, IEEE Internet of Things Journal 9(2022) 13041--13051.
		
		\bibitem{Rusek2020JSAC}
		K. Rusek, J. Su\'{a}rez-Varela, P. Almasan, P. Barlet-Ros, A. Cabellos-Aparicio, RouteNet: Leveraging graph neural networks for network modeling and optimization in SDN, IEEE Journal on Selected Areas in Communications 38(2020) 2260--2270.
		
		\bibitem{Shekar2021IJCV}
		A. K. Shekar, L. Gou, L. Ren, A. Wendt, Label-free robustness estimation of object detection CNNs for autonomous driving applications, International Journal of Computer Vision 129(2021) 1185--1201.
		
		\bibitem{Shi2019SJ}
		L. Shi, Y. He, B. Li, T. Cheng, Y. Huang, Y. Sui, Tilt angle monitoring by using sparse residual LSTM network and grid search, IEEE Sensors Journal 19(2019) 8803--8812.
		
		\bibitem{Szeto2020TPAMI}
		R. Szeto, X. Sun, K. Lu, J. J. Corso, A temporally-aware interpolation network for video frame inpainting, IEEE Transactions on Pattern Analysis and Machine Intelligence 42(2020) 1053--1068.
		

		
		\bibitem{Wang2022INS}
		W. Wang, J. Zhan, E. Herrera-Viedma, A three-way decision approach with a probability dominance relation based on prospect theory for incomplete information systems, Information Sciences 611(2022) 199--224.
		
		\bibitem{Wu2022TKDD}
		Y. Wu, L. Luo, Y. Li, L. Guo, P. Fournier-Viger, X. Zhu, X. Wu, NTP-Miner: Nonoverlapping three-way sequential pattern mining, ACM Transactions on Knowledge Discovery from Data 16(2022) 1--21.
		
		
		\bibitem{Xu2022INS}
		Y. Xu, B. Li, Multiview sequential three-way decisions based on partition order product space, Information Sciences 600(2022) 401--430.
		
		\bibitem{Yang2022KBS}
		X. Yang, Y. Chen, H. Fujita, D. Liu, T. Li, Mixed data-driven sequential three-way decision via subjective–objective dynamic fusion, Knowledge-Based Systems 237(2022) 107728.
		
		\bibitem{Yang2019IJAR}
		X. Yang, T. Li. H. Fujita, D. Liu, A sequential three-way approach to multi-class decision, International Journal of Approximate Reasoning 104(2019) 108--125.
		
		\bibitem{Yang2017KBS}
		X. Yang, T. Li, H. Fujita, D. Liu, Y. Yao, A unified model of sequential three-way decisions and multilevel incremental processing, Knowledge-Based Systems 134(2017) 172--188.
		
		\bibitem{Yang2019AIR}
		X. Yang, G. Liu, S. Deng, Z. Wei, H. He, Y. Shang, N. Deng, Exploration of a mechanism to form bionic, self-growing and self-organizing neural network, Artificial Intelligence Review 52(2019) 585–605.
		
		\bibitem{Yao2019KBS}
		Y. Yao, Three-way conflict analysis: Reformulations and extensions of the Pawlak model, Knowledge-Based Systems 180(2019) 26--37.
		
		\bibitem{Ye2018ESWA}
		H. Ye, F. Cao, D. Wang, H. Li, Building feedforward neural networks with random weights for large scale datasets, Expert Systems with Applications 106(2018) 233--243.
		
		\bibitem{Yi2022ESWA}
		J. Yi, Y. Liu, J. Y.-L. Forrest, X. Guo, X. Xu, A three-way decision approach with S-shaped utility function under Pythagorean fuzzy information, Expert Systems with Applications 210(2022) 118370.
		
		\bibitem{Zhan2022TFS}
		J. Zhan, J. Ye, W. Ding, P. Liu, A novel three-way decision model based on utility theory in incomplete fuzzy decision systems, IEEE Transactions on Fuzzy Systems 30(2022) 2210--2226.
		
		\bibitem{Zhang-gs-2021KBS} 
		X. Zhang, H. Gou, Z. Lv, D. Miao, Double-quantitative distance measurement and classification learning based on the tri-level granular structure of neighborhood system, Knowledge-Based Systems 217(2021) 106799.
		
		\bibitem{Zhang-stwd-2021KBS}
		Q. Zhang, Z. Huang, G. Wang, A novel sequential three-way decision model with autonomous error correction, Knowledge-Based System 212(2021) 106526.
		
		\bibitem{Zhang2021TSMCS}
		Q. Zhang, C. Yang, G. Wang, A sequential three-way decision model with intuitionistic fuzzy numbers, IEEE Transactions on Systems, Man, and Cybernetics Systems 51(2021) 2640--2652.
		
		\bibitem{Zhang2022ESWA}
		X. Zhang, Y. Yao, Tri-level attribute reduction in rough set theory, Expert Systems with Applications 190(2022) 116187.	
		
		\bibitem{Zhou2020TIP}
		Q. Zhou, B. Zhong, X. Lan, G. Sun, Y. Zhang, B. Zhang, R. Li, Fine-grained spatial alignment model for person re-identification with focal triplet loss, IEEE Transactions on Image Processing 29(2020) 7578--7589.
	\end{thebibliography}
\end{document}